\PassOptionsToPackage{final,main}{neurips_2026}
\documentclass{article}

\PassOptionsToPackage{numbers, super, compress}{natbib}

\usepackage{neurips_2026}

\usepackage[utf8]{inputenc}
\usepackage[T1]{fontenc}
\usepackage{hyperref}
\usepackage{url}
\usepackage{booktabs}
\usepackage{makecell}      %
\setcellgapes{3pt}
\makegapedcells
\usepackage{amsfonts}
\usepackage{amsmath}
\usepackage{amssymb}
\usepackage{amsthm}
\usepackage{nicefrac}

\newtheorem{lemma}{Lemma}[section]
\newtheorem{proposition}[lemma]{Proposition}
\usepackage{microtype}
\usepackage{xcolor}
\usepackage{graphicx}
\usepackage{multirow}
\usepackage{xspace}
\usepackage{wrapfig}
\usepackage{placeins}
\usepackage{needspace}
\usepackage{listings}

\usepackage{minitoc}

\noptcrule

\definecolor{lstkw}{RGB}{0,0,153}       %
\definecolor{lstcomment}{RGB}{106,153,85} %
\definecolor{lststring}{RGB}{163,21,21} %
\definecolor{lstbuiltin}{RGB}{43,145,175} %

\definecolor{archCanvas}{HTML}{e05050}     %
\definecolor{archGlimpse}{HTML}{4080d0}    %
\definecolor{archCls}{HTML}{50b050}        %
\definecolor{archReg}{HTML}{b0b0b0}        %
\definecolor{archVpe}{HTML}{f0a030}        %
\definecolor{archBackbone}{HTML}{d0a8e0}   %

\newif\ifpolicyswatches
\policyswatchestrue
\newcommand{\policyswatch}[1]{\ifpolicyswatches\textcolor[HTML]{#1}{\rule[0.35ex]{0.9em}{1.2pt}}\hspace{0.4em}\fi}

\lstdefinestyle{canvitpython}{
  language=Python,
  basicstyle=\ttfamily\footnotesize,
  keywordstyle=\color{lstkw}\bfseries,
  commentstyle=\color{lstcomment}\itshape,
  stringstyle=\color{lststring},
  emphstyle=\color{lstbuiltin}\bfseries,
  emph={LayerNorm,Linear,Identity,nn,Module,arange,stack,meshgrid,reshape,
        to_multihead,from_multihead,sdpa,apply_2d_rope,compute_2d_rope},
  showstringspaces=false,
  breaklines=false,
  columns=fullflexible,
  keepspaces=true,
  upquote=true,
  frame=single,
  framerule=0.4pt,
  rulecolor=\color{lightgray},
  xleftmargin=0pt,
  xrightmargin=0pt,
}

\newcommand{\adeBestMiou}{45.9}
\newcommand{\adeSingleGlimpseMiou}{38.5}
\newcommand{\adeBestPriorMiou}{27.6}

\newcommand{\adeCheapestBeatMiou}{38.5}
\newcommand{\adeCheapestBeatFlopRatio}{20}
\newcommand{\ameSetrMiou}{27.6}
\newcommand{\adaglimpseEightMiou}{25.7}
\newcommand{\adeMaxT}{21}
\newcommand{\adeLastT}{20}
\newcommand{\bootstrapCiPct}{95}
\newcommand{\adePolicyEvalStochN}{10}
\newcommand{\adeSweepN}{10}
\newcommand{\adeSweepHeader}{\textbf{Policy} & \textbf{$c=8$} & \textbf{$c=16$} & \textbf{$c=32$} & \textbf{$c=64$}}
\newcommand{\canvasGridImpactGrids}{8, 16, 32, 64}
\newcommand{\canvasGridImpactN}{10}
\newcommand{\inkFrozenBest}{81.1}
\newcommand{\inkFinetunedBest}{84.5}
\newcommand{\inkMaxT}{21}
\newcommand{\inkFrozenStochN}{10}
\newcommand{\inkFinetunedStochN}{10}
\newcommand{\inkSweepHasData}{1}
\newcommand{\inkSweepN}{10}
\newcommand{\inkSweepHeader}{\textbf{Policy} & \textbf{$c=8$} & \textbf{$c=16$} & \textbf{$c=32$} & \textbf{$c=64$}}
\newcommand{\inkFtOptimizer}{AdamW}
\newcommand{\inkFtLr}{2.5 \times 10^{-5}}
\newcommand{\inkFtWeightDecay}{1 \times 10^{-4}}
\newcommand{\inkFtGradClip}{1.0}
\newcommand{\inkFtBatchSize}{256}
\newcommand{\inkFtLabelSmoothing}{0.1}
\newcommand{\inkFtWarmupSteps}{25{,}000}
\newcommand{\inkFtTotalSteps}{100{,}080}
\newcommand{\inkFtEpochs}{20}
\newcommand{\inkFtNGlimpses}{4}
\newcommand{\inkFtBpttMode}{full}
\newcommand{\ablStepsK}{215k}
\newcommand{\ablGlimpsesM}{110}
\newcommand{\ablT}{10}
\newcommand{\ablTLast}{9}
\newcommand{\ablPolicyName}{R-IID}
\newcommand{\ablNRuns}{5}

\newcommand{\ablDownstreamSegPolicy}{EG-C2F}
\newcommand{\ablDownstreamClfPolicy}{C2F}
\newcommand{\ablDownstreamGrid}{32}
\newcommand{\ablDownstreamN}{10}
\newcommand{\benchTimeBudgetS}{20}
\newcommand{\benchMaxIters}{500}
\newcommand{\benchNPasses}{3}
\newcommand{\benchWarmupItersPhrase}{3 warmup iterations}
\newcommand{\benchMinItersPhrase}{5 iterations}
\newcommand{\benchNPassesPhrase}{3 times}
\newcommand{\cSixFourCanvasGrid}{64}
\newcommand{\cSixFourAsymRwPairGf}{2.8}
\newcommand{\cSixFourFullRwPairGf}{37.3}
\newcommand{\flopsCvCanvasDim}{1024}
\newcommand{\flopsCvNLocal}{71}
\newcommand{\flopsCvNPatches}{64}
\newcommand{\flopsCvBackboneRegs}{5}
\newcommand{\projSdpaRatio}{7.2}
\newcommand{\lowessFrac}{0.25}
\newcommand{\lowessIt}{0}
\newcommand{\lowessNBoot}{1{,}000}
\newcommand{\lowessCiPct}{95}
\newcommand{\lowessTDeltaEarly}{0}
\newcommand{\lowessTDeltaLate}{20}
\newcommand{\dsT}{20}
\newcommand{\dsAdeGrid}{64}
\newcommand{\dsInkGrid}{32}
\newcommand{\dsMinN}{10}
\newcommand{\dsInOneStep}{1{,}837{,}056}
\newcommand{\dsSchedule}{2{,}000{,}000}
\newcommand{\dsInOneStepPct}{92}
\newcommand{\dsInOneLpft}{84.2}
\newcommand{\dsInOneAde}{43.4}
\newcommand{\dsAdeFrozenHeader}{\textbf{$t{=}0$} & \textbf{$t{=}1$} & \textbf{$t{=}2$} & \textbf{$t{=}3$} & \textbf{$t{=}4$} & \textbf{$t{=}9$} & \textbf{$t{=}16$} & \textbf{$t{=}20$}}
\newcommand{\dsAdeFrozenN}{10}
\newcommand{\dsInkFrozenHeader}{\textbf{$t{=}0$} & \textbf{$t{=}1$} & \textbf{$t{=}2$} & \textbf{$t{=}3$} & \textbf{$t{=}4$} & \textbf{$t{=}5$} & \textbf{$t{=}9$} & \textbf{$t{=}15$} & \textbf{$t{=}20$}}
\newcommand{\dsInkFrozenN}{10}
\newcommand{\dsInkLpftHeader}{\textbf{$t{=}0$} & \textbf{$t{=}1$} & \textbf{$t{=}2$} & \textbf{$t{=}3$} & \textbf{$t{=}4$} & \textbf{$t{=}5$} & \textbf{$t{=}9$} & \textbf{$t{=}15$} & \textbf{$t{=}20$}}
\newcommand{\dsInkLpftN}{11}
\newcommand{\vpeHeader}{\textbf{$t{=}0$} & \textbf{$t{=}3$} & \textbf{$t{=}10$} & \textbf{$t{=}15$} & \textbf{$t{=}20$}}
\newcommand{\vpeGeometries}{C32: $32 \times 32$ canvas, $512^2$\,px scenes. C64: $64 \times 64$ canvas, $1024^2$\,px scenes.}
\newcommand{\vpeScenes}{$512^2$ and $1024^2$\,px}
\newcommand{\ctrlLastT}{3}
\newcommand{\ctrlLead}{18.1}
\newcommand{\ctrlParamsCanvit}{96.5}
\newcommand{\ctrlParamsBaseline}{112.3}
\newcommand{\ctrlChipHoursCanvit}{22}
\newcommand{\ctrlChipHoursBaseline}{212}
\newcommand{\ctrlTokensPerSide}{2}
\newcommand{\ctrlBaselineTokens}{4}
\newcommand{\ctrlCanvitPatchesPerSide}{8}
\newcommand{\ctrlPatchPx}{16}
\newcommand{\ctrlCanvitTokens}{64}
\newcommand{\ctrlPretrainSteps}{26{,}875}
\newcommand{\ctrlT}{4}
\newcommand{\ctrlGrid}{32}
\newcommand{\cocoHeader}{\textbf{$t{=}0$} & \textbf{$t{=}1$} & \textbf{$t{=}2$} & \textbf{$t{=}3$} & \textbf{$t{=}4$} & \textbf{$t{=}9$} & \textbf{$t{=}20$}}
\newcommand{\cocoRepeats}{5}
\newcommand{\cocoScene}{512}
\newcommand{\cocoGrid}{32}
\newcommand{\cocoMatchedT}{5}
\newcommand{\cocoFlagshipT}{10}
\newcommand{\cocoFlagshipEpochs}{140}
\newcommand{\cocoFlagshipLastT}{20}
\newcommand{\inkEffRuns}{$n = 10$ runs, or 11 for LP-FT at canvas $8^2$}
\newcommand{\inkEffLpftSmallGrid}{8}
\newcommand{\rlHeader}{\textbf{$t{=}1$} & \textbf{$t{=}2$} & \textbf{$t{=}3$} & \textbf{$t{=}4$}}
\newcommand{\rlLastT}{4}
\newcommand{\rlTrainingRuns}{8}
\newcommand{\rlCtwofN}{10}
\newcommand{\rlScales}{0.5, 0.25}
\newcommand{\rlCenters}{16}
\newcommand{\rlCandidates}{512}
\newcommand{\rlGrid}{64}
\newcommand{\rlFeatureGrid}{32}
\newcommand{\adaglimpseFlopGrid}{3}
\newcommand{\rlGlimpses}{5}
\newcommand{\rlBatch}{16}
\newcommand{\rlSteps}{8{,}000}
\newcommand{\rlPolicyGlimpses}{640{,}000}
\newcommand{\rlPretrainingSharePct}{0.06}
\newcommand{\rlLr}{2 \times 10^{-4}}
\newcommand{\rlWeightDecay}{1 \times 10^{-2}}
\newcommand{\rlWarmupSteps}{1{,}000}
\newcommand{\rlBetaOne}{0.9}
\newcommand{\rlBetaTwo}{0.95}
\newcommand{\rlGradClip}{1}
\newcommand{\rlValInterval}{1{,}000}
\newcommand{\rlEpsilon}{0.5}
\newcommand{\rlEntropyFloor}{1}
\newcommand{\rlMinutesBound}{95}
\newcommand{\rlDevice}{RTX~4090}
\newcommand{\rlDeviceMemoryGB}{24}
\newcommand{\computeSlurmHours}{285}
\newcommand{\computeTpuVsixeChipHours}{3{,}431}
\newcommand{\computeTpuVfivepChipHours}{200}
\newcommand{\computeTpuUsd}{4{,}796}
\newcommand{\computeAprilChipHours}{1{,}318}
\newcommand{\computeRtxHours}{369}
\newcommand{\pretrainEpochs}{10}
\newcommand{\onlineTeacherIncreasePct}{41}
\newcommand{\pretrainHours}{166}
\newcommand{\teacherExportHours}{8}
\newcommand{\pretrainTotalHours}{174}
\newcommand{\onlineTeacherExportHours}{80}
\newcommand{\onlineTeacherTotalHours}{246}
\newcommand{\teacherFeaturesTiB}{19}



\newcommand{\abldelta}[2]{\csname abl#1#2\endcsname}

\usepackage{acro}
\DeclareAcronym{ACV}{short=ACV, long=Active Computer Vision}
\DeclareAcronym{AVFM}{short=AVFM, long=Active-Vision Foundation Model, long-plural-form=Active-Vision Foundation Models}
\DeclareAcronym{ViT}{short=ViT, long=Vision Transformer}
\DeclareAcronym{ANN}{short=ANN, long=Artificial Neural Network}
\DeclareAcronym{MAE}{short=MAE, long=Masked Autoencoder}
\DeclareAcronym{MLP}{short=MLP, long=Multi-Layer Perceptron, long-plural-form=Multi-Layer Perceptrons}
\DeclareAcronym{CNN}{short=CNN, long=Convolutional Neural Network}
\DeclareAcronym{RoPE}{short=RoPE, long=Rotary Position Embeddings}
\DeclareAcronym{BPTT}{short=BPTT, long=Backpropagation Through Time}
\DeclareAcronym{SDPA}{short=SDPA, long=Scaled Dot Product Attention}
\DeclareAcronym{RL}{short=RL, long=Reinforcement Learning}
\DeclareAcronym{ROI}{short=ROI, long=region of interest, long-plural-form=regions of interest}
\DeclareAcronym{LLM}{short=LLM, long=Large Language Model}
\DeclareAcronym{FLOP}{short=FLOP, long=floating point operation}
\DeclareAcronym{RFF}{short=RFF, long=Random Fourier Features}
\DeclareAcronym{PCA}{short=PCA, long=Principal Component Analysis}

\newcommand{\Dcan}{D_{\text{can}}}
\newcommand{\Dbb}{D_{\text{bb}}}
\newcommand{\Ncan}{N_{\text{can}}}

\newcommand{\Ztgt}{\mathbf{Z}^*}
\newcommand{\ztgt}{\mathbf{z}^*}
\newcommand{\Ztrecon}{\hat{\mathbf{Z}}_t}
\newcommand{\ztrecon}{\hat{\mathbf{z}}_t}

\newcommand{\papertitle}{CanViT: Toward Active-Vision Foundation Models}
\title{\papertitle}
\definecolor{linkBlue}{RGB}{0,51,153}
\hypersetup{
  pdftitle={\papertitle},
  pdfauthor={Yoha\"i-Eliel Berreby, Sabrina Du, Audrey Durand, B. Suresh Krishna},
  colorlinks=true, linkcolor=linkBlue, citecolor=linkBlue, urlcolor=linkBlue,
}

\author{%
  Yoha\"i-Eliel Berreby$^{1,2,4}$\quad
  Sabrina Du$^{1,2,4}$\quad
  Audrey Durand$^{2,3}$\quad
  B.~Suresh Krishna$^{1}$ \\[3pt]
  {\normalfont $^{1}$McGill University\quad $^{2}$Mila -- Quebec AI Institute\quad $^{3}$Universit\'e Laval\quad $^{4}$Lamarck Labs} \\[3pt]
  {\normalfont \texttt{me@yberreby.com}\quad \texttt{sabrina.du@mail.mcgill.ca}} \\
  {\normalfont \texttt{audrey.durand@ift.ulaval.ca}\quad \texttt{suresh.krishna@mcgill.ca}} \\ \noalign{\vskip 8.6pt}%
  {\normalfont Homepage: {\hypersetup{pdfborder={0 0 0}, urlcolor=magenta}\url{https://m2b3.github.io/CanViT/}}}%
}

\newcommand{\tightenparagraph}{\looseness=-1\relax}

\begin{document}

\doparttoc
\faketableofcontents

\maketitle
\vspace{-18.6pt}%

\begin{abstract}
Active computer vision promises efficient, biologically plausible perception through sequential, localized glimpses,
but lacks scalable general-purpose architectures and pretraining pipelines,
leaving Active-Vision Foundation Models (AVFMs) underexplored.
We introduce the Canvas Vision Transformer (CanViT), the first task- and policy-agnostic AVFM.
CanViT uses scene-relative RoPE to bind
a retinotopic Vision Transformer backbone
and a spatiotopic scene-wide latent workspace, the \emph{canvas}.
Efficient interaction with this high-capacity working memory
is supported by Canvas Attention,
a novel asymmetric cross-attention mechanism.
We decouple \emph{thinking} (backbone-level) and \emph{memory} (canvas-level),
eliminating canvas-side self-attention and fully-connected layers
to achieve fast sequential inference and scalability to high output resolutions.
We propose a label-free active vision pretraining scheme,
policy-agnostic passive-to-active dense latent distillation:
reconstructing scene-wide DINOv3 embeddings
from sequences of low-resolution glimpses with randomized locations, zoom levels, and lengths.
We pretrain CanViT-B from a random initialization
on 13.2 million ImageNet-21k scenes and 1 billion random glimpses,
in 166 hours on a single H100.
On ADE20K segmentation,
a frozen CanViT-B achieves \adeSingleGlimpseMiou\% mIoU
in a single low-resolution glimpse,
outperforming the best active model's \adeBestPriorMiou\%
with \adeCheapestBeatFlopRatio{}$\times$ fewer inference FLOPs
as well as its FLOP- or input-matched DINOv3 teacher.
Given additional glimpses, CanViT-B reaches \adeBestMiou\% ADE20K mIoU.
On ImageNet-1k classification,
CanViT-B also sets a new active-vision state of the art,
with \inkFinetunedBest\% top-1 accuracy after fine-tuning.
CanViT generalizes to longer rollouts, larger scenes, and new policies.
Our work narrows the wide gap between passive and active computer vision,
demonstrating the potential of task- and policy-agnostic AVFM pretraining.\tightenparagraph%
\end{abstract}

\section{Introduction}
\label{sec:introduction}

Deep \acp{ANN} have achieved outstanding performance in a variety of computer vision tasks,
and proven valuable to life sciences research
as computational models of biological visual processing
\cite{yaminsPerformanceoptimizedHierarchicalModels2014,
yaminsUsingGoaldrivenDeep2016,
schrimpfBrainScoreWhichArtificial2018,
zhuangUnsupervisedNeuralNetwork2021,
bakhtiariFunctionalSpecializationVisual2021,
mehrerEcologicallyMotivatedImage2021,
raugelDisentanglingFactorsConvergence2025}.
This practical and scientific success has largely hinged on
vision encoders which process individual frames independently and passively,
without the ability to reuse previous computation to guide further processing in a recurrent, active, human-like manner.

Unlike most \acp{ANN},
humans sample their visual environment
by actively and frequently orienting their sensory apparatus towards \acp{ROI},
through gaze shifts.
This process is inherently sequential,
and involves strategic planning \cite{yarbusEyeMovementsVision1967,hoppeMultistepPlanningEye2019},
integration of evidence across time in visual working memory \cite{BADDELEY197447,melcherPersistenceVisualMemory2001},
and top-down recurrent feedback---rich conditioning of early visual pathways by signals from higher brain areas, allowing each fixation's visual information to be processed in a contextually informed manner, based on what was seen earlier
\cite{raoPredictiveCodingVisual1999,
gilbertTopdownInfluencesVisual2013,
karEvidenceThatRecurrent2019}, rather than \emph{tabula rasa}.

\begin{figure}[t]
  \centering
  \includegraphics[width=\linewidth]{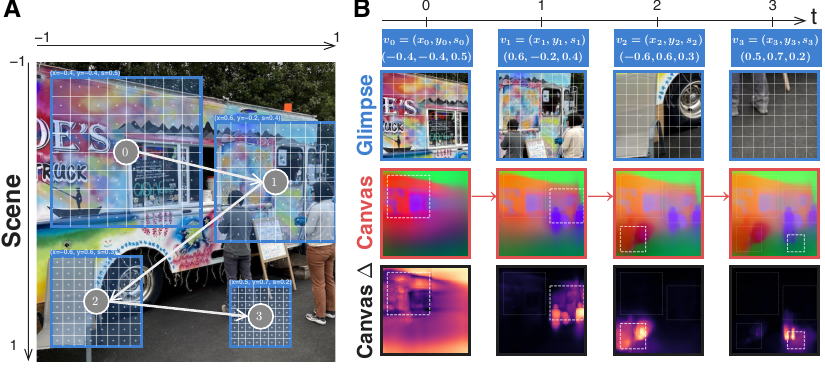}
  \caption{\textbf{A CanViT rollout.}
    We consider a high-resolution scene (\textbf{A}).
    At each timestep $t$,
    CanViT ingests a $128^2\,\text{px}$ \emph{glimpse} (\textbf{B}, 1st row),
    a crop extracted at a \emph{viewpoint} with center $(x_t, y_t) \in [-1,+1]^2$ and scale (zoom level) $s_t \in (0,1]$.
    This updates a scene-wide latent representation, the \emph{canvas},
    with which CanViT integrates broad context and fine detail from variable-scale glimpses,
    extrapolates to unobserved regions,
    and conditions visual processing on a working understanding of the scene.
    We visualize the canvas via Principal Component Analysis across tokens (\textbf{B}, 2nd row),
    and canvas updates $(\Delta)$ as cosine dissimilarity heatmaps across consecutive timesteps (\textbf{B}, 3rd row).
  }
  \label{fig:example-canvit-rollout}
  \vspace{-1em}%
\end{figure}

Drawing inspiration from human vision,
various \ac{ACV} models, which process visual scenes through sequential, localized glimpses, have been proposed
\cite{mnihRecurrentModelsVisual2014,
DRAMBaMK14,
matheReinforcementLearningVisual2016,
ablavatski2017enriched,
elsayedSaccaderImprovingAccuracy2019,
wangGlanceFocusDynamic2020,
papadopoulosHardAttentionScalableImage2021,
seifiGlimpseAttendandExploreSelfAttentionActive2021,
liuImprovedEfficiencyBased2023,
AMEPardylRK0T23,
pardylAdaGlimpseActiveVisual2025,
wangEmulatingHumanlikeAdaptive2025,
pourrahimiNeuralSignaturesAssociational2025,
ibrayevExploringFoveationSaccade2024,
ibrayevTwoStreamFoveationBasedActive2024,
renEndToEndInstanceSegmentation2017,
liModelingHumanEye2023,
olszewskiTORETokenRecycling2025}.
These models often seek biological relevance and computational efficiency.
Yet, despite their theoretical advantages, they have struggled to match the accuracy, efficiency, flexibility and representational richness of their passive counterparts.
This disconnect is particularly striking on dense prediction,
which is unsupported by most existing active vision models;
the few exceptions lag dramatically behind passive models~\cite{yangFocalAttentionLongRange2021,yangFocalModulationNetworks2022}
on standard benchmarks like ADE20K segmentation~\cite{zhouSceneParsingADE20K2017} when evaluated on it~\cite{AMEPardylRK0T23,pardylAdaGlimpseActiveVisual2025}.

An active vision model can be considered along three axes:
making sense of what it sees at a given moment (instantaneous vision);
updating a persistent, evolving understanding of the scene,
which may in turn inform further processing (memory);
and deciding where and at what zoom level to look next (action selection).
The first two axes define an observer's ability to understand, persist, and recall visual inputs,
while the last defines its input-selection strategy, or sensory policy.
They do not play comparable roles:
given perfect instantaneous vision and memory,
naive and strategic policies should reach the same accuracy after enough glimpses of a static scene,
but no policy can make up for a poor observer.
Yet, prior \ac{ACV} work has often focused on action selection
\cite{mnihRecurrentModelsVisual2014,
DRAMBaMK14,
matheReinforcementLearningVisual2016,
ablavatski2017enriched,
elsayedSaccaderImprovingAccuracy2019,
wangGlanceFocusDynamic2020,
papadopoulosHardAttentionScalableImage2021,
seifiGlimpseAttendandExploreSelfAttentionActive2021,
liuImprovedEfficiencyBased2023,
AMEPardylRK0T23,
pardylAdaGlimpseActiveVisual2025,
wangEmulatingHumanlikeAdaptive2025,
pourrahimiNeuralSignaturesAssociational2025,
ibrayevTwoStreamFoveationBasedActive2024}.
Concurrently, passive-vision foundation models like DINOv3 \cite{simeoniDINOv32026}
excel at general-purpose instantaneous vision
but lack essential active-vision components,
as they have no concept of localized glimpses or memory.

Focusing on the intersection of instantaneous vision and memory,
we sought to build a task- and policy-agnostic \ac{AVFM}:
a vision model capable of understanding the spatial and semantic structure of scenes
across arbitrary sequences of glimpses,
with rich representations that transfer across tasks and viewing policies.
This policy-agnostic approach disentangles ``how to see in an active-vision setting'' from ``where to look,''
thus freeing active-vision pretraining from the complexity of \ac{RL}.
Policy agnosticism may also be desirable for real-world deployment,
where physical limitations (e.g., a motorized camera's range of motion and optical zoom levels) may preclude the use of a general-purpose policy.

\textbf{Contributions.}
We introduce the Canvas Vision Transformer (CanViT, Figure~\ref{fig:example-canvit-rollout}),
a recurrent architecture
built around a latent scene-wide representation called the \emph{canvas} (\S\ref{sec:arch})
and pretrained with a novel passive-to-active distillation scheme (\S\ref{sec:pretraining}).
We evaluate CanViT-B
on ADE20K semantic segmentation
and ImageNet-1k~\cite{russakovskyImageNetLargeScale2015} classification (\S\ref{sec:experiments}).
CanViT-B sets a new state of the art among active vision models on both ADE20K and ImageNet-1k
while transferring across policies, temporal horizons, and canvas resolutions without retraining.

\section{Related work}
\label{sec:related-work}

\textbf{Deep active vision} models typically process sequences of \emph{glimpses}---fixed- or variable-scale crops extracted from a larger image or video.
This line of research traces back to the Recurrent Attention Model~\cite{mnihRecurrentModelsVisual2014},
and remained largely confined to simple tasks such as digit recognition~\cite{DRAMBaMK14,ablavatski2017enriched}
until 2019, when Saccader~\cite{elsayedSaccaderImprovingAccuracy2019}
achieved 75\% ImageNet-1k~\cite{russakovskyImageNetLargeScale2015} top-1 accuracy by
introducing an intermediate pretraining step to stabilize learning.
GFNet~\cite{wangGlanceFocusDynamic2020} and AdaptiveNN~\cite{wangEmulatingHumanlikeAdaptive2025}
later showed that active vision could deliver computational efficiency gains on real-world tasks,
although both remained structurally limited to classification tasks and fixed zoom levels.

\textbf{Dense prediction in active vision} has remained underexplored, as most \ac{ACV} architectures cannot produce the scene-wide, spatially dense outputs required in tasks like semantic segmentation or depth estimation.
Among the few exceptions, AME~\cite{AMEPardylRK0T23} and AdaGlimpse~\cite{pardylAdaGlimpseActiveVisual2025}
achieve dense prediction through post-hoc expansion of encoder outputs:
at each timestep, a MAE-style~\cite{heMaskedAutoencodersAre2022} Transformer decoder receives all encoded glimpse tokens
with a full grid of learnable mask tokens and performs self-attention over the entire grid to produce scene-wide predictions.
This becomes intractable at high scene resolutions, where active vision is most appealing.
Despite being the state of the art on active ADE20K segmentation,
these models respectively achieve only \ameSetrMiou\% and \adaglimpseEightMiou\% mIoU.

\textbf{Dense latent distillation from self-supervised \acsp{ViT}} allows the visuospatial intelligence acquired through extensive pretraining to be quickly transferred into randomly-initialized models.
The largest DINOv2~\cite{oquab2024dinov}/v3~\cite{simeoniDINOv32026}
models were distilled into smaller \acp{ViT}
using the same dataset and objective as during pretraining.
Proteus~\cite{zhang2025accessing}
distilled DINOv2-\{g,L\}/14~\cite{oquab2024dinov}
into a smaller model using 100$\times$ less data than during pretraining
and a simple loss combining CLS token matching with dense feature matching.
Our distillation scheme follows a similar philosophy to Proteus,
but transfers across problem settings rather than across model sizes:
we use a passive teacher's visual world knowledge to teach an active student
how to see from arbitrary glimpse sequences.

\textbf{Cross-attention for dimensionality/computation decoupling}
was pioneered by the Set Transformer~\cite{leeSetTransformerFramework2019},
popularized by Perceiver models~\cite{jaeglePerceiverGeneralPerception2021,jaeglePerceiverIOGeneral2021},
then generalized to iterative bidirectional routing by Recurrent Interface Networks (RINs)~\cite{jabriScalableAdaptiveComputation2023}.
Like RINs, CanViT alternates read and write cross-attention across depth and time.
Unlike Perceiver and RINs, CanViT ingests external input on its few-token side (glimpse), with latents on the many-token side (canvas).
Moreover, CanViT's many-token side forgoes not only self-attention, but also all fully-connected layers:
canvas tokens never go through \acp{MLP},
QKVO projections in cross-attention,
or GRU~\cite{chungEmpiricalEvaluationGated2014}/LSTM~\cite{hochreiterLongShortTermMemory1997} recurrent gates.
Instead, canvas tokens evolve solely via our Canvas Attention mechanism.

\textbf{Latent-space recurrent reasoning} iterates over a fixed external input in a weight-tied fashion,
decoupling representational capacity from computational depth
to enable improved algorithmic reasoning~\cite{dehghaniUniversalTransformers2018,yangLoopedTransformersAre2023,saunshiReasoningLatentThoughts2024,wangHierarchicalReasoningModel2025,jolicoeur-martineauLessMoreRecursive2025}
and flexible test-time compute allocation~\cite{gravesAdaptiveComputationTime2017,baninoPonderNetLearningPonder2021}.
This paradigm has recently seen renewed interest,
both in the \ac{LLM} space~\cite{haoTrainingLargeLanguage2025,geipingScalingTestTimeCompute2025}
and for its ability to produce small yet highly capable models~\cite{wangHierarchicalReasoningModel2025,jolicoeur-martineauLessMoreRecursive2025}.
Leveraging top-down recurrent feedback allows previous computation to be reused,
without losing most of it to ephemeral step-wise outputs.
Active vision offers an opportunity to generalize this framework
by allowing each processing step to benefit from a new perspective on the scene
instead of operating on a constant input.
CanViT achieves latent-space recurrent reasoning
with a semantically rich (rather than pixel-like) workspace, the \emph{canvas},
which provides top-down recurrent feedback for early layers to build upon.

\section{Preliminaries}
\label{sec:preliminaries}

\textbf{Definitions: scenes, glimpses, viewpoints.}
For a timestep $t \in \mathbb{N}$,
we consider a bounded 2D \emph{scene},
which can be formally represented as
a function $\psi_t : [-1, +1]^2 \to \mathbb{R}^3$
mapping continuous-valued $(x,y)$ coordinates to RGB values.
In our experiments, we consider static scenes,
with $\psi_t = \psi_0$ for all timesteps.
We define a \emph{glimpse}
as a fixed-resolution crop,
extracted at a \emph{viewpoint} $\mathbf{v}_t = (x_t, y_t, s_t)$,
where $(x_t, y_t) \in [-1,+1]^2$ is the crop center in scene coordinates
and $s_t \in (0, 1]$ is the crop's \emph{scale}, or half-side-length.
That is, the crop spans the scene coordinates
$[x_t - s_t, x_t + s_t] \times [y_t - s_t, y_t + s_t]$,
covering a fraction $s_t^2$ of the scene's surface area.
Regardless of its scale $s_t$,
the crop is resized to a fixed resolution of $H_g \times W_g$ pixels,
which determines its information capacity;
under that constraint, $s_t$ smoothly controls the tradeoff between spatial coverage and perception of detail.

\textbf{General-purpose, spatially grounded active vision.}
We wish for our model to build up a general-purpose understanding of the \emph{scene}
through a sequence of \emph{glimpses} that it perceives,
allowing each additional processing step
to build upon and refine this understanding.
This evolving latent representation should be readily decodable into predictions at every timestep,
whether for non-spatial, global prediction tasks like object classification
or spatially-grounded, dense tasks like semantic segmentation or depth estimation.
Dense predictions require explicit architectural handling,
as the active vision setting breaks the direct,
connectivity-based mapping between input and output feature maps
that passive \acp{CNN} and \acp{ViT} commonly exploit:
glimpses need not align with the scene from which they are extracted.

To address this problem, we introduce the \textbf{Canvas Vision Transformer}, or \textbf{CanViT} (\S\ref{sec:arch}),
a recurrent vision architecture built around a scene-wide latent representation called the \textbf{canvas}.

\section{The Canvas Vision Transformer (CanViT)}
\label{sec:arch}

\begin{figure}[!ht]
  \centering
  \includegraphics[width=\linewidth]{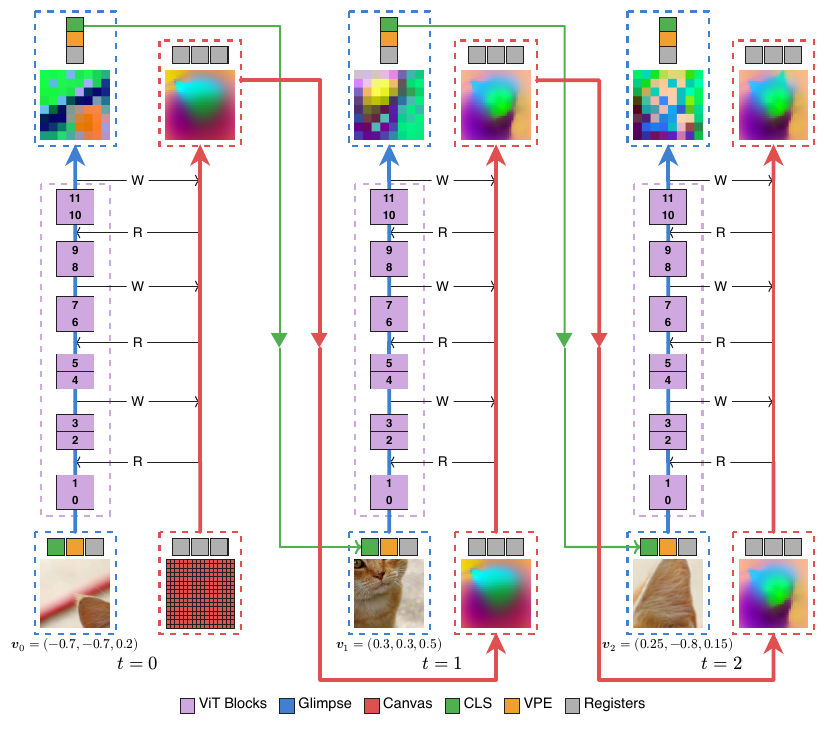}
  \caption{\textbf{CanViT architecture diagram.}
    We adopt a dual-stream structure, equipping a \textbf{ViT backbone} (\textcolor{archBackbone}{\textbf{purple}}, left-hand columns), which processes localized \textbf{glimpses} (\textcolor{archGlimpse}{\textbf{blue}}), with a \textbf{canvas} (\textcolor{archCanvas}{\textbf{red}}, right-hand columns), a fine-grained scene-wide spatio-semantic memory.
    At each timestep $t$, a glimpse is extracted from a viewpoint $\mathbf{v}_t = (x_t, y_t, s_t)$, patchified, and processed through the backbone, alongside a recurrent CLS token (\textcolor{archCls}{\textbf{green}}) and a Viewpoint Encoding (VPE) token (\textcolor{archVpe}{\textbf{orange}}).
    The canvas regularly interacts with the glimpse stream via \textbf{Canvas Attention} (Figure~\ref{fig:canvas-attention}), alternating between Read (R) and Write (W) operations
    to, respectively, condition the backbone's processing on the canvas and populate the canvas.
    Both streams are equipped with register tokens (\textcolor{archReg}{\textbf{gray}})~\cite{darcetVisionTransformersNeed2023}.
  }
  \label{fig:canvit-high-level-arch}
\end{figure}

The CanViT architecture (Figure~\ref{fig:canvit-high-level-arch})
formulates active-vision processing
as the interaction between a high-capacity memory stream (the \emph{canvas})
and a \ac{ViT}~\cite{dosovitskiyImageWorth16x162020} backbone's compact glimpse processing stream.
These streams interact bidirectionally through
Canvas Attention (Figure~\ref{fig:canvas-attention}),
a mechanism that allows the backbone to efficiently pull information from the canvas and send updates to it.

The \textbf{glimpse stream}, made up of $\Dbb$-dimensional glimpse tokens, is largely ephemeral.
Each glimpse is extracted from the scene at a given viewpoint,
split into $16^2\,\text{px}$ patches,
and fed to the backbone alongside
ephemeral register tokens~\cite{darcetVisionTransformersNeed2023},
a recurrent CLS token,
and a viewpoint encoding (VPE) token derived from the glimpse's position and zoom level.
Since spatial alignment between consecutive patch grids cannot be assumed,
as a glimpse may be taken from any position and at any zoom level,
patch tokens cannot be directly forwarded across time without destroying their grid structure.
Instead, CanViT persists relevant information from glimpses via the canvas stream.

The \textbf{canvas stream}, made up of $\Dcan$-dimensional canvas tokens, is fully persistent.
The canvas acts as a scene-wide spatio-semantic memory,
which can function as a cognitive map~\cite{tolmanCognitiveMapsRats1948} of the scene.
It comprises a few \emph{canvas registers}, which act as a non-spatial memory,
and a large $H \times W$ spatial grid of \emph{canvas patches}
tiling the $[-1, +1]^2$ scene coordinate space.
At the start of each rollout, this grid is broadcasted to the desired size
from a single learnable initial patch,
enabling the canvas resolution to be set at inference time.
Each canvas patch maps onto a fixed scene region, regardless of the current viewpoint,
thus allowing direct token-wise decoding into dense predictions
and unbroken gradient flow across time.
After initialization,
the canvas is read from and written to by consecutive Canvas Attention layers,
whose outputs are residuals injected into each stream.
No \acp{MLP} or self-attention layers are applied to canvas tokens,
which evolve solely by interacting with the backbone's glimpse tokens via Canvas Attention Write operations.
This restriction is key to CanViT's efficiency, as the canvas stream is designed to accommodate a much larger number of tokens than the glimpse stream.

\textbf{Scene-Relative Rotary Position Embeddings (SR-RoPE).}
We compute 2D \acs{RoPE}~\cite{suRoFormerEnhancedTransformer2024,heoRotaryPositionEmbedding2025}
from the centers of glimpse patches and canvas patches
in the scene's $[-1,+1]^2$ coordinate system,
both in the ViT backbone's self-attention and in Canvas Attention layers.
The positions of glimpse patch centers depend on the current viewpoint $(x_t, y_t, s_t)$,
with their relative distances implicitly communicating the viewing scale $s_t$ (zoom):
a zoomed-in glimpse's patches are tightly clustered in a small scene region,
while a zoomed-out glimpse's tokens span a wider region (Figure~\ref{fig:example-canvit-rollout}).
The positions of canvas patches are constant for any given canvas grid size,
since they uniformly tile the scene.
Consistent use of scene-relative coordinates
binds the retinotopic glimpse stream and the spatiotopic canvas stream
with a shared reference frame,
while \ac{RoPE} enables generalization across patch grid sizes.

\textbf{Canvas Attention} (Figure~\ref{fig:canvas-attention}\textbf{A}; pseudocode in Appendix~\ref{supp:canvas-attention-pseudocode}),
a mechanism based on cross-attention,
enables efficient interaction between CanViT's relatively small set of glimpse tokens
and its much larger set of canvas tokens,
allowing CanViT to support fine-grained scene/canvas grids without incurring the quadratic cost of self-attention (Figure~\ref{fig:canvas-attention}\textbf{B}).
We alternate between \emph{Read} and \emph{Write} operations
along depth (and, implicitly, time)
using a \emph{stride} of 2 ViT blocks in CanViT-B.
In a \emph{Read}, glimpse tokens query the canvas;
in a \emph{Write}, canvas tokens query the glimpse.
In both cases, the cross-attention output is added back to the querying stream via residual addition.
SR-RoPE makes this process spatially aware,
allowing Canvas Attention to bind the two streams.
Unlike standard cross-attention implementations
(e.g.,~multi-head attention in PyTorch~\cite{anselPyTorch2Faster2024} or Flax NNX~\cite{jax2018github,flax2020github}),
Canvas Attention is \emph{asymmetric}: it restricts Query, Key, Value and Output (QKVO) projections
to one token set (glimpse-side tokens),
applying only LayerNorm, RoPE (for Queries/Keys) and element-wise residual addition
to the other one (canvas-side tokens).

\textbf{Asymmetric projections.}
The glimpse token count $N_g$ must be kept low due to
the quadratic cost of the ViT backbone's self-attention,
and it is desirable for the canvas to have high information capacity
both \emph{spatially}, by tiling the scene in a fine-grained manner with a large number $H \times W$ of canvas patch tokens (making up the bulk of the $\Ncan$ canvas tokens),
and \emph{semantically} at any given position within the scene,
with a large canvas embedding dimension $\Dcan$.
This makes the asymmetric design of Canvas Attention highly advantageous,
as the FLOP footprint of a single canvas-side projection relative to the \ac{SDPA} call that it accompanies would be:
\begin{equation}
  \frac{\text{projection FLOPs}}{\text{SDPA FLOPs}}
  = \frac{2 \Ncan \Dcan^2}{4 N_g \Ncan \Dcan}
  = \frac{\Dcan}{2 N_g}\,,
\end{equation}
which is a $\projSdpaRatio{}\times$ ratio with $\Dcan = \flopsCvCanvasDim{}$ and $N_g = \flopsCvNLocal{}$ (\flopsCvNPatches{} glimpse patches, \flopsCvBackboneRegs{} registers, and CLS + VPE tokens).
The FLOP savings from asymmetric cross-attention projections
are magnified when using fewer glimpse patches and more canvas patches (Figure~\ref{fig:canvas-attention}\textbf{C}).
In CanViT-B,
with a $\cSixFourCanvasGrid{} \times \cSixFourCanvasGrid{}$ canvas patch grid, adding canvas-side QKVO projections would increase the cost of each Canvas Attention Read/Write pair
from \cSixFourAsymRwPairGf{} to \cSixFourFullRwPairGf{} GFLOPs.

\begin{figure}[t]
  \centering
  \includegraphics[width=\linewidth]{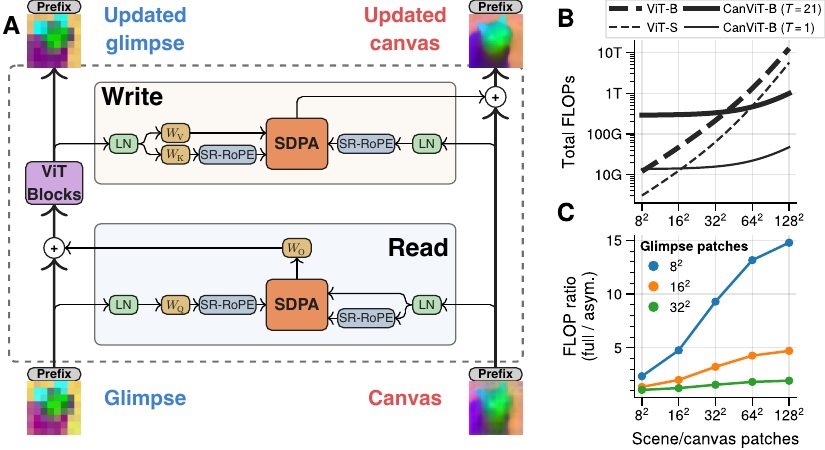}
  \caption{\textbf{A.}~A Canvas Attention Read-Write pair.
  \textbf{B.}~Total inference FLOPs vs.\ output patch grid.
  \textbf{C.}~Relative cost of canvas-side QKVO projections, per Canvas Attention Read/Write pair.}
  \label{fig:canvas-attention}
  \vspace{-4pt}%
\end{figure}

\textbf{Viewpoint Encoding (VPE) token.}
We supplement SR-RoPE,
which distributes the encoding of viewpoint position and scale over glimpse patches,
with a dedicated viewpoint encoding (VPE) token
that concentrates the viewpoint triplet $(x_t, y_t, s_t)$
into a single glimpse-side token.
The VPE token is a non-essential architectural affordance,
meant to facilitate future end-to-end policy learning
by allowing the next viewpoint to be decoded from a rich transformation of the current viewpoint.
We provide additional details on VPE in Appendix~\ref{supp:vpe}
and ablate it in Tables~\ref{tab:ablations} and~\ref{tab:ablation-downstream} and Appendix~\ref{supp:vpe-long-horizon}.

\section{Policy-agnostic passive-to-active dense latent distillation}
\label{sec:pretraining}

An active vision model should (1) make sense of what it sees,
(2) integrate observations into a persistent scene representation,
and (3) decide where to look next.
Here, we ask: how can we teach CanViT (1) and (2)
in a task-agnostic, spatially aware, label-free manner,
while remaining robust to the choice of policy (3)?
We achieve this via
\emph{policy-agnostic passive-to-active dense latent distillation},
which combines highly informative reconstruction targets with rollout randomization.

\subsection{Passive-to-active dense distillation}
\label{subsec:distillation}

A natural pretext task for active-vision pretraining is \emph{scene reconstruction}:
training the model to produce a best-guess approximation of the overall scene from a sequence of glimpses.
This incentivizes understanding of the spatial and semantic structure of scenes,
in order to faithfully extrapolate to unseen regions and details.
Since we wish for CanViT to iteratively build a semantically rich ``mental image,''
we formulate this reconstruction objective in DINOv3 \cite{simeoniDINOv32026} latent space,
rather than in pixel space like passive \acp{MAE} \cite{heMaskedAutoencodersAre2022}
and some active vision models \cite{pardylAdaGlimpseActiveVisual2025}.

\textbf{Teacher.}
We build upon DINOv3 \cite{simeoniDINOv32026},
a self-supervised vision model family
whose dense representations deliver outstanding frozen-feature transfer
to classification, segmentation, depth estimation, and other downstream tasks.
We use DINOv3 ViT-B as a \emph{high-resolution scene-wide spatio-semantic teacher},
which produces patch tokens (dense features) and a global CLS token.
These highly informative reference embeddings represent an idealized scene understanding
for CanViT to match using sequences of cheap, partial glimpses.
We apply per-position z-score standardization to reconstruction targets.
High-resolution teacher inference requires more computation than a CanViT forward pass;
however, since the teacher is frozen,
we precompute reference features once,
storing them for subsequent use across epochs and hyperparameter sweeps.
This is an optional tradeoff between storage and compute cost; we estimate that computing teacher features online would increase total training compute by approximately \onlineTeacherIncreasePct\%, while reducing storage from \teacherFeaturesTiB{}~TiB to zero.

\textbf{Decoding.} At each timestep $t$,
CanViT produces updated canvas tokens,
including canvas patches $\mathbf{C}_t \in \mathbb{R}^{H \times W \times \Dcan}$ and canvas registers,
and an updated CLS token $\mathbf{h}_t \in \mathbb{R}^{\Dbb}$.
We apply layer normalization \cite{baLayerNormalization2016},
then decode into DINOv3-space reconstructions via token-wise linear projections:
\begin{equation}
  \Ztrecon = W_C \cdot \text{LayerNorm}(\mathbf{C}_t), \quad
  \ztrecon = W_h \cdot \text{LayerNorm}(\mathbf{h}_t)
\end{equation}

\textbf{Loss.} Our loss combines patch-level and CLS-level reconstruction, averaged across space, teacher embedding dimensions ($D$), and time:
\begin{equation}
  \mathcal{L} = \frac{1}{T} \sum_{t=0}^{T-1}
  \left[
    \frac{1}{HWD} \|\Ztrecon - \Ztgt\|_F^2 + \frac{1}{D} \|\ztrecon - \ztgt\|^2
  \right]
\end{equation}

\subsection{Policy agnosticism}
\label{subsec:policy-agnosticism}

\textbf{Dual rollouts.}
For each scene, we run two independent rollouts from a freshly-initialized canvas, averaging their losses.
These rollouts and their viewpoint sampling policies are referred to as R-IID and F-IID,
which only differ by their behavior at $t=0$.
R-IID (Random-then-IID) rollouts sample random viewpoints at all timesteps, including $t=0$,
ensuring robustness to arbitrary rollout starts.
F-IID (Full-then-IID) rollouts always start with the full-scene zoomed-out viewpoint $(x_0,y_0,s_0) = (0,0,1)$,
providing a low-resolution but high-spatial-coverage view of each scene at least once over the course of pretraining.
Training with 1 F-IID + 1 R-IID rather than 2 R-IID rollouts accelerates convergence,
even when evaluating on held-out data with an R-IID policy (Table~\ref{tab:ablations}\,\textbf{h}).

\textbf{Sampling of viewpoint center and scale.}
For all non-initial timesteps, as well as the $t = 0$ timestep of R-IID rollouts,
we sample $A \sim \mathcal{U}([0, A_{\max}])$ and set $s = 1 - \sqrt{A}$.
For a glimpse of scale $s$,
valid centers form a box of half-side-length $\sqrt{A}$,
within which we draw $(x, y)$ uniformly.
The resulting marginal scale density is $p(s) \propto (1 - s)$,
favoring small, localized glimpses.
We set $A_{\max} = (1 - s_{\min})^2$ with a minimum glimpse scale of $s_{\min} = 0.05$, or $0.25\%$ of the scene's area.

\textbf{Rollout length randomization.}
Our loss provides \emph{temporally} dense supervision, with per-timestep credit assignment.
This allows us to train CanViT with truncated \ac{BPTT} over chunks of only $K = 2$ glimpses.
At each chunk boundary, we stop the rollout with probability $p_{\text{stop}} = 0.5$,
resulting in a geometric distribution of chunk counts
with an \emph{average} sequence length of $T = K / p_{\text{stop}} = 4$ glimpses
while occasionally exposing the model to longer sequences.
This scheme ensures sequence length robustness on a constant train-time VRAM footprint,
with a modest train-time compute overhead due to the low average sequence length.

\section{Experiments}
\label{sec:experiments}

\begin{figure}[t]
  \centering
  \includegraphics[width=\linewidth]{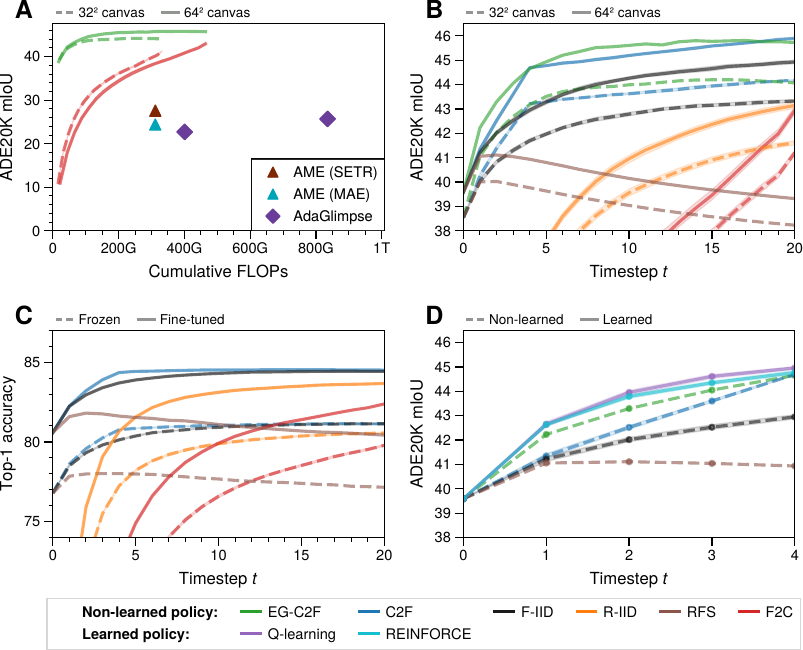}
  \caption{\textbf{Benchmark results on ADE20K and ImageNet-1k}.
    \textbf{A.}~Accuracy--efficiency frontier on ADE20K segmentation.
    \textbf{B.}~ADE20K segmentation mIoU by non-learned policy ($32^2$ or $64^2$ canvas).
    \textbf{C.}~\mbox{ImageNet-1k} classification accuracy by non-learned policy, frozen or fine-tuned ($32^2$ canvas).
    \textbf{D.}~Learned vs.\ non-learned policies on short-horizon ADE20K segmentation ($\rlGrid^2$ canvas).}
  \label{fig:main-ade20k-in1k-results}
\end{figure}

To assess generalization across tasks, policies, temporal horizons, and canvas resolutions,
we pretrain a general-purpose CanViT-B checkpoint and evaluate it with linear decoding
on ADE20K~\cite{zhouSceneParsingADE20K2017} semantic segmentation
and ImageNet-1k (IN1k) classification~\cite{russakovskyImageNetLargeScale2015}.
We use frozen weights to assess feature decodability on ADE20K and IN1k,
and full fine-tuning to measure peak IN1k performance;
we do not fine-tune on ADE20K due to its low scene count.
We use $128^2\,\text{px}$ glimpses consistently.
We provide additional details on pretraining (Appendix~\ref{supp:pretraining-details}) and evaluation (Appendix~\ref{supp:evaluation-details}), as well as a detailed ablation study (Appendix~\ref{supp:ablation-study}), inference latency measurements (Appendix~\ref{supp:latency-experiments}) and a visualization of canvas updates within and across glimpses (Appendix~\ref{supp:interpretability}).
We also report additional experiments in Appendix~\ref{supp:additional-experiments}, covering the impact of pretraining on ImageNet-1k instead of ImageNet-21k, the relevance of the VPE token at long horizons and high canvas resolutions, ImageNet-1k accuracy against inference cost compared with AdaptiveNN and GFNet, an architecture control against AdaGlimpse with the same pretraining data and objective, and COCO~\cite{linMicrosoftCOCOCommon2014} object detection.

\textbf{Task-agnostic pretraining.}
Following the protocol described in \S\ref{sec:pretraining},
we pretrain CanViT-B from a random initialization
in just 166 hours on a single H100,
sampling approximately 1 billion glimpses
from 13.2 million $512^2\,\text{px}$ ImageNet-21k~\cite{russakovskyImageNetLargeScale2015,ridnikImageNet21KPretrainingMasses2021} scenes.
We use an average sequence length of $T = 4$ (via $p_{\mathrm{stop}} = 0.5$ and $K = 2$) and a canvas resolution of $32^2 = 1024$ patches during pretraining.

\textbf{Task-specific probing.}
For ADE20K, we train linear probes to predict segmentation masks from canvas patches.
We train a separate probe for each considered canvas resolution.
For ImageNet-1k, we use a linear probe to predict logits from CanViT's recurrent CLS token.
When fine-tuning, we start from our frozen-weights probes and unfreeze CanViT's own weights (i.e., LP-FT~\cite{kumarFineTuningCanDistort2022}).

\textbf{Baselines.}
We consider leading active vision models with published results on ADE20K segmentation and/or ImageNet-1k classification:
Saccader~\cite{elsayedSaccaderImprovingAccuracy2019},
AME~\cite{AMEPardylRK0T23},
AdaGlimpse~\cite{pardylAdaGlimpseActiveVisual2025},
and AdaptiveNN~\cite{wangEmulatingHumanlikeAdaptive2025}.
We report each model's best published results;
for AME, we include both its SETR~\cite{zhengRethinkingSemanticSegmentation2021} and MAE~\cite{heMaskedAutoencodersAre2022}-initialized variants for completeness.
When assessing the effects of canvas resolution and ground-truth mask area, we also compare against the DINOv3 ViT-B passive teacher.

\textbf{Non-learned policies.}
To assess CanViT's ability to generalize across policies,
we supplement the R-IID and F-IID train-time policies described in \S\ref{subsec:policy-agnosticism}
with additional inference-time-only policies.
We introduce a \textbf{Coarse-to-Fine (C2F)} policy,
which traverses a quadtree over the scene,
deterministically decreasing the viewpoint scale as the rollout progresses
while randomizing the visitation order at each scale.
To isolate the effect of processing order,
we pair C2F with the \textbf{Fine-to-Coarse (F2C)} policy,
which reverses C2F viewpoint sequences.
For ADE20K, we also introduce an \textbf{Entropy-Guided C2F (EG-C2F)} variant,
which greedily selects the highest-uncertainty tile among those that have not yet been visited at a given scale,
using the segmentation probe's per-position class entropy.
This image-dependent dynamic policy showcases the ability of the canvas to guide viewpoint selection, even without \ac{RL}.
Lastly, to disentangle the impact of additional recurrent processing from that of ingesting different inputs,
we consider a \textbf{Repeated Full-Scene (RFS)} policy,
which simply iterates over the $(x,y,s)=(0,0,1)$ zoomed-out viewpoint.
We use up to $T = 21$ glimpses.

\textbf{Learned policies.}
To test whether a task-specific viewing policy can be learned quickly on top of a frozen CanViT-B and its ADE20K probe, we train policy networks with Q-learning~\cite{watkinsQlearning1992} and \mbox{REINFORCE}~\cite{williamsSimpleStatisticalGradientfollowing1992} to maximize segmentation quality (Appendix~\ref{supp:learned-policies}).
We train each network on \rlPolicyGlimpses{} glimpses, \rlPretrainingSharePct\% of CanViT-B's pretraining glimpses, in under \rlMinutesBound{} minutes on a single \rlDevice{} GPU (\rlDeviceMemoryGB{}~GB).

\pagebreak
\subsection{Results}
\label{subsec:results}

\begin{wraptable}{r}{0.48\linewidth}
  \vspace{-1.2em}
  \makeatletter
  \setlength{\abovecaptionskip}{\@neuripsbelowcaptionskip}%
  \setlength{\belowcaptionskip}{\@neuripsabovecaptionskip}%
  \makeatother
  \caption{\textbf{Comparison with prior work.} \\ $\dagger$: Frozen weights with linear probing.}
  \label{tab:active-vision-comparison}
  \centering
  \small
  \setlength{\tabcolsep}{4pt}%
  \begin{tabular}{lcc}
    \toprule
    \textbf{Model} & \textbf{ADE20K mIoU} & \textbf{IN1k top-1} \\
    \midrule
    Saccader \cite{elsayedSaccaderImprovingAccuracy2019} & --- & 75.0 \\
    AME \cite{AMEPardylRK0T23} (MAE) & 24.4 & --- \\
    AME \cite{AMEPardylRK0T23} (SETR) & 27.6 & --- \\
    AdaGlimpse \cite{pardylAdaGlimpseActiveVisual2025} & 25.7 & 77.5 \\
    AdaptiveNN \cite{wangEmulatingHumanlikeAdaptive2025} & --- & 82.2 \\
    \midrule
    \textbf{CanViT-B} & \textbf{\adeBestMiou}$^{\dagger}$ & \inkFrozenBest$^{\dagger}$ / \textbf{\inkFinetunedBest} \\
    \bottomrule
  \end{tabular}
  \vspace{-0.5em}
\end{wraptable}

CanViT-B outperforms prior active vision models
on both ADE20K segmentation and ImageNet-1k classification by a wide margin (Table~\ref{tab:active-vision-comparison}),
achieving up to \adeBestMiou\% ADE20K mIoU (up from 27.6\%)
and \inkFinetunedBest\% IN1k top-1 accuracy (up from 82.2\%).
On ADE20K, state-of-the-art accuracy is paired with striking efficiency:
in a single glimpse, CanViT-B reaches \adeCheapestBeatMiou\% mIoU
using \adeCheapestBeatFlopRatio{}$\times$ fewer FLOPs than AME does at \adeBestPriorMiou\% mIoU
(Figure~\ref{fig:main-ade20k-in1k-results}\textbf{A}).
CanViT's advantage holds even with the F2C policy's worse-than-random viewing order,
showing that perception, not viewpoint selection, was the bottleneck in this task.

On both tasks, CanViT-B generalizes across policies, temporal horizons, and canvas resolutions, even with frozen weights (Figure~\ref{fig:main-ade20k-in1k-results}\textbf{B,C}).
We observe consistent ADE20K mIoU gains through $T = \adeMaxT$ glimpses,
well beyond the average of $T \approx 4$ used during pretraining,
while IN1k accuracy quickly plateaus.
The inference-time-only C2F policy
dramatically outperforms the R-IID and F-IID policies on ADE20K
in both peak accuracy and efficiency,
and provides a modest efficiency boost early into IN1k rollouts.
EG-C2F's uncertainty-based viewpoint selection
enables further efficiency gains on ADE20K.
ADE20K Q-learning policies outperform every non-learned policy from $t = 1$ to $t = \rlLastT$ (Figure~\ref{fig:main-ade20k-in1k-results}\textbf{D}).
On both tasks, CanViT reaches higher peak accuracies with C2F than F2C
despite both policies yielding the same set of viewpoints by $t = \adeLastT$,
highlighting the importance of contextually informed processing.
Given a constant input (RFS),
accuracy initially improves purely from recurrent processing,
then subsequently declines in the absence of diverse viewpoints.
Despite training exclusively with a $32^2$ canvas,
evaluating CanViT-B with an inference-time $64^2$ canvas
consistently provides an accuracy boost on ADE20K,
an effect that we probe further in Figure~\ref{fig:resolution-and-mask-size-analysis}.

\begin{figure}[t]
  \centering
  \includegraphics[width=\linewidth]{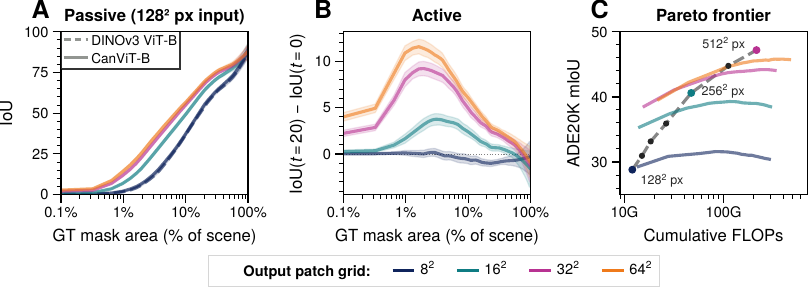}
  \caption{\textbf{Effects of canvas resolution and ground-truth mask area.}
    We evaluate on ADE20K segmentation with canvas resolution in $\{8^2, 16^2, 32^2, 64^2\}$,
    using the EG-C2F policy, frozen weights, and per-resolution linear probes.
    \textbf{A.}~Per-mask IoU given a single full-scene $128^2$\,px input ($8^2$ input patch grid).
    \textbf{B.}~Per-mask IoU gain given additional glimpses.
    \textbf{C.}~Accuracy--efficiency Pareto frontier, varying FLOPs via timestep count for CanViT and via input resolution (px) for DINOv3.
  }
  \label{fig:resolution-and-mask-size-analysis}
\end{figure}

CanViT's dual-stream design decouples glimpse (instantaneous vision) resolution from canvas (memory/output) resolution.
This input-output dissociation can be advantageous even in single-timestep (passive vision) contexts:
given a low-resolution full-scene glimpse, a frozen CanViT-B with a finer canvas outperforms its coarser-canvas counterpart on ADE20K segmentation,
and even its input-matched DINOv3 ViT-B teacher (Figure~\ref{fig:resolution-and-mask-size-analysis}\textbf{A}).
Finer canvases are most beneficial when dealing with small ground-truth segmentation masks (e.g., for small, distant or occluded objects) (Figure~\ref{fig:resolution-and-mask-size-analysis}\textbf{A})
or using multiple timesteps (Figure~\ref{fig:resolution-and-mask-size-analysis}\textbf{B}).
Together,
the accuracy gains provided by finer canvases
and the low computational overhead of Canvas Attention
allow CanViT-B to offer an attractive accuracy-efficiency frontier on ADE20K segmentation at low FLOP budgets (Figure~\ref{fig:resolution-and-mask-size-analysis}\textbf{C}).

\section{Conclusion}
\label{sec:conclusion}

CanViT applies the foundation model playbook to active vision,
with a general-purpose active-vision model
that can be used as-is across tasks and viewing policies
while delivering outstanding performance.
Our results show that pairing a carefully designed active-vision architecture
with a highly informative learning signal
is sufficient to dramatically advance the state of the art in active computer vision,
without relying on complex \ac{RL} pipelines.
Together, CanViT's strong inference-time generalization, high computational efficiency
and state-of-the-art performance on active-vision benchmarks
highlight the potential of our approach
and of task- and policy-agnostic \ac{AVFM} pretraining.

\subsection{Limitations and future work}
\label{subsec:limitations}

\textbf{Dynamic scenes.}
Like most active-vision models, CanViT was trained and evaluated on static scenes.
However, its constant-memory recurrent design and high computational efficiency
make it a strong candidate for future adaptation to real-time video processing and embodied active perception.
In a changing environment, the addition of a lightweight gating mechanism may be desirable to facilitate forgetting, as the contents of the canvas may become stale when objects move.

\textbf{Foveated sensors.}
CanViT accepts a different zoom level for each glimpse, but not for each patch within a glimpse.
Future extensions of CanViT could incorporate a more human-like sensory front end, with a high-resolution fovea and a low-resolution periphery~\cite{cheungEmergenceFovealImage2017,killickFoveationEraDeep2023,gizdovSeeingMoreLess2025,schmidtSegmentThisThing2025,blauch2026fovi}, making it a more faithful computational model of human active vision.

\textbf{Teacher-free pretraining.}
Our passive-to-active distillation scheme accelerates active-vision pretraining
by allowing it to remain focused on active-vision-specific considerations
such as partial observability and information integration across glimpses,
but requires a pretrained passive teacher.
As a result, while much of CanViT's efficiency can be attributed to its architecture, its absolute performance also reflects the strength of its \mbox{DINOv3} teacher.
Moreover, the need for a teacher may prove limiting at very high scene resolutions
due to the quadratic cost of teacher self-attention
and the large storage footprint of precomputed features,
especially when working with video, which would involve distinct per-timestep targets.
We thus see pretraining paradigms that do not require an external pretrained teacher, such as active-to-active self-distillation and pixel-space reconstruction, as promising avenues for future research.

\textbf{Policy learning.}
Our \ac{RL}-trained viewing policies operate on top of a frozen CanViT-B and its ADE20K probe,
choose among a fixed set of viewpoints in episodes of \rlGlimpses{} glimpses, and optimize only the immediate reward.
Policies that plan over longer horizons, choose from a continuous range of viewpoints, are conditioned on a goal given as input (such as a target category or a text query), or are trained jointly with CanViT
would be natural extensions.

\textbf{Scale.}
Our results use a single model size
pretrained with a modest compute budget
on a dataset over 100$\times$ smaller than DINOv3's LVD-1689M;
we expect further gains from more extensive pretraining and larger models,
though the latter may be undesirable for edge deployment.

\textbf{Other tasks.}
We evaluated CanViT-B on ADE20K segmentation and ImageNet-1k classification, with a preliminary experiment on COCO object detection (Appendix~\ref{supp:coco-detection}).
We expect CanViT-B to be applicable to other tasks, such as monocular depth estimation, without retraining,
similarly to its DINOv3 teacher and passive models distilled from DINOv2/v3 features~\cite{zhang2025accessing}.

\begin{ack}
Funded by a Natural Sciences and Engineering Research Council of Canada (NSERC) Canada Graduate Research Scholarship -- Doctoral (CGRS D-613218-2026), a Vision Sciences Research Network (VSRN) PhD Recruitment Scholarship, Unifying Neuroscience and Artificial Intelligence~-- Québec (UNIQUE) MSc and PhD Excellence Scholarships, a Centre for Applied Mathematics in Bioscience and Medicine (CAMBAM) Fellowship, a McGill Department of Physiology Max and Jane Childress Entrance Fellowship, and a McGill Faculty of Medicine and Health Sciences Internal Studentship to Yohaï-Eliel Berreby; an NSERC Undergraduate Student Research Award administered by the McGill Faculty of Medicine and Health Sciences to Sabrina Du; a Canada CIFAR AI Chair to Audrey Durand; and an NSERC Discovery Research program grant (RGPIN-2022-05399) and Supplement (DGECR-2022-00321) together with a computing resources grant from Calcul Québec and the Digital Research Alliance of Canada to B. Suresh Krishna. This project received support from Lamarck Labs (Enactant, Inc.).

\end{ack}

\bibliographystyle{unsrtnat}
\pdfbookmark[1]{References}{references}
\bibliography{references}

@inproceedings{heMaskedAutoencodersAre2022,
  title = {{Masked {{Autoencoders Are Scalable Vision Learners}}}},
  booktitle = {Proceedings of the {{IEEE}}/{{CVF Conference}} on {{Computer Vision}} and {{Pattern Recognition}}},
  author = {He, Kaiming and Chen, Xinlei and Xie, Saining and Li, Yanghao and Doll{\'a}r, Piotr and Girshick, Ross},
  year = 2022,
  pages = {16000--16009},
}

@inproceedings{dosovitskiyImageWorth16x162020,
  title = {{An {{Image}} Is {{Worth}} 16x16 {{Words}}: {{Transformers}} for {{Image Recognition}} at {{Scale}}}},
  booktitle = {International {{Conference}} on {{Learning Representations}}},
  author = {Dosovitskiy, Alexey and Beyer, Lucas and Kolesnikov, Alexander and Weissenborn, Dirk and Zhai, Xiaohua and Unterthiner, Thomas and Dehghani, Mostafa and Minderer, Matthias and Heigold, Georg and Gelly, Sylvain and Uszkoreit, Jakob and Houlsby, Neil},
  year = 2021,
}

@inproceedings{heoRotaryPositionEmbedding2025,
  title = {{Rotary {{Position Embedding}} for {{Vision Transformer}}}},
  booktitle = {Computer {{Vision}} -- {{ECCV}} 2024},
  author = {Heo, Byeongho and Park, Song and Han, Dongyoon and Yun, Sangdoo},
  year = 2025,
  pages = {289--305},
  publisher = {Springer Nature Switzerland},
  address = {Cham},
  doi = {10.1007/978-3-031-72684-2_17},
}

@inproceedings{darcetVisionTransformersNeed2023,
  title = {{Vision {{Transformers Need Registers}}}},
  booktitle = {International {{Conference}} on {{Learning Representations}}},
  author = {Darcet, Timoth{\'e}e and Oquab, Maxime and Mairal, Julien and Bojanowski, Piotr},
  year = 2024,
}

@article{simeoniDINOv32026,
  title = {{{DINOv3}}},
  author = {Sim{\'e}oni, Oriane and Vo, Huy V. and Seitzer, Maximilian and Baldassarre, Federico and Oquab, Maxime and Jose, Cijo and Khalidov, Vasil and Szafraniec, Marc and Yi, Seungeun and Ramamonjisoa, Micha{\"e}l and Massa, Francisco and Haziza, Daniel and Wehrstedt, Luca and Wang, Jianyuan and Darcet, Timoth{\'e}e and Moutakanni, Th{\'e}o and Sentana, Leonel and Roberts, Claire and Vedaldi, Andrea and Tolan, Jamie and Brandt, John and Couprie, Camille and Mairal, Julien and J{\'e}gou, Herv{\'e} and Labatut, Patrick and Bojanowski, Piotr},
  year = 2026,
  journal = {Transactions on Machine Learning Research},
  issn = {2835-8856},
  url = {https://openreview.net/forum?id=2NlGyqNjns},
  note = {Featured Certification},
}

@inproceedings{mnihRecurrentModelsVisual2014,
  title = {{Recurrent {{Models}} of {{Visual Attention}}}},
  booktitle = {Advances in {{Neural Information Processing Systems}}},
  author = {Mnih, Volodymyr and Heess, Nicolas and Graves, Alex and {Kavukcuoglu}, Koray},
  year = 2014,
  volume = {27},
  publisher = {Curran Associates, Inc.},
}

@inproceedings{DRAMBaMK14,
  author       = {Jimmy Ba and
                  Volodymyr Mnih and
                  Koray Kavukcuoglu},
  title        = {{Multiple Object Recognition with Visual Attention}},
  booktitle    = {International Conference on Learning Representations},
  year         = {2015},
}

@inproceedings{elsayedSaccaderImprovingAccuracy2019,
  title = {{Saccader: {{Improving Accuracy}} of {{Hard Attention Models}} for {{Vision}}}},
  booktitle = {Advances in {{Neural Information Processing Systems}}},
  author = {Elsayed, Gamaleldin F. and Kornblith, Simon and Le, Quoc V.},
  year = 2019,
  volume = {32},
  publisher = {Curran Associates, Inc.},
}

@inproceedings{AMEPardylRK0T23,
  author       = {Adam Pardyl and
                  Grzegorz Rype{\'s}{\'c} and
                  Grzegorz Kurzejamski and
                  Bartosz Zieli{\'n}ski and
                  Tomasz Trzci{\'n}ski},
  title        = {{Active Visual Exploration Based on Attention-Map Entropy}},
  booktitle    = {Proceedings of the International Joint Conference on Artificial Intelligence},
  pages        = {1303--1311},
  publisher    = {ijcai.org},
  year         = {2023},
  doi          = {10.24963/IJCAI.2023/145},
}

@inproceedings{pardylAdaGlimpseActiveVisual2025,
  title = {{{{AdaGlimpse}}: {{Active Visual Exploration}} with {{Arbitrary Glimpse Position}} and {{Scale}}}},
  booktitle = {Computer {{Vision}} -- {{ECCV}} 2024},
  author = {Pardyl, Adam and Wronka, Micha{\l} and Wo{\l}czyk, Maciej and Adamczewski, Kamil and Trzci{\'n}ski, Tomasz and Zieli{\'n}ski, Bartosz},
  year = 2025,
  pages = {112--129},
  publisher = {Springer Nature Switzerland},
  address = {Cham},
  doi = {10.1007/978-3-031-72664-4_7},
}

@article{wangEmulatingHumanlikeAdaptive2025,
  title = {{Emulating Human-like Adaptive Vision for Efficient and Flexible Machine Visual Perception}},
  author = {Wang, Yulin and Yue, Yang and Yue, Yang and Wang, Huanqian and Jiang, Haojun and Han, Yizeng and Ni, Zanlin and Pu, Yifan and Shi, Minglei and Lu, Rui and Yang, Qisen and Zhao, Andrew and Xia, Zhuofan and Song, Shiji and Huang, Gao},
  year = 2025,
  journal = {Nature Machine Intelligence},
  volume = {7},
  number = {11},
  pages = {1804--1822},
  publisher = {Nature Publishing Group},
  doi = {10.1038/s42256-025-01130-7},
}

@inproceedings{zhouSceneParsingADE20K2017,
  title = {{Scene {{Parsing Through ADE20K Dataset}}}},
  booktitle = {Proceedings of the {{IEEE Conference}} on {{Computer Vision}} and {{Pattern Recognition}}},
  author = {Zhou, Bolei and Zhao, Hang and Puig, Xavier and Fidler, Sanja and Barriuso, Adela and Torralba, Antonio},
  year = 2017,
  pages = {633--641},
}

@article{yaminsUsingGoaldrivenDeep2016,
  title = {{Using Goal-Driven Deep Learning Models to Understand Sensory Cortex}},
  author = {Yamins, Daniel L. K. and DiCarlo, James J.},
  year = 2016,
  journal = {Nature Neuroscience},
  volume = {19},
  number = {3},
  pages = {356--365},
  publisher = {Nature Publishing Group},
  doi = {10.1038/nn.4244},
}

@misc{schrimpfBrainScoreWhichArtificial2018,
  title = {{Brain-{{Score}}: {{Which Artificial Neural Network}} for {{Object Recognition}} Is Most {{Brain-Like}}?}},
  author = {Schrimpf, Martin and Kubilius, Jonas and Hong, Ha and Majaj, Najib J. and Rajalingham, Rishi and Issa, Elias B. and Kar, Kohitij and Bashivan, Pouya and {Prescott-Roy}, Jonathan and Geiger, Franziska and Schmidt, Kailyn and Yamins, Daniel L. K. and DiCarlo, James J.},
  year = 2018,
  primaryclass = {New Results},
  publisher = {bioRxiv},
  doi = {10.1101/407007},
  archiveprefix = {bioRxiv},
  chapter = {New Results},
  note = {bioRxiv preprint, \doi{10.1101/407007}},
}

@article{zhuangUnsupervisedNeuralNetwork2021,
  title = {{Unsupervised Neural Network Models of the Ventral Visual Stream}},
  author = {Zhuang, Chengxu and Yan, Siming and Nayebi, Aran and Schrimpf, Martin and Frank, Michael C. and DiCarlo, James J. and Yamins, Daniel L. K.},
  year = 2021,
  journal = {Proceedings of the National Academy of Sciences},
  volume = {118},
  number = {3},
  pages = {e2014196118},
  publisher = {Proceedings of the National Academy of Sciences},
  doi = {10.1073/pnas.2014196118},
}

@inproceedings{bakhtiariFunctionalSpecializationVisual2021,
  title = {{The Functional Specialization of Visual Cortex Emerges from Training Parallel Pathways with Self-Supervised Predictive Learning}},
  booktitle = {Advances in {{Neural Information Processing Systems}}},
  author = {Bakhtiari, Shahab and Mineault, Patrick and Lillicrap, Timothy and Pack, Christopher and Richards, Blake},
  year = 2021,
  volume = {34},
  pages = {25164--25178},
  publisher = {Curran Associates, Inc.},
}

@inproceedings{raugelDisentanglingFactorsConvergence2025,
  title = {Disentangling the {{Factors}} of {{Convergence}} between {{Brains}} and {{DINOv3}}},
  booktitle = {The {{Fourteenth International Conference}} on {{Learning Representations}}},
  author = {Raugel, Jos{\'e}phine and Szafraniec, Marc and Vo, Huy V. and Couprie, Camille and Rapin, J{\'e}r{\'e}my and {d'Ascoli}, St{\'e}phane and Labatut, Patrick and Bojanowski, Piotr and Wyart, Valentin and King, Jean-Remi},
  year = 2026,
  url = {https://openreview.net/forum?id=i99ccgfad8},
}

@article{mehrerEcologicallyMotivatedImage2021,
  title = {{An Ecologically Motivated Image Dataset for Deep Learning Yields Better Models of Human Vision}},
  author = {Mehrer, Johannes and Spoerer, Courtney J. and Jones, Emer C. and Kriegeskorte, Nikolaus and Kietzmann, Tim C.},
  year = 2021,
  journal = {Proceedings of the National Academy of Sciences},
  volume = {118},
  number = {8},
  pages = {e2011417118},
  publisher = {Proceedings of the National Academy of Sciences},
  doi = {10.1073/pnas.2011417118},
}

@article{gilbertTopdownInfluencesVisual2013,
  title = {{Top-down Influences on Visual Processing}},
  author = {Gilbert, Charles D. and Li, Wu},
  year = 2013,
  journal = {Nature Reviews Neuroscience},
  volume = {14},
  number = {5},
  pages = {350--363},
  publisher = {Nature Publishing Group},
  doi = {10.1038/nrn3476},
}

@article{karEvidenceThatRecurrent2019,
  title = {{Evidence That Recurrent Circuits Are Critical to the Ventral Stream's Execution of Core Object Recognition Behavior}},
  author = {Kar, Kohitij and Kubilius, Jonas and Schmidt, Kailyn and Issa, Elias B. and DiCarlo, James J.},
  year = 2019,
  journal = {Nature Neuroscience},
  volume = {22},
  number = {6},
  pages = {974--983},
  publisher = {Nature Publishing Group},
  doi = {10.1038/s41593-019-0392-5},
}

@article{hoppeMultistepPlanningEye2019,
  title = {{Multi-Step Planning of Eye Movements in Visual Search}},
  author = {Hoppe, David and Rothkopf, Constantin A.},
  year = 2019,
  journal = {Scientific Reports},
  volume = {9},
  number = {1},
  pages = {144},
  publisher = {Nature Publishing Group},
  doi = {10.1038/s41598-018-37536-0},
}

@article{yaminsPerformanceoptimizedHierarchicalModels2014,
  title = {{Performance-Optimized Hierarchical Models Predict Neural Responses in Higher Visual Cortex}},
  author = {Yamins, Daniel L. K. and Hong, Ha and Cadieu, Charles F. and Solomon, Ethan A. and Seibert, Darren and DiCarlo, James J.},
  year = 2014,
  journal = {Proceedings of the National Academy of Sciences},
  volume = {111},
  number = {23},
  pages = {8619--8624},
  publisher = {Proceedings of the National Academy of Sciences},
  doi = {10.1073/pnas.1403112111},
}

@article{raoPredictiveCodingVisual1999,
  title = {{Predictive Coding in the Visual Cortex: A Functional Interpretation of Some Extra-Classical Receptive-Field Effects}},
  author = {Rao, Rajesh P. N. and Ballard, Dana H.},
  year = 1999,
  journal = {Nature Neuroscience},
  volume = {2},
  number = {1},
  pages = {79--87},
  publisher = {Nature Publishing Group},
  doi = {10.1038/4580},
}

@article{melcherPersistenceVisualMemory2001,
  title = {{Persistence of Visual Memory for Scenes}},
  author = {Melcher, David},
  year = 2001,
  journal = {Nature},
  volume = {412},
  number = {6845},
  pages = {401},
  publisher = {Nature Publishing Group},
  doi = {10.1038/35086646},
}

@inproceedings{ablavatski2017enriched,
  title={{Enriched Deep Recurrent Visual Attention Model for Multiple Object Recognition}},
  author={Ablavatski, Artsiom and Lu, Shijian and Cai, Jianfei},
  booktitle={Proceedings of the IEEE Winter Conference on Applications of Computer Vision},
  pages={971--978},
  year={2017},
  organization={IEEE},
}

@inproceedings{wangGlanceFocusDynamic2020,
  title = {{Glance and Focus: A Dynamic Approach to Reducing Spatial Redundancy in Image Classification}},
  booktitle = {Advances in {{Neural Information Processing Systems}}},
  author = {Wang, Yulin and Lv, Kangchen and Huang, Rui and Song, Shiji and Yang, Le and Huang, Gao},
  year = {2020},
  volume = {33},
  pages = {2432--2444},
  publisher = {Curran Associates, Inc.},
}

@inproceedings{liuImprovedEfficiencyBased2023,
  title = {{Improved {{Efficiency Based}} on {{Learned Saccade}} and {{Continuous Scene Reconstruction From Foveated Visual Sampling}}}},
  booktitle = {International {{Conference}} on {{Learning Representations}}},
  author = {Liu, Jiayang and Bu, Yiming and Ts'o, Daniel and Qiu, Qinru},
  year = 2024,
}

@inproceedings{papadopoulosHardAttentionScalableImage2021,
  title = {{Hard-{{Attention}} for {{Scalable Image Classification}}}},
  booktitle = {Advances in {{Neural Information Processing Systems}}},
  author = {Papadopoulos, Athanasios N. and Korus, Pawe{\l} and Memon, Nasir},
  year = 2021,
  volume = {34},
  pages = {14694--14707},
  publisher = {Curran Associates, Inc.},
}

@inproceedings{matheReinforcementLearningVisual2016,
  title = {{Reinforcement {{Learning}} for {{Visual Object Detection}}}},
  booktitle = {Proceedings of the {{IEEE Conference}} on {{Computer Vision}} and {{Pattern Recognition}}},
  author = {Mathe, Stefan and Pirinen, Aleksis and Sminchisescu, Cristian},
  year = 2016,
  pages = {2894--2902},
  doi = {10.1109/CVPR.2016.316},
}

@misc{pourrahimiNeuralSignaturesAssociational2025,
  title = {{Neural Signatures of Associational Cortex Emerge in a Goal-Directed Model of Visual Search}},
  author = {Pourrahimi, Motahareh and Bashivan, Pouya},
  year = 2025,
  primaryclass = {New Results},
  publisher = {bioRxiv},
  doi = {10.1101/2025.06.06.658387},
  archiveprefix = {bioRxiv},
  chapter = {New Results},
  note = {bioRxiv preprint, \doi{10.1101/2025.06.06.658387}},
}

@article{russakovskyImageNetLargeScale2015,
  title = {{{ImageNet Large Scale Visual Recognition Challenge}}},
  author = {Russakovsky, Olga and Deng, Jia and Su, Hao and Krause, Jonathan and Satheesh, Sanjeev and Ma, Sean and Huang, Zhiheng and Karpathy, Andrej and Khosla, Aditya and Bernstein, Michael and Berg, Alexander C. and {Fei-Fei}, Li},
  year = 2015,
  journal = {International Journal of Computer Vision},
  volume = {115},
  number = {3},
  pages = {211--252},
  doi = {10.1007/s11263-015-0816-y},
}

@inproceedings{schmidtSegmentThisThing2025,
  title = {{Segment {{This Thing}}: {{Foveated Tokenization}} for {{Efficient Point-Prompted Segmentation}}}},
  booktitle = {Proceedings of the {{IEEE}}/{{CVF Conference}} on {{Computer Vision}} and {{Pattern Recognition}}},
  author = {Schmidt, Tanner and Newcombe, Richard},
  year = 2025,
  pages = {29428--29437},
  doi = {10.1109/CVPR52734.2025.02740},
}

@inproceedings{zhengRethinkingSemanticSegmentation2021,
  title = {{Rethinking {{Semantic Segmentation From}} a {{Sequence-to-Sequence Perspective With Transformers}}}},
  booktitle = {Proceedings of the {{IEEE}}/{{CVF Conference}} on {{Computer Vision}} and {{Pattern Recognition}}},
  author = {Zheng, Sixiao and Lu, Jiachen and Zhao, Hengshuang and Zhu, Xiatian and Luo, Zekun and Wang, Yabiao and Fu, Yanwei and Feng, Jianfeng and Xiang, Tao and Torr, Philip H. S. and Zhang, Li},
  year = 2021,
  pages = {6881--6890},
}

@inproceedings{seifiGlimpseAttendandExploreSelfAttentionActive2021,
  title = {{Glimpse-{{Attend-and-Explore}}: {{Self-Attention}} for {{Active Visual Exploration}}}},
  booktitle = {Proceedings of the {{IEEE}}/{{CVF International Conference}} on {{Computer Vision}}},
  author = {Seifi, Soroush and Jha, Abhishek and Tuytelaars, Tinne},
  year = 2021,
  pages = {16137--16146},
}

@inproceedings{olszewskiTORETokenRecycling2025,
  title = {{{{TORE}}: {{Token Recycling}} in {{Vision Transformers}} for {{Efficient Active Visual Exploration}}}},
  booktitle = {Proceedings of the {{IEEE}}/{{CVF Winter Conference}} on {{Applications}} of {{Computer Vision}}},
  author = {Olszewski, Jan and Rymarczyk, Dawid and W{\'o}jcik, Piotr and Pach, Mateusz and Zieli{\'n}ski, Bartosz},
  year = 2025,
  pages = {8606--8616},
  doi = {10.1109/WACV61041.2025.00834},
}

@inproceedings{touvronGoingDeeperImage2021,
  title = {{Going {{Deeper With Image Transformers}}}},
  booktitle = {Proceedings of the {{IEEE}}/{{CVF International Conference}} on {{Computer Vision}}},
  author = {Touvron, Hugo and Cord, Matthieu and Sablayrolles, Alexandre and Synnaeve, Gabriel and J{\'e}gou, Herv{\'e}},
  year = 2021,
  pages = {32--42},
}

@inproceedings{tancikFourierFeaturesLet2020,
  title = {{Fourier {{Features Let Networks Learn High Frequency Functions}} in {{Low Dimensional Domains}}}},
  booktitle = {Advances in {{Neural Information Processing Systems}}},
  author = {Tancik, Matthew and Srinivasan, Pratul and Mildenhall, Ben and {Fridovich-Keil}, Sara and Raghavan, Nithin and Singhal, Utkarsh and Ramamoorthi, Ravi and Barron, Jonathan and Ng, Ren},
  year = {2020},
  volume = {33},
  pages = {7537--7547},
  publisher = {Curran Associates, Inc.},
}

@inproceedings{saunshiReasoningLatentThoughts2024,
  title = {{Reasoning with {{Latent Thoughts}}: {{On}} the {{Power}} of {{Looped Transformers}}}},
  booktitle = {International {{Conference}} on {{Learning Representations}}},
  author = {Saunshi, Nikunj and Dikkala, Nishanth and Li, Zhiyuan and Kumar, Sanjiv and Reddi, Sashank J.},
  year = 2025,
}

@misc{baLayerNormalization2016,
  title = {{Layer {{Normalization}}}},
  author = {Ba, Jimmy Lei and Kiros, Jamie Ryan and Hinton, Geoffrey E.},
  year = 2016,
  number = {arXiv:1607.06450},
  eprint = {1607.06450},
  primaryclass = {stat},
  publisher = {arXiv},
  note = {arXiv preprint arXiv:1607.06450},
  doi = {10.48550/arXiv.1607.06450},
  archiveprefix = {arXiv},
}

@inproceedings{zhang2025accessing,
  title={{Accessing Vision Foundation Models via ImageNet-1K}},
  author={Yitian Zhang and Xu Ma and Yue Bai and Huan Wang and Yun Fu},
  booktitle={International Conference on Learning Representations},
  year={2025},
  url={https://openreview.net/forum?id=LC6ZtQV6u2},
}

@article{oquab2024dinov,
  title={{{DINO}v2: Learning Robust Visual Features without Supervision}},
  author={Maxime Oquab and Timoth{\'e}e Darcet and Th{\'e}o Moutakanni and Huy V. Vo and Marc Szafraniec and Vasil Khalidov and Pierre Fernandez and Daniel Haziza and Francisco Massa and Alaaeldin El-Nouby and Mido Assran and Nicolas Ballas and Wojciech Galuba and Russell Howes and Po-Yao Huang and Shang-Wen Li and Ishan Misra and Michael Rabbat and Vasu Sharma and Gabriel Synnaeve and Hu Xu and Herv{\'e} J{\'e}gou and Julien Mairal and Patrick Labatut and Armand Joulin and Piotr Bojanowski},
  journal={Transactions on Machine Learning Research},
  year={2024},
  url={https://openreview.net/forum?id=a68SUt6zFt},
  note={Featured Certification},
}

@inproceedings{ridnikImageNet21KPretrainingMasses2021,
  title = {{{{ImageNet-21K Pretraining}} for the {{Masses}}}},
  booktitle = {Proceedings of the {{Neural Information Processing Systems Track}} on {{Datasets}} and {{Benchmarks}}},
  author = {Ridnik, Tal and {Ben-Baruch}, Emanuel and Noy, Asaf and {Zelnik-Manor}, Lihi},
  year = 2021,
  volume = {1},
  archiveprefix = {arXiv},
}

@misc{wangHierarchicalReasoningModel2025,
  title = {{Hierarchical {{Reasoning Model}}}},
  author = {Wang, Guan and Li, Jin and Sun, Yuhao and Chen, Xing and Liu, Changling and Wu, Yue and Lu, Meng and Song, Sen and Yadkori, Yasin Abbasi},
  year = 2025,
  number = {arXiv:2506.21734},
  eprint = {2506.21734},
  primaryclass = {cs},
  publisher = {arXiv},
  doi = {10.48550/arXiv.2506.21734},
  archiveprefix = {arXiv},
  note = {arXiv preprint arXiv:2506.21734},
}

@inproceedings{yangLoopedTransformersAre2023,
  title = {{Looped {{Transformers}} Are {{Better}} at {{Learning Learning Algorithms}}}},
  booktitle = {International {{Conference}} on {{Learning Representations}}},
  author = {Yang, Liu and Lee, Kangwook and Nowak, Robert D. and Papailiopoulos, Dimitris},
  year = 2024,
}

@inproceedings{jaeglePerceiverGeneralPerception2021,
  title = {{Perceiver: {{General Perception}} with {{Iterative Attention}}}},
  booktitle = {Proceedings of the {{International Conference}} on {{Machine Learning}}},
  author = {Jaegle, Andrew and Gimeno, Felix and Brock, Andy and Vinyals, Oriol and Zisserman, Andrew and Carreira, Joao},
  year = 2021,
  pages = {4651--4664},
  publisher = {PMLR},
}

@inproceedings{jaeglePerceiverIOGeneral2021,
  title = {{Perceiver {{IO}}: {{A General Architecture}} for {{Structured Inputs}} \& {{Outputs}}}},
  booktitle = {International {{Conference}} on {{Learning Representations}}},
  author = {Jaegle, Andrew and Borgeaud, Sebastian and Alayrac, Jean-Baptiste and Doersch, Carl and Ionescu, Catalin and Ding, David and Koppula, Skanda and Zoran, Daniel and Brock, Andrew and Shelhamer, Evan and H{\'e}naff, Olivier J. and Botvinick, Matthew M. and Zisserman, Andrew and Vinyals, Oriol and Carreira, João},
  year = 2022,
}

@inproceedings{jabriScalableAdaptiveComputation2023,
  title = {{Scalable {{Adaptive Computation}} for {{Iterative Generation}}}},
  booktitle = {Proceedings of the {{International Conference}} on {{Machine Learning}}},
  author = {Jabri, Allan and Fleet, David J. and Chen, Ting},
  year = 2023,
  pages = {14569--14589},
  publisher = {PMLR},
}

@inproceedings{leeSetTransformerFramework2019,
  title = {{Set {{Transformer}}: {{A Framework}} for {{Attention-based Permutation-Invariant Neural Networks}}}},
  booktitle = {Proceedings of the {{International Conference}} on {{Machine Learning}}},
  author = {Lee, Juho and Lee, Yoonho and Kim, Jungtaek and Kosiorek, Adam R. and Choi, Seungjin and Teh, Yee Whye},
  year = 2019,
  pages = {3744--3753},
  publisher = {PMLR},
}

@misc{jolicoeur-martineauLessMoreRecursive2025,
  title = {{Less Is {{More}}: {{Recursive Reasoning}} with {{Tiny Networks}}}},
  author = {{Jolicoeur-Martineau}, Alexia},
  year = 2025,
  number = {arXiv:2510.04871},
  eprint = {2510.04871},
  primaryclass = {cs},
  publisher = {arXiv},
  doi = {10.48550/arXiv.2510.04871},
  archiveprefix = {arXiv},
  note = {arXiv preprint arXiv:2510.04871},
}

@misc{gravesAdaptiveComputationTime2017,
  title = {{Adaptive {{Computation Time}} for {{Recurrent Neural Networks}}}},
  author = {Graves, Alex},
  year = 2017,
  number = {arXiv:1603.08983},
  eprint = {1603.08983},
  primaryclass = {cs},
  publisher = {arXiv},
  doi = {10.48550/arXiv.1603.08983},
  archiveprefix = {arXiv},
  note = {arXiv preprint arXiv:1603.08983},
}

@inproceedings{baninoPonderNetLearningPonder2021,
  title = {{{{PonderNet}}: {{Learning}} to {{Ponder}}}},
  booktitle = {{{ICML Workshop}} on {{Automated Machine Learning}} ({{AutoML}})},
  author = {Banino, Andrea and Balaguer, Jan and Blundell, Charles},
  year = 2021,
}

@inproceedings{haoTrainingLargeLanguage2025,
  title = {{Training {{Large Language Models}} to {{Reason}} in a {{Continuous Latent Space}}}},
  booktitle = {Second {{Conference}} on {{Language Modeling}}},
  author = {Hao, Shibo and Sukhbaatar, Sainbayar and Su, DiJia and Li, Xian and Hu, Zhiting and Weston, Jason E. and Tian, Yuandong},
  year = 2025,
}

@inproceedings{geipingScalingTestTimeCompute2025,
  title = {Scaling up {{Test-Time Compute}} with {{Latent Reasoning}}: {{A Recurrent Depth Approach}}},
  booktitle = {Advances in {{Neural Information Processing Systems}}},
  author = {Geiping, Jonas and McLeish, Sean Michael and Jain, Neel and Kirchenbauer, John and Singh, Siddharth and Bartoldson, Brian R. and Kailkhura, Bhavya and Bhatele, Abhinav and Goldstein, Tom},
  year = 2025,
  volume = {38},
  publisher = {Curran Associates, Inc.},
  doi = {10.52202/085713-1380},
}

@inproceedings{dehghaniUniversalTransformers2018,
  title = {{Universal {{Transformers}}}},
  booktitle = {International {{Conference}} on {{Learning Representations}}},
  author = {Dehghani, Mostafa and Gouws, Stephan and Vinyals, Oriol and Uszkoreit, Jakob and Kaiser, Lukasz},
  year = 2019,
}

@misc{chungEmpiricalEvaluationGated2014,
  title = {{Empirical {{Evaluation}} of {{Gated Recurrent Neural Networks}} on {{Sequence Modeling}}}},
  author = {Chung, Junyoung and Gulcehre, Caglar and Cho, KyungHyun and Bengio, Yoshua},
  year = 2014,
  number = {arXiv:1412.3555},
  eprint = {1412.3555},
  primaryclass = {cs},
  publisher = {arXiv},
  doi = {10.48550/arXiv.1412.3555},
  archiveprefix = {arXiv},
  note = {arXiv preprint arXiv:1412.3555},
}

@article{hochreiterLongShortTermMemory1997,
  title = {{Long {{Short-Term Memory}}}},
  author = {Hochreiter, Sepp and Schmidhuber, J{\"u}rgen},
  year = 1997,
  journal = {Neural Computation},
  volume = {9},
  number = {8},
  pages = {1735--1780},
  doi = {10.1162/neco.1997.9.8.1735},
}

@inproceedings{liModelingHumanEye2023,
  title = {{Modeling {{Human Eye Movements}} with {{Neural Networks}} in a {{Maze-Solving Task}}}},
  booktitle = {Proceedings of the {{Gaze Meets ML}} Workshop},
  author = {Li, Jason and Watters, Nicholas and Sohn, Hansem and Jazayeri, Mehrdad},
  year = 2023,
  pages = {98--112},
  publisher = {PMLR},
}

@incollection{BADDELEY197447,
  title = {{Working Memory}},
  author = {Baddeley, Alan D. and Hitch, Graham},
  editor = {Bower, Gordon H.},
  year = 1974,
  booktitle = {Psychology of Learning and Motivation},
  series = {Psychology of Learning and Motivation},
  volume = {8},
  pages = {47--89},
  publisher = {Academic Press},
  doi = {10.1016/S0079-7421(08)60452-1},
}

@article{suRoFormerEnhancedTransformer2024,
  title = {{{{RoFormer}}: {{Enhanced}} Transformer with {{Rotary Position Embedding}}}},
  author = {Su, Jianlin and Ahmed, Murtadha and Lu, Yu and Pan, Shengfeng and Bo, Wen and Liu, Yunfeng},
  year = 2024,
  journal = {Neurocomputing},
  volume = {568},
  pages = {127063},
  doi = {10.1016/j.neucom.2023.127063},
}

@book{yarbusEyeMovementsVision1967,
  title = {{Eye {{Movements}} and {{Vision}}}},
  author = {Yarbus, Alfred L.},
  year = 1967,
  publisher = {Springer US},
  address = {Boston, MA},
  doi = {10.1007/978-1-4899-5379-7},
}

@article{tolmanCognitiveMapsRats1948,
  title = {{Cognitive Maps in Rats and Men}},
  author = {Tolman, Edward C.},
  year = 1948,
  journal = {Psychological Review},
  volume = {55},
  number = {4},
  pages = {189--208},
  publisher = {American Psychological Association},
  address = {US},
  doi = {10.1037/h0061626},
}

@inproceedings{anselPyTorch2Faster2024,
  title = {{{{PyTorch}} 2: {{Faster Machine Learning Through Dynamic Python Bytecode Transformation}} and {{Graph Compilation}}}},
  booktitle = {Proceedings of the {{ACM International Conference}} on {{Architectural Support}} for {{Programming Languages}} and {{Operating Systems}}, {{Volume}} 2},
  author = {Ansel, Jason and Yang, Edward and He, Horace and Gimelshein, Natalia and Jain, Animesh and Voznesensky, Michael and Bao, Bin and Bell, Peter and Berard, David and Burovski, Evgeni and Chauhan, Geeta and Chourdia, Anjali and Constable, Will and Desmaison, Alban and DeVito, Zachary and Ellison, Elias and Feng, Will and Gong, Jiong and Gschwind, Michael and Hirsh, Brian and Huang, Sherlock and Kalambarkar, Kshiteej and Kirsch, Laurent and Lazos, Michael and Lezcano, Mario and Liang, Yanbo and Liang, Jason and Lu, Yinghai and Luk, C. K. and Maher, Bert and Pan, Yunjie and Puhrsch, Christian and Reso, Matthias and Saroufim, Mark and Siraichi, Marcos Yukio and Suk, Helen and Zhang, Shunting and Suo, Michael and Tillet, Phil and Zhao, Xu and Wang, Eikan and Zhou, Keren and Zou, Richard and Wang, Xiaodong and Mathews, Ajit and Wen, William and Chanan, Gregory and Wu, Peng and Chintala, Soumith},
  year = 2024,
  series = {{{ASPLOS}} '24},
  pages = {929--947},
  publisher = {Association for Computing Machinery},
  address = {New York, NY, USA},
  doi = {10.1145/3620665.3640366},
}

@misc{flax2020github,
  author = {Jonathan Heek and Anselm Levskaya and Avital Oliver and Marvin Ritter and Bertrand Rondepierre and Andreas Steiner and Marc van {Z}ee},
  title = {{{F}lax: A neural network library and ecosystem for {JAX}}},
  url = {http://github.com/google/flax},
  year = {2026},
}

@misc{beyerAreWeDone2020,
  title = {{Are We Done with {{ImageNet}}?}},
  author = {Beyer, Lucas and H{\'e}naff, Olivier J. and Kolesnikov, Alexander and Zhai, Xiaohua and van den Oord, A{\"a}ron},
  year = 2020,
  number = {arXiv:2006.07159},
  eprint = {2006.07159},
  primaryclass = {cs},
  publisher = {arXiv},
  doi = {10.48550/arXiv.2006.07159},
  note = {arXiv preprint arXiv:2006.07159},
}

@inproceedings{kumarFineTuningCanDistort2022,
  title = {{Fine-{{Tuning}} Can {{Distort Pretrained Features}} and {{Underperform Out-of-Distribution}}}},
  booktitle = {International {{Conference}} on {{Learning Representations}}},
  author = {Kumar, Ananya and Raghunathan, Aditi and Jones, Robbie Matthew and Ma, Tengyu and Liang, Percy},
  year = 2022,
}

@misc{jax2018github,
  author = {James Bradbury and Roy Frostig and Peter Hawkins and Matthew James Johnson and Yash Katariya and Chris Leary and Dougal Maclaurin and George Necula and Adam Paszke and Jake Vander{P}las and Skye Wanderman-{M}ilne and Qiao Zhang},
  title = {{{JAX}: Composable Transformations of {P}ython+{N}um{P}y Programs}},
  url = {http://github.com/jax-ml/jax},
  year = {2018},
}

@inproceedings{akibaOptunaNextgenerationHyperparameter2019,
  title = {{Optuna: {{A Next-generation Hyperparameter Optimization Framework}}}},
  booktitle = {Proceedings of the {{ACM SIGKDD International Conference}} on {{Knowledge Discovery}} \& {{Data Mining}}},
  author = {Akiba, Takuya and Sano, Shotaro and Yanase, Toshihiko and Ohta, Takeru and Koyama, Masanori},
  year = 2019,
  series = {{{KDD}} '19},
  pages = {2623--2631},
  publisher = {Association for Computing Machinery},
  address = {New York, NY, USA},
  doi = {10.1145/3292500.3330701},
}

@inproceedings{seabold2010,
  author    = {Skipper Seabold and Josef Perktold},
  title     = {{statsmodels: Econometric and Statistical Modeling with {Python}}},
  booktitle = {Proceedings of the Python in Science Conference},
  pages     = {57--61},
  year      = {2010},
}

@article{srivastavaDropoutSimpleWay2014,
  title = {{Dropout: {{A Simple Way}} to {{Prevent Neural Networks}} from {{Overfitting}}}},
  author = {Srivastava, Nitish and Hinton, Geoffrey and Krizhevsky, Alex and Sutskever, Ilya and Salakhutdinov, Ruslan},
  year = 2014,
  journal = {Journal of Machine Learning Research},
  volume = {15},
  number = {56},
  pages = {1929--1958},
}

@inproceedings{ioffeBatchNormalizationAccelerating2015,
  title = {{Batch {{Normalization}}: {{Accelerating Deep Network Training}} by {{Reducing Internal Covariate Shift}}}},
  booktitle = {Proceedings of the {{International Conference}} on {{Machine Learning}}},
  author = {Ioffe, Sergey and Szegedy, Christian},
  year = 2015,
  pages = {448--456},
  publisher = {PMLR},
}

@inproceedings{renEndToEndInstanceSegmentation2017,
  title = {{End-{{To-End Instance Segmentation With Recurrent Attention}}}},
  booktitle = {Proceedings of the {{IEEE Conference}} on {{Computer Vision}} and {{Pattern Recognition}}},
  author = {Ren, Mengye and Zemel, Richard S.},
  year = 2017,
  pages = {6656--6664},
}

@article{zhouPlaces10Million2017,
  title = {{Places: {{A}} 10 Million Image Database for Scene Recognition}},
  author = {Zhou, Bolei and Lapedriza, Agata and Khosla, Aditya and Oliva, Aude and Torralba, Antonio},
  year = 2018,
  journal = {IEEE Transactions on Pattern Analysis and Machine Intelligence},
  volume = {40},
  number = {6},
  pages = {1452--1464},
  publisher = {IEEE},
  doi = {10.1109/TPAMI.2017.2723009},
}

@inproceedings{ibrayevExploringFoveationSaccade2024,
  title = {{Exploring Foveation and Saccade for Improved Weakly-Supervised Localization}},
  booktitle = {Proceedings of the 2nd {{Gaze Meets ML}} Workshop},
  author = {Ibrayev, Timur and Nagaraj, Manish and Mukherjee, Amitangshu and Roy, Kaushik},
  year = 2024,
  series = {Proceedings of Machine Learning Research},
  volume = {226},
  pages = {61--89},
  publisher = {PMLR},
}

@article{ibrayevTwoStreamFoveationBasedActive2024,
  title = {{Toward Two-Stream Foveation-Based Active Vision Learning}},
  author = {Ibrayev, Timur and Mukherjee, Amitangshu and Aketi, Sai Aparna and Roy, Kaushik},
  year = 2024,
  journal = {IEEE Transactions on Cognitive and Developmental Systems},
  volume = {16},
  number = {5},
  pages = {1843--1860},
  doi = {10.1109/TCDS.2024.3390597},
}

@inproceedings{yangFocalAttentionLongRange2021,
  title = {{Focal Attention for Long-Range Interactions in Vision Transformers}},
  booktitle = {Advances in {{Neural Information Processing Systems}}},
  author = {Yang, Jianwei and Li, Chunyuan and Zhang, Pengchuan and Dai, Xiyang and Xiao, Bin and Yuan, Lu and Gao, Jianfeng},
  year = 2021,
  volume = {34},
  pages = {30008--30022},
  publisher = {Curran Associates, Inc.},
}

@inproceedings{yangFocalModulationNetworks2022,
  title = {{Focal Modulation Networks}},
  booktitle = {Advances in {{Neural Information Processing Systems}}},
  author = {Yang, Jianwei and Li, Chunyuan and Dai, Xiyang and Gao, Jianfeng},
  year = 2022,
  volume = {35},
  pages = {4203--4217},
  publisher = {Curran Associates, Inc.},
  doi = {10.52202/068431-0304},
}

@inproceedings{gizdovSeeingMoreLess2025,
  title = {{Seeing More with Less: Human-like Representations in Vision Models}},
  booktitle = {Proceedings of the {{IEEE}}/{{CVF Conference}} on {{Computer Vision}} and {{Pattern Recognition}}},
  author = {Gizdov, Andrey and Ullman, Shimon and Harari, Daniel},
  year = 2025,
  pages = {4408--4417},
  doi = {10.1109/CVPR52734.2025.00416},
}

@inproceedings{linMicrosoftCOCOCommon2014,
  title = {{Microsoft {{COCO}}: {{Common Objects}} in {{Context}}}},
  booktitle = {Computer {{Vision}} -- {{ECCV}} 2014},
  author = {Lin, Tsung-Yi and Maire, Michael and Belongie, Serge and Hays, James and Perona, Pietro and Ramanan, Deva and Doll{\'a}r, Piotr and Zitnick, C. Lawrence},
  year = 2014,
  pages = {740--755},
  publisher = {Springer International Publishing},
  doi = {10.1007/978-3-319-10602-1_48},
}

@misc{zhouObjectsPoints2019,
  title = {{Objects as {{Points}}}},
  author = {Zhou, Xingyi and Wang, Dequan and Kr{\"a}henb{\"u}hl, Philipp},
  year = 2019,
  number = {arXiv:1904.07850},
  eprint = {1904.07850},
  primaryclass = {cs.CV},
  publisher = {arXiv},
  doi = {10.48550/arXiv.1904.07850},
  archiveprefix = {arXiv},
  note = {arXiv preprint arXiv:1904.07850},
}

@inproceedings{wooConvNeXtV2Codesigning2023,
  title = {{{{ConvNeXt V2}}: {{Co-designing}} and {{Scaling ConvNets}} with {{Masked Autoencoders}}}},
  booktitle = {Proceedings of the {{IEEE}}/{{CVF Conference}} on {{Computer Vision}} and {{Pattern Recognition}}},
  author = {Woo, Sanghyun and Debnath, Shoubhik and Hu, Ronghang and Chen, Xinlei and Liu, Zhuang and Kweon, In So and Xie, Saining},
  year = 2023,
  pages = {16133--16142},
  doi = {10.1109/CVPR52729.2023.01548},
}

@inproceedings{ronnebergerUNetConvolutionalNetworks2015,
  title = {{{{U-Net}}: {{Convolutional Networks}} for {{Biomedical Image Segmentation}}}},
  booktitle = {Medical {{Image Computing}} and {{Computer-Assisted Intervention}} ({{MICCAI}})},
  author = {Ronneberger, Olaf and Fischer, Philipp and Brox, Thomas},
  year = 2015,
  series = {Lecture {{Notes}} in {{Computer Science}}},
  volume = {9351},
  pages = {234--241},
  publisher = {Springer},
  doi = {10.1007/978-3-319-24574-4_28},
}

@article{watkinsQlearning1992,
  title = {{Q-learning}},
  author = {Watkins, Christopher J. C. H. and Dayan, Peter},
  year = 1992,
  journal = {Machine Learning},
  volume = {8},
  number = {3--4},
  pages = {279--292},
  doi = {10.1007/BF00992698},
}

@article{williamsSimpleStatisticalGradientfollowing1992,
  title = {{Simple Statistical Gradient-Following Algorithms for Connectionist Reinforcement Learning}},
  author = {Williams, Ronald J.},
  year = 1992,
  journal = {Machine Learning},
  volume = {8},
  number = {3--4},
  pages = {229--256},
  doi = {10.1007/BF00992696},
}

@misc{haarnojaSoftActorCriticAlgorithms2018,
  title = {{Soft Actor-Critic Algorithms and Applications}},
  author = {Haarnoja, Tuomas and Zhou, Aurick and Hartikainen, Kristian and Tucker, George and Ha, Sehoon and Tan, Jie and Kumar, Vikash and Zhu, Henry and Gupta, Abhishek and Abbeel, Pieter and Levine, Sergey},
  year = 2018,
  number = {arXiv:1812.05905},
  eprint = {1812.05905},
  primaryclass = {cs.LG},
  publisher = {arXiv},
  doi = {10.48550/arXiv.1812.05905},
  archiveprefix = {arXiv},
  note = {arXiv preprint arXiv:1812.05905},
}

@inproceedings{blauch2026fovi,
  title={{FOVI}: A biologically-inspired foveated interface for deep vision models},
  author={Nicholas Blauch and George A. Alvarez and Talia Konkle},
  booktitle={Proceedings of the 43rd International Conference on Machine Learning},
  year={2026},
  pages={8446--8469},
  publisher={PMLR},
}

@inproceedings{cheungEmergenceFovealImage2017,
  title = {{Emergence of Foveal Image Sampling from Learning to Attend in Visual Scenes}},
  booktitle = {International {{Conference}} on {{Learning Representations}}},
  author = {Cheung, Brian and Weiss, Eric and Olshausen, Bruno A.},
  year = 2017,
  url = {https://openreview.net/forum?id=SJJKxrsgl},
}

@inproceedings{killickFoveationEraDeep2023,
  title = {{Foveation in the Era of Deep Learning}},
  booktitle = {British {{Machine Vision Conference}} ({{BMVC}})},
  author = {Killick, George and Henderson, Paul and Siebert, Paul and Aragon-Camarasa, Gerardo},
  year = 2023,
  eprint = {2312.01450},
  archiveprefix = {arXiv},
}

\appendix

\clearpage
\part[Appendix]{}
\section*{\centering \LARGE Appendix}
\mtcsettitle{parttoc}{Contents}
\parttoc
\clearpage

\section{Canvas Attention pseudocode}
\label{supp:canvas-attention-pseudocode}

\begin{lstlisting}[style=canvitpython]
class CanvasAttention(nn.Module):
    def __init__(self, D_q, D_kv):
        self.ln_q  = LayerNorm(D_q)
        self.ln_kv = LayerNorm(D_kv)

    # Common template for Reads and Writes
    def forward(self, x_q, x_kv, rope_q, rope_kv):
        q  = to_multihead(self.q_map(self.ln_q(x_q)))
        kv = self.ln_kv(x_kv)
        k  = to_multihead(self.k_map(kv))
        v  = to_multihead(self.v_map(kv))
        q  = apply_2d_rope(q, rope_q)
        k  = apply_2d_rope(k, rope_kv)
        return self.o_map(from_multihead(sdpa(q, k, v)))


class CanvasAttentionRead(CanvasAttention):  # backbone queries canvas
    def __init__(self, D_bb, D_can):
        super().__init__(D_q=D_bb, D_kv=D_can)
        # backbone-side: fully-connected Query and Output projections
        self.q_map = Linear(D_bb, D_can)
        self.o_map = Linear(D_can, D_bb)
        # canvas-side: no-op for Key and Value
        self.k_map = Identity()
        self.v_map = Identity()


class CanvasAttentionWrite(CanvasAttention):  # canvas queries backbone
    def __init__(self, D_bb, D_can):
        super().__init__(D_q=D_can, D_kv=D_bb)
        # backbone-side: fully-connected Key and Value projections
        self.k_map = Linear(D_bb, D_can)
        self.v_map = Linear(D_bb, D_can)
        # canvas-side: no-op for Query and Output
        self.q_map = Identity()
        self.o_map = Identity()


# cell centers of uniform R x C grid
# in [-1,+1]^2, shape [R*C, 2]
def grid(R, C):
    ys = (arange(R) + 0.5) / R * 2 - 1
    xs = (arange(C) + 0.5) / C * 2 - 1
    return stack(meshgrid(ys, xs), -1).reshape(R * C, 2)


# Scene-Relative 2D Rotary Position Embeddings
# center=(y,x) in [-1,+1]^2, scale in (0,1]: where/how zoomed-out the viewpoint is
rope_bb  = compute_2d_rope(center + scale * grid(H_g, W_g))  # dynamic
rope_can = compute_2d_rope(grid(H_c, W_c))                   # fixed

# CanViT alternates reads and writes across depth:
x_bb  = blk1(blk0(x_bb))
x_bb  = x_bb  + read(x_bb, x_can, rope_bb, rope_can)
x_bb  = blk3(blk2(x_bb))
x_can = x_can + write(x_can, x_bb, rope_can, rope_bb)
x_bb  = blk5(blk4(x_bb))
x_bb  = x_bb  + read(x_bb, x_can, rope_bb, rope_can)
# ...
\end{lstlisting}
\clearpage

\section{Viewpoint Encoding (VPE)}
\label{supp:vpe}

The VPE token is instantiated by first encoding the current viewpoint $(x, y, s)$ as the triplet $(x/s, y/s, \log s)$, a parameterization with scale invariance, same-scale translation invariance, and planar isotropy properties (proven below). We then lift this triplet into $\mathbb{R}^{D_{\mathrm{bb}}}$ via \ac{RFF}~\cite{tancikFourierFeaturesLet2020}, and apply layer normalization~\cite{baLayerNormalization2016} before letting the backbone process it alongside the other glimpse tokens (patches, registers, and recurrent CLS token).

\subsection{Definitions: scene, viewpoint, crop}

We consider a finite 2D scene whose $(x,y)$ coordinates span $[-1,+1]^2$.

We call a \emph{viewpoint} a triplet $(x,y,s) \in \mathcal{V}_{\mathrm{raw}}$ such that the corresponding square crop,
\[
[x-s, x+s] \times [y-s, y+s],
\]
lies inside $[-1,+1]^2$. Equivalently,
\[
\mathcal{V}_{\mathrm{raw}} = \{\, (x,y,s) \in \mathbb{R}^2 \times (0,1] : |x| \le 1 - s,\ |y| \le 1 - s \,\}.
\]
For instance, the viewpoint $(0,0,1)$ spans the entire scene; the viewpoint $(0.5,0.5,0.5)$ spans the quadrant $[0,1]^2$; the viewpoint $(2, 2, 0.5)$ is invalid, as its center lies outside of the scene; the viewpoint $(0.5, 0.5, 1)$ is \emph{also} invalid, even though its center lies within the scene, because its borders extend beyond the scene boundaries.

\subsection{Finding a scale-invariant representation}

While the above representation of a viewpoint as a triplet $(x,y,s) \in \mathcal{V}_{\mathrm{raw}}$ is simple to understand and uniquely defines crops within the scene, it fails to represent an important property: \textbf{scale invariance}.

When considering viewpoints as vectors, we would like distances between viewpoints to be \textbf{invariant to global rescaling}. In other words, the distance between two side-by-side square crops should be identical regardless of zoom level; equivalently, a 10\% zoom of the scene should leave any pairwise viewpoint distance unchanged. This is not the case in a straightforward $(x,y,s)$ encoding, which becomes artificially insensitive to shifts in position and scale for small crops (i.e., when $s \ll 1$), thus being forced to under-represent fine detail at small scales, leading to loss of information, or over-represent it at large scales, leading to ill-conditioned representations with excessive sensitivity to small perturbations at large scales (when $s \approx 1$).

Beyond scale invariance, two further properties are natural: \textbf{same-scale translation invariance}, so that at a fixed zoom the encoding distance depends only on the relative offset between viewpoints, and \textbf{planar isotropy}, so that the encoding does not privilege a particular axis orientation. Because the scene is bounded, these properties are stated for transformations whose resulting viewpoints remain valid crops. We thus seek a smooth, injective $u : \mathcal{V}_{\mathrm{raw}} \to \mathcal{V}$ that, equipping $\mathcal{V}$ with the Euclidean distance $d$ inherited from $\mathbb{R}^3$, satisfies all three.

We propose the embedding $u : \mathcal{V}_{\mathrm{raw}} \to \mathcal{V} = u(\mathcal{V}_{\mathrm{raw}}) \subset \mathbb{R}^3$ defined by
\begin{align*}
u &: \mathcal{V}_{\mathrm{raw}} \to \mathcal{V} \\
& (x,y,s) \mapsto ( x/s, y/s, \log s ) = (u_1, u_2, u_3) \\
u^{-1} &: \mathcal{V} \to \mathcal{V}_{\mathrm{raw}} \\
& (u_1, u_2, u_3) \mapsto (\exp(u_3) u_1, \exp(u_3) u_2, \exp(u_3)) = (x,y,s).
\end{align*}
The component functions are smooth for $s > 0$, and the displayed inverse shows that $u$ is injective.

\subsection{Properties of $u$}\label{sec:vpe-properties}

Throughout, $q_i = (x_i, y_i, s_i) \in \mathcal{V}_{\mathrm{raw}}$ and $d$ is the Euclidean distance on $\mathcal{V} \subset \mathbb{R}^3$.

\begin{lemma}[Pairwise distance identity]\label{lem:pairwise}
For $q_i = (x_i, y_i, s_i) \in \mathcal{V}_{\mathrm{raw}}$,
\begin{align*}
d(u(q_1), u(q_2))^2 &= \| u(q_1) - u(q_2) \|_2^2 \\
&= \| ( x_1/s_1 - x_2/s_2,\ y_1/s_1 - y_2/s_2,\ \log s_1 - \log s_2 ) \|_2^2 \\
&= ( x_1/s_1 - x_2/s_2 )^2 + ( y_1/s_1 - y_2/s_2 )^2 + ( \log s_1 - \log s_2 )^2.
\end{align*}
\end{lemma}
\begin{proof}
Direct computation from the definition of $u$ and the Euclidean norm on $\mathcal{V} \subset \mathbb{R}^3$.
\end{proof}

\begin{proposition}[Scale invariance]\label{prop:scale-invariance}
For all $c > 0$ such that $c q_i = (c x_i, c y_i, c s_i) \in \mathcal{V}_{\mathrm{raw}}$ for $i=1,2$,
\[
d(u(c q_1), u(c q_2)) = d(u(q_1), u(q_2)).
\]
\end{proposition}
\begin{proof}
Let $c > 0$ and $q_i = (x_i,y_i,s_i)$. By the definition of $u$,
\[
u(c q_i) = ( x_i/s_i,\ y_i/s_i,\ \log s_i + \log c ),
\]
so the additive $\log c$ cancels in the difference:
\[
u(c q_1) - u(c q_2) = \left( \frac{x_1}{s_1} - \frac{x_2}{s_2},\ \frac{y_1}{s_1} - \frac{y_2}{s_2},\ \log s_1 - \log s_2 \right).
\]
Setting $c = 1$ in this identity yields $u(q_1) - u(q_2)$ component-wise; the two difference vectors are therefore equal, hence so are their norms.
\end{proof}

\begin{proposition}[Same-scale translation invariance]\label{prop:translation-invariance}
With $(x_1, y_1, s), (x_2, y_2, s) \in \mathcal{V}_{\mathrm{raw}}$, for any planar offset $(\Delta_x,\Delta_y) \in \mathbb{R}^2$ such that $(x_1+\Delta_x,y_1+\Delta_y,s)$ and $(x_2+\Delta_x,y_2+\Delta_y,s)$ lie in $\mathcal{V}_{\mathrm{raw}}$,
\[
d(u(x_1+\Delta_x, y_1+\Delta_y, s), u(x_2+\Delta_x, y_2+\Delta_y, s)) = d(u(x_1, y_1, s), u(x_2, y_2, s)).
\]
\end{proposition}
\begin{proof}
By Lemma~\ref{lem:pairwise},
\[
d(u(x_1, y_1, s), u(x_2, y_2, s))^2 = \left( \frac{x_1 - x_2}{s} \right)^2 + \left( \frac{y_1 - y_2}{s} \right)^2,
\]
and
\begin{align*}
d(u(x_1+\Delta_x, y_1+\Delta_y, s), u(x_2+\Delta_x, y_2+\Delta_y, s))^2
&= \left( \frac{(x_1+\Delta_x) - (x_2+\Delta_x)}{s} \right)^2 \\
&\quad + \left( \frac{(y_1+\Delta_y) - (y_2+\Delta_y)}{s} \right)^2 \\
&= \left( \frac{x_1 - x_2}{s} \right)^2 + \left( \frac{y_1 - y_2}{s} \right)^2.
\end{align*}
The two right-hand sides agree.
\end{proof}

\begin{proposition}[Planar isotropy]\label{prop:isotropy}
For any $Q \in \mathrm{O}(2)$, with $(x_{i,Q}, y_{i,Q}, s_i) \in \mathcal{V}_{\mathrm{raw}}$ for $i=1,2$ denoting the transformed coordinates $(x_{i,Q}, y_{i,Q}) = Q(x_i, y_i)$,
\[
d(u(x_{1,Q}, y_{1,Q}, s_1), u(x_{2,Q}, y_{2,Q}, s_2)) = d(u(x_1, y_1, s_1), u(x_2, y_2, s_2)).
\]
\end{proposition}
\begin{proof}
By Lemma~\ref{lem:pairwise},
\[
d(u(x_1,y_1,s_1), u(x_2,y_2,s_2))^2 = \| (x_1/s_1, y_1/s_1) - (x_2/s_2, y_2/s_2) \|_2^2 + (\log s_1 - \log s_2)^2.
\]
For any $Q \in \mathrm{O}(2)$,
\begin{align*}
d(u(x_{1,Q}, y_{1,Q}, s_1), u(x_{2,Q}, y_{2,Q}, s_2))^2
&= \| Q(x_1,y_1)/s_1 - Q(x_2,y_2)/s_2 \|_2^2 \\
&\quad + (\log s_1 - \log s_2)^2 \\
&= \| Q((x_1/s_1, y_1/s_1) - (x_2/s_2, y_2/s_2)) \|_2^2 \\
&\quad + (\log s_1 - \log s_2)^2 \\
&= \| (x_1/s_1, y_1/s_1) - (x_2/s_2, y_2/s_2) \|_2^2 \\
&\quad + (\log s_1 - \log s_2)^2,
\end{align*}
where the second equality uses linearity of $Q$, and the third uses orthogonality $\|Qv\|_2 = \|v\|_2$ for all $v \in \mathbb{R}^2$. The result follows.
\end{proof}

\section{CanViT-B pretraining details}
\label{supp:pretraining-details}

\textbf{Hyperparameters.}
We list architectural hyperparameters in Table~\ref{tab:pretrain-arch-hparams},
and ImageNet-21k pretraining hyperparameters in Table~\ref{tab:pretrain-hparams}.
Similarly to DINOv3 ViTs, we use LayerScale~\cite{touvronGoingDeeperImage2021} in the ViT backbone.

\textbf{Precomputation of teacher features.}
To eliminate the repeated cost of the teacher's forward pass during training, we precompute DINOv3 ViT-B features for all 13.2 million ImageNet-21k images at $512^2$\,px input resolution, which corresponds to a $32 \times 32$ patch grid.
Images are processed without augmentation (resize shortest side to 512\,px, center crop).
We performed this one-time export in parallel over many MIG 1g.10gb H100 instances, with a small total footprint of approximately 8 H100-equivalent hours for our entire training set.
We store features in float16, totaling approximately 19~TiB, and organize them into shards of 4096 images each.
Each shard contains whole-scene dense patch features and CLS tokens, both obtained after DINOv3's final LayerNorm.
Per-position z-score standardization statistics (mean and variance for each spatial position and embedding dimension) are precomputed from a representative sample of 4096 images.
We pre-shuffle shards during dataset export, enabling us to process the shards themselves sequentially during pretraining.
This strategy enables high-bandwidth streaming from networked storage, without incurring the high cost of a random data access pattern.

\textbf{Pretraining compute.} Our pretraining used \pretrainTotalHours{} H100-equivalent hours: \pretrainHours{} for training and \teacherExportHours{} for the one-time teacher-feature export. This is comparable to or lower than AdaGlimpse's training, which we estimate at 210--410 H100-equivalent hours: its official code release\footnote{\url{https://github.com/apardyl/AdaGlimpse}} reports that training ``takes around 1 week on 4x A100 GPUs'' for ImageNet-1k, and we account for the A100's lower bfloat16 FLOP/s and memory bandwidth.
Since our pretraining covers approximately \pretrainEpochs{} epochs of ImageNet-21k, computing teacher features online would only increase this total to \mbox{$\pretrainHours+\onlineTeacherExportHours=\onlineTeacherTotalHours$} H100-equivalent hours, while reducing storage from \teacherFeaturesTiB{}~TiB to zero.
This would not be an advantageous tradeoff in a static-images setting with abundant storage, but becomes relevant under storage constraints or for future extension to video.

\textbf{Numerical precision considerations.}
The use of mixed-precision training is critical for efficiency, but comes with numerical correctness considerations, particularly in long-horizon scenarios.
Over the course of this project, we encountered several subtle correctness issues, which only led to meaningful regressions after several hundred thousand steps of pretraining.
For safe CanViT pretraining with minimal performance impact, we keep the canvas in float32 as it accumulates updates across time and depth, and keep all coordinate-related components in float32: grid coordinates, \ac{RoPE} computations, VPE token projection matrix and creation.
Other operations, including \ac{SDPA} operations and learned projection layers, can be safely performed in bfloat16, following standard autocast operator rules.
We perform the backward pass outside of the AMP region, setting \texttt{backward\_pass\_autocast="off"} accordingly when using \texttt{torch.compile}.

\begin{table}[h]
  \caption{\textbf{CanViT-B architecture.}}
  \label{tab:pretrain-arch-hparams}
  \centering
  \begin{tabular}{ll}
    \toprule
    \textbf{Hyperparameter} & \textbf{Value} \\
    \midrule
    Backbone & ViT-B/16 \\
    Backbone embedding dim & 768 \\
    Backbone registers (ephemeral) & 5 \\
    Glimpse patch size & $16^2$\,px \\
    Canvas embedding dim & 1024 \\
    Canvas registers (persistent) & 16 \\
    Canvas Attention heads & 8 \\
    Canvas Attention head dim. & 128 \\
    Canvas Attention R/W stride & 2 \\
    RoPE base period & 100 \\
    RoPE precision & float32 \\
    VPE token & Enabled \\
    VPE \ac{RFF}~\cite{tancikFourierFeaturesLet2020} std.\ deviation $\sigma$ & 1 \\
    \bottomrule
  \end{tabular}
\end{table}

\begin{table}[h]
  \caption{\textbf{CanViT-B pretraining hyperparameters.}}
  \label{tab:pretrain-hparams}
  \centering
  \begin{minipage}[t]{0.49\linewidth}
    \centering
    \setlength{\tabcolsep}{3.4pt}%
    \begin{tabular}{ll}
      \toprule
      \textbf{Hyperparameter} & \textbf{Value} \\
      \midrule
      Optimizer & AdamW \\
      Initial learning rate & $1 \times 10^{-7}$ \\
      Peak learning rate & $4 \times 10^{-4}$ \\
      LR schedule & Warmup $\to$ Const. \\
      LR warmup steps & 100{,}000 \\
      Weight decay & $1 \times 10^{-4}$ \\
      Gradient clipping & 1.0 max norm \\
      AdamW $\beta_1$, $\beta_2$ & 0.9, 0.999 \\
      TBPTT chunk size & $K = 2$ glimpses \\
      Stop prob. & $p_{\mathrm{stop}} = 0.5$ \\
      Rollouts per step & 1 F-IID + 1 R-IID \\
      \bottomrule
    \end{tabular}
  \end{minipage}
  \hfill
  \begin{minipage}[t]{0.49\linewidth}
    \centering
    \setlength{\tabcolsep}{3.4pt}%
    \begin{tabular}{ll}
      \toprule
      \textbf{Hyperparameter} & \textbf{Value} \\
      \midrule
      Dataset & ImageNet-21k \\
      IN21k version & winter21\_whole \\
      Preprocessing & Resize(512), CenterCrop \\
      Data augmentation & None \\
      Scene resolution & $512^2$\,px (1024 patches) \\
      Glimpse resolution & $128^2$\,px (64 patches) \\
      Batch size & 64 scenes \\
      Total training steps & $2 \times 10^6$ \\
      Min scale & 0.05 (0.25\% of area) \\
      Forward precision & AMP bfloat16 \\
      ViT LayerScale init. & $1 \times 10^{-5}$ \\
      \bottomrule
    \end{tabular}
  \end{minipage}
\end{table}

\section{Evaluation details}
\label{supp:evaluation-details}

\subsection{Non-learned viewing policies}
\label{supp:policy-definitions}

\textbf{Full-then-IID (F-IID)} and \textbf{Random-then-IID (R-IID)} are described as part of our pretraining scheme (Section~\ref{sec:pretraining}).
F-IID begins with a full-scene viewpoint ($s = 1$) then samples i.i.d.\ random crops, whereas R-IID samples i.i.d.\ random crops at all timesteps, including $t = 0$.

\textbf{Coarse-to-Fine (C2F)} traverses a quadtree over the scene from coarse to fine. At level $\ell \ge 0$, the viewpoint scale is $s_\ell = 2^{-\ell}$, tiling the scene into a $2^\ell \times 2^\ell$ grid of non-overlapping crops. Let $V_\ell$ denote the ordered set of viewpoints at level $\ell$; within each level, we visit tiles in a random order $\sigma_\ell(V_\ell)$ (where $\sigma_\ell$ is an independent random permutation). The full sequence is the concatenation $V_0, \sigma_1(V_1), \sigma_2(V_2), \ldots$, truncated to the glimpse budget $T$.

\textbf{Fine-to-Coarse (F2C)} reverses C2F viewpoint sequences. For any given glimpse budget, F2C starts from the finest level and progresses toward coarser views. After identical image coverage, the C2F vs.\ F2C comparison isolates the effect of processing order.

\textbf{Entropy-Guided C2F (EG-C2F)} is a variant of C2F that prioritizes informative regions at each scale level. Rather than visiting tiles in random order, EG-C2F decodes the current canvas with the ADE20K probe before each glimpse, averages the Shannon entropy of the per-position class distribution over each unvisited tile, and visits the tile of highest mean entropy.

\textbf{Repeated Full-Scene (RFS)} repeats the viewpoint $(x, y, s) = (0, 0, 1)$ at every timestep. It serves as a recurrence-only control: any improvement over $t = 0$ must come from iterative canvas refinement with a fixed glimpse input.

\FloatBarrier
\subsection{Learned ADE20K viewing policies}
\label{supp:learned-policies}

We trained viewing policies with \ac{RL} on top of a frozen CanViT-B and its ADE20K probe, using a $\rlGrid \times \rlGrid$ canvas.

\textbf{Action space.} As in C2F and EG-C2F, the policy first views the full scene ($t = 0$).
At each subsequent timestep $t$, it chooses the viewpoint $\mathbf{v}_t$ from a fixed set of \rlCandidates{} candidates: for each scale $s \in \{\rlScales\}$, a $\rlCenters \times \rlCenters$ grid of centers spanning $[-(1-s), 1-s]^2$, which keeps each glimpse entirely within the scene.

\textbf{Network.} We implement the policy network as a U-Net~\cite{ronnebergerUNetConvolutionalNetworks2015} built from ConvNeXt-V2 blocks~\cite{wooConvNeXtV2Codesigning2023}.
Its input combines the current canvas, the canvas before the last glimpse, and the initial canvas.
At each canvas position, the input holds the layer-normalized canvas features and the entropy of the probe's prediction, the change in both over the last glimpse, and the cosine dissimilarity of the features to the canvas before the last glimpse and to the initial canvas.
We average-pool these inputs to a $\rlFeatureGrid \times \rlFeatureGrid$ grid.
The network outputs one score $f_t(\mathbf{v})$ for each of the \rlCandidates{} candidate viewpoints $\mathbf{v}$.

\textbf{Reward.} We formulate the reward for the viewpoint chosen at timestep $t$ as the relative reduction in segmentation cross-entropy after its glimpse:
\begin{equation}
  r_t = \frac{\mathrm{CE}_{t-1} - \mathrm{CE}_t}{\mathrm{CE}_{t-1}},
\end{equation}
where $\mathrm{CE}_t$ is the probe's segmentation cross-entropy against the ADE20K labels after glimpse $t$.

\textbf{Objectives.} We compare two objectives, Q-learning~\cite{watkinsQlearning1992} and REINFORCE~\cite{williamsSimpleStatisticalGradientfollowing1992}.
For both, we set the discount factor to zero, so the policy optimizes the immediate reward.
We standardize each timestep's rewards by running estimates of their mean and standard deviation, giving $\hat{r}_t$.
For \mbox{Q-learning}, we train $f_t(\mathbf{v})$ to predict the standardized reward of choosing $\mathbf{v}$, with the squared error $(f_t(\mathbf{v}_t) - \hat{r}_t)^2$ and $\varepsilon$-greedy exploration ($\varepsilon = \rlEpsilon$).
For REINFORCE, the scores $f_t$ are the logits of a softmax policy $\pi_t$, from which we sample viewpoints during training, and we minimize $-\hat{r}_t \log \pi_t(\mathbf{v}_t) - \alpha\,\mathcal{H}(\pi_t)$, where $\mathcal{H}$ is the policy's entropy.
As in Soft Actor-Critic~\cite{haarnojaSoftActorCriticAlgorithms2018}, the entropy coefficient $\alpha$ is adjusted automatically: it increases when the mean entropy over a batch falls below a target of \rlEntropyFloor{}\,nat, and decreases otherwise.
At evaluation time, both policies choose the viewpoint with the highest score, without exploration.

\textbf{Optimization.} We trained each network for \rlSteps{} optimizer steps with AdamW ($\beta_1 = \rlBetaOne$, $\beta_2 = \rlBetaTwo$, learning rate $\rlLr$ held constant after a \rlWarmupSteps{}-step linear warmup, weight decay $\rlWeightDecay$, gradient norm clipped at \rlGradClip{}), on batches of \rlBatch{} ADE20K training images with episodes of $T = \rlGlimpses$ glimpses.
We trained \rlTrainingRuns{} networks per objective, with different seeds.
Every \rlValInterval{} optimizer steps, we evaluated the policy on the ADE20K validation set.
For each training run, we report the checkpoint with the lowest mean validation cross-entropy over $t = 1$ to $\rlLastT$ (best) and the final checkpoint (last).

\textbf{Evaluation.} We evaluated the learned policies with the protocol of Appendix~\ref{supp:task-ade20k} and compared them with the non-learned policies in Figure~\ref{fig:main-ade20k-in1k-results}\textbf{D} (best checkpoints), and with C2F and EG-C2F in Table~\ref{tab:learned-policies}.

\begin{table}[h!]
  \caption{\textbf{ADE20K mIoU (\%) with learned policies, frozen CanViT-B ($\rlGrid \times \rlGrid$ canvas).}
  Cells: $\pm$ \bootstrapCiPct\% bootstrap CI half-width ($n = \rlTrainingRuns$ training runs for learned policies, $n = \rlCtwofN$ evaluation runs for C2F).
  $^\dagger$: one evaluation run of a deterministic policy.
  \textbf{Bold} = best policy at each timestep.}
  \label{tab:learned-policies}
  \centering
  \small
  \setcellgapes{1pt}%
  \begin{tabular}{lcccc}
    \toprule
    \textbf{Policy} & \rlHeader \\
    \midrule
        \multicolumn{5}{l}{\textit{Non-learned}} \\
    \policyswatch{2CA02C}EG-C2F$^\dagger$ & 42.23 & 43.29 & 44.05 & 44.65 \\
    \policyswatch{1F77B4}C2F & 41.34 $\pm$ 0.10 & 42.52 $\pm$ 0.08 & 43.60 $\pm$ 0.06 & 44.69 $\pm$ 0.03 \\
    \midrule
    \multicolumn{5}{l}{\textit{Learned}} \\
    \policyswatch{9467BD}Q-learning (best) & 42.65 $\pm$ 0.10 & \textbf{43.95 $\pm$ 0.10} & \textbf{44.61 $\pm$ 0.08} & \textbf{44.96 $\pm$ 0.06} \\
    \policyswatch{9467BD}Q-learning (last) & \textbf{42.71 $\pm$ 0.06} & \textbf{43.95 $\pm$ 0.08} & 44.59 $\pm$ 0.06 & 44.89 $\pm$ 0.08 \\
    \policyswatch{17BECF}REINFORCE (best) & 42.64 $\pm$ 0.07 & 43.79 $\pm$ 0.08 & 44.35 $\pm$ 0.07 & 44.77 $\pm$ 0.09 \\
    \policyswatch{17BECF}REINFORCE (last) & 42.58 $\pm$ 0.09 & 43.75 $\pm$ 0.10 & 44.27 $\pm$ 0.11 & 44.66 $\pm$ 0.14 \\
 
    \bottomrule
  \end{tabular}
\end{table}

\FloatBarrier
\subsection{FLOP counting}
\label{supp:flop-counting}

We count each multiply-accumulate (MAC) operation as two floating-point operations (FLOPs). We sourced architecture parameters for DINOv3~\cite{simeoniDINOv32026}, AME~\cite{AMEPardylRK0T23} and AdaGlimpse~\cite{pardylAdaGlimpseActiveVisual2025} from the corresponding papers and code releases.\footnote{\url{https://github.com/facebookresearch/dinov3}, \url{https://github.com/apardyl/AME},\\ \url{https://github.com/apardyl/AdaGlimpse}}
For AdaGlimpse's published results in Figure~\ref{fig:main-ade20k-in1k-results}\textbf{A}, we count $\adaglimpseFlopGrid \times \adaglimpseFlopGrid$ tokens per glimpse, as its paper describes; for the architecture control (Table~\ref{tab:architecture-control}), we count the $\ctrlTokensPerSide \times \ctrlTokensPerSide$ tokens of its released code, which the control follows.
To obtain fine-grained FLOP curves, we computed FLOP counts analytically, following the same standard formulas as PyTorch's \texttt{flop\_counter} module.
We validated our total counts end-to-end against traced FLOP counts as well as previously-reported FLOP counts, which were published for DINOv3 but not for AdaGlimpse or AME.
Among the non-learned policies (Appendix~\ref{supp:policy-definitions}), all but EG-C2F have virtually zero associated cost; EG-C2F does require entropy computation, but this leverages already-present logits output by the segmentation probe, with an additional cost of approximately 1--2\,MFLOP per selection with a $32^2$ canvas and approximately 5\,MFLOP with a $64^2$ canvas (counting each exponential and logarithm as one operation), under 0.03\% of a single CanViT-B forward pass.
We did not include this cost in our FLOP counts because it is insignificant, similarly to the common omission of element-wise operations in FLOP accounting.

\subsection{Task: ADE20K segmentation}
\label{supp:task-ade20k}

\textbf{Probe training.} Each probe consists of dropout~\cite{srivastavaDropoutSimpleWay2014} ($p = 0.1$), batch normalization~\cite{ioffeBatchNormalizationAccelerating2015}, and a linear classifier in the form of a $1 \times 1$ convolution.
For CanViT, this head is preceded by a layer normalization~\cite{baLayerNormalization2016} and applied to canvas patches.
For the DINOv3 baselines, the head is applied directly to the feature map, as the DINOv3 features that we consider are already layer-normalized.
We train for 40{,}000 steps with AdamW with peak learning rate $3 \times 10^{-4}$ and weight decay $10^{-3}$,
under a 1{,}500-step linear warmup followed by cosine decay,
at batch size 16.
Training-time augmentations follow DINOv3's segmentation recipe: random scale-jittered $512^2$ crops (scale factor in $[0.5, 2.0]$), horizontal flips, and photometric distortion.
Crops in which a single class covers at least 75\% of the labeled pixels are rejected and redrawn, up to ten times.
During CanViT probe training, glimpse viewpoints are drawn i.i.d.\ under the R-IID policy across $T = 10$ timesteps.

\textbf{Evaluation.} Similarly to our active-vision ADE20K baselines~\cite{AMEPardylRK0T23,pardylAdaGlimpseActiveVisual2025}, we resize all images and masks to a fixed size (here, $512 \times 512$). We apply this preprocessing consistently when evaluating \mbox{CanViT-B} and DINOv3 ViT-B, ensuring a fair comparison between the models. Across up to $T = \adeMaxT$ timesteps, we report ADE20K mIoU by policy and timestep at canvas grids $32 \times 32$ and $64 \times 64$ on Figure~\ref{fig:main-ade20k-in1k-results}\textbf{B} and Table~\ref{supp:tab-policy-eval}. Table~\ref{supp:tab-ade20k-canvas-sweep} reports best-$t$ mIoU at canvas grids $c \times c$ for $c \in \{8, 16, 32, 64\}$.
Figure~\ref{supp:fig-canvas-grid-impact} visualizes CanViT-B mIoU at canvas grids $c \times c$ ($c \in \{8, 16, 32, 64\}$) under the C2F policy across all $T = \adeMaxT$ timesteps.

\begin{table}[h]
  \caption{\textbf{ADE20K mIoU (\%) by policy and timestep ($T = 21$), frozen CanViT-B.}
  Cells: $\pm$ \bootstrapCiPct\% bootstrap CI half-width ($n = \adePolicyEvalStochN$ runs) for stochastic policies.
  $^\dagger$: deterministic policy.
  \textbf{Bold} = best policy and canvas size per timestep.}
  \label{supp:tab-policy-eval}
  \centering
  \small
  \setlength{\tabcolsep}{4pt}%
  \begin{tabular}{lccccccccc}
    \toprule
    \textbf{Policy} & \textbf{Canvas} & \textbf{$t{=}0$} & \textbf{$t{=}1$} & \textbf{$t{=}2$} & \textbf{$t{=}3$} & \textbf{$t{=}4$} & \textbf{$t{=}9$} & \textbf{$t{=}16$} & \textbf{$t{=}20$} \\
    \midrule
        \policyswatch{2CA02C}EG-C2F$^\dagger$ & $32^2$ & 38.5 & 41.1 & 42.0 & 42.7 & 43.2 & 43.9 & 44.2 & 44.1 \\
    \policyswatch{2CA02C}EG-C2F$^\dagger$ & $64^2$ & \textbf{39.6} & \textbf{42.2} & \textbf{43.3} & \textbf{44.1} & \textbf{44.7} & \textbf{45.6} & \textbf{45.8} & 45.7 \\
    \policyswatch{1F77B4}C2F & $32^2$ & 38.5 & \makecell{40.3 \\ $\pm0.09$} & \makecell{41.4 \\ $\pm0.09$} & \makecell{42.3 \\ $\pm0.05$} & \makecell{43.2 \\ $\pm0.03$} & \makecell{43.6 \\ $\pm0.03$} & \makecell{44.0 \\ $\pm0.05$} & \makecell{44.2 \\ $\pm0.02$} \\
    \policyswatch{1F77B4}C2F & $64^2$ & \textbf{39.6} & \makecell{41.3 \\ $\pm0.10$} & \makecell{42.5 \\ $\pm0.08$} & \makecell{43.6 \\ $\pm0.06$} & \textbf{\makecell{44.7 \\ $\pm0.03$}} & \makecell{45.1 \\ $\pm0.05$} & \makecell{45.6 \\ $\pm0.03$} & \textbf{\makecell{45.9 \\ $\pm0.02$}} \\
    \policyswatch{1A1A1A}F-IID & $32^2$ & 38.5 & \makecell{40.0 \\ $\pm0.10$} & \makecell{40.8 \\ $\pm0.08$} & \makecell{41.3 \\ $\pm0.09$} & \makecell{41.7 \\ $\pm0.08$} & \makecell{42.7 \\ $\pm0.08$} & \makecell{43.2 \\ $\pm0.07$} & \makecell{43.3 \\ $\pm0.06$} \\
    \policyswatch{1A1A1A}F-IID & $64^2$ & \textbf{39.6} & \makecell{41.2 \\ $\pm0.10$} & \makecell{42.0 \\ $\pm0.09$} & \makecell{42.5 \\ $\pm0.10$} & \makecell{43.0 \\ $\pm0.09$} & \makecell{44.1 \\ $\pm0.06$} & \makecell{44.7 \\ $\pm0.08$} & \makecell{44.9 \\ $\pm0.07$} \\
    \policyswatch{FF7F0E}R-IID & $32^2$ & \makecell{18.1 \\ $\pm0.53$} & \makecell{26.0 \\ $\pm0.34$} & \makecell{30.5 \\ $\pm0.19$} & \makecell{33.2 \\ $\pm0.14$} & \makecell{35.1 \\ $\pm0.16$} & \makecell{39.1 \\ $\pm0.14$} & \makecell{41.1 \\ $\pm0.12$} & \makecell{41.6 \\ $\pm0.09$} \\
    \policyswatch{FF7F0E}R-IID & $64^2$ & \makecell{18.3 \\ $\pm0.30$} & \makecell{27.1 \\ $\pm0.37$} & \makecell{31.6 \\ $\pm0.29$} & \makecell{34.3 \\ $\pm0.19$} & \makecell{36.3 \\ $\pm0.16$} & \makecell{40.5 \\ $\pm0.17$} & \makecell{42.6 \\ $\pm0.12$} & \makecell{43.1 \\ $\pm0.13$} \\
    \policyswatch{8C564B}RFS$^\dagger$ & $32^2$ & 38.5 & 40.0 & 40.0 & 39.9 & 39.8 & 39.1 & 38.5 & 38.2 \\
    \policyswatch{8C564B}RFS$^\dagger$ & $64^2$ & \textbf{39.6} & 41.1 & 41.1 & 41.0 & 40.9 & 40.3 & 39.6 & 39.3 \\
    \policyswatch{D62728}F2C & $32^2$ & \makecell{10.8 \\ $\pm0.14$} & \makecell{18.0 \\ $\pm0.21$} & \makecell{22.3 \\ $\pm0.17$} & \makecell{25.5 \\ $\pm0.13$} & \makecell{27.8 \\ $\pm0.17$} & \makecell{34.2 \\ $\pm0.21$} & \makecell{38.6 \\ $\pm0.13$} & \makecell{41.2 \\ $\pm0.10$} \\
    \policyswatch{D62728}F2C & $64^2$ & \makecell{11.0 \\ $\pm0.38$} & \makecell{18.1 \\ $\pm0.32$} & \makecell{22.6 \\ $\pm0.29$} & \makecell{25.9 \\ $\pm0.33$} & \makecell{28.3 \\ $\pm0.29$} & \makecell{35.4 \\ $\pm0.36$} & \makecell{40.1 \\ $\pm0.19$} & \makecell{42.9 \\ $\pm0.15$} \\
 
    \bottomrule
  \end{tabular}
\end{table}

\begin{table}[h]
  \caption{\textbf{ADE20K mIoU (\%) across canvas grid sizes, frozen CanViT-B.}
  One probe per canvas grid $c$.
  Cells: best-$t$ mIoU mean, with $\pm$ \bootstrapCiPct\% bootstrap CI half-width ($n = \adeSweepN$ runs) for stochastic policies.
  $^\dagger$: deterministic policy.}
  \label{supp:tab-ade20k-canvas-sweep}
  \centering
  \small
  \begin{tabular}{lcccc}
    \toprule
    \adeSweepHeader \\
    \midrule
        \policyswatch{2CA02C}EG-C2F$^\dagger$ & 31.6 & 39.3 & 44.2 & 45.8 \\
    \policyswatch{1F77B4}C2F & \makecell{31.6 \\ $\pm0.03$} & \makecell{39.1 \\ $\pm0.04$} & \makecell{44.2 \\ $\pm0.02$} & \makecell{45.9 \\ $\pm0.02$} \\
    \policyswatch{1A1A1A}F-IID & \makecell{30.7 \\ $\pm0.04$} & \makecell{38.1 \\ $\pm0.09$} & \makecell{43.3 \\ $\pm0.06$} & \makecell{44.9 \\ $\pm0.07$} \\
    \policyswatch{FF7F0E}R-IID & \makecell{29.5 \\ $\pm0.10$} & \makecell{37.2 \\ $\pm0.07$} & \makecell{41.6 \\ $\pm0.09$} & \makecell{43.1 \\ $\pm0.13$} \\
    \policyswatch{8C564B}RFS$^\dagger$ & 30.3 & 36.8 & 40.0 & 41.1 \\
    \policyswatch{D62728}F2C & \makecell{29.2 \\ $\pm0.14$} & \makecell{36.8 \\ $\pm0.11$} & \makecell{41.2 \\ $\pm0.10$} & \makecell{42.9 \\ $\pm0.15$} \\
 
    \bottomrule
  \end{tabular}
\end{table}

\subsection{Task: ImageNet-1k classification}
\label{supp:task-in1k}

\textbf{Probe training.} As the DINOv3 authors did not release ImageNet-1k classification heads for models other than their 7B-parameter flagship, we trained and released our own linear probes for all five smaller ViT variants (Appendix~\ref{supp:in1k-dinov3-probe-training}). Since CanViT-B was pretrained to predict DINOv3 representations, it is possible to use a DINOv3-fitted probe with it. We obtain the linear classification head applied to CanViT-B's recurrent CLS token by algebraically fusing three consecutive affine transforms (CanViT's pretrained CLS projection, CLS destandardization, and our DINOv3 ViT-B IN1k linear probe) into a single LayerNorm $\to$ Linear layer.

\textbf{Fine-tuning.} For ImageNet-1k full fine-tuning, we start from an already-fit frozen classification head then unfreeze CanViT and fine-tune all model and head parameters end-to-end, a two-stage approach known as LP-FT (linear probing then fine-tuning)~\cite{kumarFineTuningCanDistort2022}. We train for \inkFtEpochs{} epochs with \inkFtOptimizer{} (peak learning rate $\inkFtLr$, weight decay $\inkFtWeightDecay$, gradient clipping at \inkFtGradClip{} max norm), batch size \inkFtBatchSize{}, and label smoothing \inkFtLabelSmoothing{}. The learning rate follows a linear warmup over \inkFtWarmupSteps{} steps then cosine decay to zero (\inkFtTotalSteps{} total steps). Training images are augmented with RandomResizedCrop($512^2$\,px, scale $\in [0.2, 1.0]$) and horizontal flips. Each step processes $T = \inkFtNGlimpses$ F-IID glimpses with \inkFtBpttMode{} backpropagation through time. We performed this fine-tuning step on TPU v6e-4 hardware using \texttt{torch\_xla} and SPMD for multi-chip parallelism.

\textbf{Evaluation.} When evaluating on the ImageNet-1k~\cite{russakovskyImageNetLargeScale2015} validation set, we preprocess images by resizing the shortest side to 512 pixels followed by a center crop. Across up to $T = \inkMaxT$ timesteps, we report frozen CanViT-B IN1k top-1 accuracy by policy and timestep at canvas grid $32 \times 32$ on Figure~\ref{fig:main-ade20k-in1k-results}\textbf{C} (frozen curves) and Table~\ref{supp:tab-in1k-frozen-eval}. Table~\ref{supp:tab-in1k-canvas-sweep} reports best-$t$ top-1 at canvas grids $c \times c$ for $c \in \{8, 16, 32, 64\}$.
We report fine-tuned CanViT-B IN1k top-1 accuracy by policy and timestep on Figure~\ref{fig:main-ade20k-in1k-results}\textbf{C} (fine-tuned curves) and Table~\ref{supp:tab-in1k-finetuned-eval}.
Canvas resolution has only a marginal effect on best-$t$ top-1 on IN1k (Table~\ref{supp:tab-in1k-canvas-sweep}), unlike on ADE20K (Table~\ref{supp:tab-ade20k-canvas-sweep}).

\begin{table}[h]
  \caption{\textbf{ImageNet-1k top-1 accuracy (\%) across timesteps, frozen CanViT-B ($32 \times 32$ canvas).}
  Cells: $\pm$ \bootstrapCiPct\% bootstrap CI half-width ($n = \inkFrozenStochN$ runs) for stochastic policies.
  $^\dagger$: deterministic policy.
  \textbf{Bold} = best policy at each timestep.}
  \label{supp:tab-in1k-frozen-eval}
  \centering
  \small
  \setlength{\tabcolsep}{4pt}%
  \begin{tabular}{lccccccccc}
    \toprule
    \textbf{Policy} & \textbf{$t{=}0$} & \textbf{$t{=}1$} & \textbf{$t{=}2$} & \textbf{$t{=}3$} & \textbf{$t{=}4$} & \textbf{$t{=}5$} & \textbf{$t{=}9$} & \textbf{$t{=}15$} & \textbf{$t{=}20$} \\
    \midrule
        \policyswatch{1F77B4}C2F & \textbf{76.79} & \textbf{\makecell{78.62 \\ $\pm0.05$}} & \textbf{\makecell{79.58 \\ $\pm0.02$}} & \textbf{\makecell{80.25 \\ $\pm0.03$}} & \textbf{\makecell{80.77 \\ $\pm0.02$}} & \textbf{\makecell{80.85 \\ $\pm0.02$}} & \textbf{\makecell{81.02 \\ $\pm0.03$}} & \textbf{\makecell{81.12 \\ $\pm0.03$}} & \makecell{81.13 \\ $\pm0.02$} \\
    \policyswatch{1A1A1A}F-IID & \textbf{76.79} & \makecell{78.50 \\ $\pm0.04$} & \makecell{79.32 \\ $\pm0.04$} & \makecell{79.81 \\ $\pm0.05$} & \makecell{80.13 \\ $\pm0.04$} & \makecell{80.37 \\ $\pm0.04$} & \makecell{80.86 \\ $\pm0.03$} & \makecell{81.09 \\ $\pm0.03$} & \textbf{\makecell{81.14 \\ $\pm0.02$}} \\
    \policyswatch{FF7F0E}R-IID & \makecell{47.73 \\ $\pm0.10$} & \makecell{65.23 \\ $\pm0.08$} & \makecell{72.23 \\ $\pm0.09$} & \makecell{75.50 \\ $\pm0.06$} & \makecell{77.24 \\ $\pm0.06$} & \makecell{78.22 \\ $\pm0.06$} & \makecell{79.71 \\ $\pm0.06$} & \makecell{80.37 \\ $\pm0.04$} & \makecell{80.56 \\ $\pm0.05$} \\
    \policyswatch{8C564B}RFS$^\dagger$ & \textbf{76.79} & 77.85 & 77.99 & 78.01 & 78.01 & 77.99 & 77.73 & 77.40 & 77.16 \\
    \policyswatch{D62728}F2C & \makecell{32.22 \\ $\pm0.07$} & \makecell{49.04 \\ $\pm0.08$} & \makecell{58.57 \\ $\pm0.07$} & \makecell{64.40 \\ $\pm0.06$} & \makecell{68.23 \\ $\pm0.10$} & \makecell{70.81 \\ $\pm0.10$} & \makecell{75.93 \\ $\pm0.08$} & \makecell{78.52 \\ $\pm0.06$} & \makecell{79.78 \\ $\pm0.05$} \\
 
    \bottomrule
  \end{tabular}
\end{table}

\ifnum\inkSweepHasData=1
\begin{table}[h]
  \caption{\textbf{ImageNet-1k top-1 accuracy (\%) across canvas grid sizes, frozen CanViT-B.}
  Same IN1k probe across columns, evaluated at canvas grid $c \times c$.
  Cells: best-$t$ top-1 mean, with $\pm$ \bootstrapCiPct\% bootstrap CI half-width ($n = \inkSweepN$ runs) for stochastic policies.
  $^\dagger$: deterministic policy.}
  \label{supp:tab-in1k-canvas-sweep}
  \centering
  \small
  \begin{tabular}{lcccc}
    \toprule
    \inkSweepHeader \\
    \midrule
        \policyswatch{1F77B4}C2F & \makecell{80.61 \\ $\pm0.02$} & \makecell{81.08 \\ $\pm0.02$} & \makecell{81.14 \\ $\pm0.03$} & \makecell{81.16 \\ $\pm0.01$} \\
    \policyswatch{1A1A1A}F-IID & \makecell{80.74 \\ $\pm0.04$} & \makecell{81.06 \\ $\pm0.03$} & \makecell{81.14 \\ $\pm0.02$} & \makecell{81.14 \\ $\pm0.03$} \\
    \policyswatch{FF7F0E}R-IID & \makecell{80.25 \\ $\pm0.03$} & \makecell{80.52 \\ $\pm0.05$} & \makecell{80.56 \\ $\pm0.05$} & \makecell{80.58 \\ $\pm0.04$} \\
    \policyswatch{8C564B}RFS$^\dagger$ & 77.78 & 77.99 & 78.01 & 78.02 \\
    \policyswatch{D62728}F2C & \makecell{79.37 \\ $\pm0.04$} & \makecell{79.69 \\ $\pm0.02$} & \makecell{79.78 \\ $\pm0.05$} & \makecell{79.80 \\ $\pm0.03$} \\
 
    \bottomrule
  \end{tabular}
\end{table}
\fi

\begin{table}[h]
  \caption{\textbf{ImageNet-1k top-1 accuracy (\%) across timesteps, fine-tuned CanViT-B ($32 \times 32$ canvas).}
  Cells: $\pm$ \bootstrapCiPct\% bootstrap CI half-width ($n = \inkFinetunedStochN$ runs) for stochastic policies.
  $^\dagger$: deterministic policy.
  \textbf{Bold} = best policy at each timestep.}
  \label{supp:tab-in1k-finetuned-eval}
  \centering
  \small
  \setlength{\tabcolsep}{4pt}%
  \begin{tabular}{lccccccccc}
    \toprule
    \textbf{Policy} & \textbf{$t{=}0$} & \textbf{$t{=}1$} & \textbf{$t{=}2$} & \textbf{$t{=}3$} & \textbf{$t{=}4$} & \textbf{$t{=}5$} & \textbf{$t{=}9$} & \textbf{$t{=}15$} & \textbf{$t{=}20$} \\
    \midrule
        \policyswatch{1F77B4}C2F & \textbf{80.59} & \textbf{\makecell{82.25 \\ $\pm0.04$}} & \textbf{\makecell{83.21 \\ $\pm0.03$}} & \textbf{\makecell{83.90 \\ $\pm0.03$}} & \textbf{\makecell{84.35 \\ $\pm0.03$}} & \textbf{\makecell{84.43 \\ $\pm0.03$}} & \textbf{\makecell{84.51 \\ $\pm0.01$}} & \textbf{\makecell{84.53 \\ $\pm0.01$}} & \textbf{\makecell{84.51 \\ $\pm0.03$}} \\
    \policyswatch{1A1A1A}F-IID & \textbf{80.59} & \makecell{82.24 \\ $\pm0.03$} & \makecell{82.95 \\ $\pm0.04$} & \makecell{83.40 \\ $\pm0.03$} & \makecell{83.70 \\ $\pm0.03$} & \makecell{83.89 \\ $\pm0.04$} & \makecell{84.25 \\ $\pm0.04$} & \makecell{84.42 \\ $\pm0.03$} & \makecell{84.43 \\ $\pm0.04$} \\
    \policyswatch{FF7F0E}R-IID & \makecell{50.99 \\ $\pm0.11$} & \makecell{68.90 \\ $\pm0.07$} & \makecell{75.90 \\ $\pm0.09$} & \makecell{79.07 \\ $\pm0.08$} & \makecell{80.66 \\ $\pm0.06$} & \makecell{81.56 \\ $\pm0.04$} & \makecell{82.97 \\ $\pm0.05$} & \makecell{83.50 \\ $\pm0.05$} & \makecell{83.67 \\ $\pm0.04$} \\
    \policyswatch{8C564B}RFS$^\dagger$ & \textbf{80.59} & 81.59 & 81.81 & 81.76 & 81.62 & 81.53 & 81.21 & 80.69 & 80.43 \\
    \policyswatch{D62728}F2C & \makecell{35.77 \\ $\pm0.10$} & \makecell{53.47 \\ $\pm0.07$} & \makecell{63.01 \\ $\pm0.07$} & \makecell{68.78 \\ $\pm0.08$} & \makecell{72.47 \\ $\pm0.09$} & \makecell{74.91 \\ $\pm0.09$} & \makecell{79.38 \\ $\pm0.07$} & \makecell{81.40 \\ $\pm0.05$} & \makecell{82.38 \\ $\pm0.04$} \\
 
    \bottomrule
  \end{tabular}
\end{table}

\subsection{Methodology for Figure~\ref{fig:resolution-and-mask-size-analysis} analyses}
\label{supp:iou-obj}

\textbf{Panels A and B.} We evaluate the impact of canvas resolution on ADE20K segmentation performance, notably examining how IoU varies as a function of ground truth mask size.
To do so, we collect all (scene, class) pairs from the validation set (i.e., where the ground-truth mask is non-empty).
For the mask corresponding to each such pair, we measure its area as the ratio between mask and scene pixels and compute per-mask IoU with the predictions from frozen DINOv3 and frozen CanViT-B models.
Unlike the class-level mIoU reported elsewhere, per-mask IoU is computed for each non-zero mask instance rather than by summing pixels across the full validation set, and consequently does not consider false positive predictions from scenes where the class is absent.

We use $128^2$\,px input images for DINOv3 and CanViT-B models in Figure~\ref{fig:resolution-and-mask-size-analysis}\,A--B.
To obtain a desired output resolution, we fix CanViT-B's canvas resolution to $8^2$, $16^2$, $32^2$, or $64^2$ at inference.
DINOv3's output resolution is determined by its input resolution; with its $16^2$\,px patch size, it produces an $8^2$ feature map.
Each model-resolution pair under consideration is evaluated with its own probe, yielding four CanViT-B probes (one per canvas resolution) and one DINOv3 probe.
This pipeline produces one (scene, class, mask area, timestep, IoU) datapoint per ground-truth mask per model.
For CanViT, timesteps range from $t = \lowessTDeltaEarly$ to $t = \lowessTDeltaLate$ under an EG-C2F policy.
At $t = 0$, Table~\ref{supp:tab-passive-comparison} reports the corresponding passive comparison against DINOv3 ViT-B/16 (our passive teacher) and DINOv3 ViT-S/16.

To compute the trends shown in Figure~\ref{fig:resolution-and-mask-size-analysis}\,A--B, we select the first timestep (\textbf{A}) or a difference of timesteps (\textbf{B}) and apply locally-weighted scatter-plot smoothing to the scatter of (mask area, IoU) datapoints, using the Python package statsmodels~\cite{seabold2010} with a smoothing fraction of \lowessFrac{}. We disable iterative-based reweightings (i.e., we set $\mathrm{it} = \lowessIt$) to prevent outliers from being artificially downweighted. The \lowessCiPct\% confidence intervals were computed with \lowessNBoot{} bootstrap resamples.

\textbf{Panel C.} For CanViT-B, we sweep the EG-C2F glimpse budget from $t = 0$ to $t = 20$ at four canvas resolutions ($8^2$, $16^2$, $32^2$, $64^2$), producing one curve per grid.
For DINOv3 ViT-B/16, we sweep input resolution over $128$, $144$, $160$, $192$, $256$, $384$, and $512$\,px, with one probe per resolution.

\begin{table}[h]
  \caption{\textbf{Passive-vision comparison: ADE20K mIoU at $t = 0$ (single full-scene glimpse).} CanViT-B vs.\ DINOv3 ViT-B/16 and ViT-S/16 at various input and output resolutions. All models use frozen features with a linear segmentation probe.}
  \label{supp:tab-passive-comparison}
  \centering
  \small
  \setlength{\tabcolsep}{4pt}%
  \begin{tabular}{lcccc}
    \toprule
    \textbf{Model} & \textbf{Input} & \textbf{Scene grid} & \textbf{GFLOPs} & \textbf{mIoU (\%)} \\
    \midrule
        DINOv3 ViT-S/16 & 128\,px & $8 \times 8$ & 3.1 & 25.2 \\
     & 144\,px & $9 \times 9$ & 3.8 & 27.1 \\
     & 160\,px & $10 \times 10$ & 4.7 & 29.4 \\
     & 192\,px & $12 \times 12$ & 6.8 & 32.3 \\
     & 256\,px & $16 \times 16$ & 12.5 & 36.9 \\
     & 384\,px & $24 \times 24$ & 31.3 & 40.9 \\
     & 512\,px & $32 \times 32$ & 63.9 & 43.3 \\
    \midrule
    DINOv3 ViT-B/16 & 128\,px & $8 \times 8$ & 12.0 & 28.8 \\
     & 144\,px & $9 \times 9$ & 15.0 & 31.0 \\
     & 160\,px & $10 \times 10$ & 18.4 & 33.2 \\
     & 192\,px & $12 \times 12$ & 26.3 & 35.9 \\
     & 256\,px & $16 \times 16$ & 47.2 & 40.6 \\
     & 384\,px & $24 \times 24$ & 112.0 & 44.8 \\
     & 512\,px & $32 \times 32$ & 215.3 & 47.2 \\
    \midrule
    CanViT-B (t=0, full scene) & 128\,px & $8 \times 8$ & 13.8 & 29.3 \\
     & 128\,px & $9 \times 9$ & 13.9 & 29.5 \\
     & 128\,px & $10 \times 10$ & 13.9 & 31.5 \\
     & 128\,px & $12 \times 12$ & 14.0 & 33.4 \\
     & 128\,px & $16 \times 16$ & 14.2 & 35.4 \\
     & 128\,px & $24 \times 24$ & 14.9 & 37.6 \\
     & 128\,px & $32 \times 32$ & 15.8 & 38.5 \\
     & 128\,px & $64 \times 64$ & 22.1 & 39.6 \\
 
    \bottomrule
  \end{tabular}
\end{table}

\begin{figure}[!ht]
  \centering
  \includegraphics{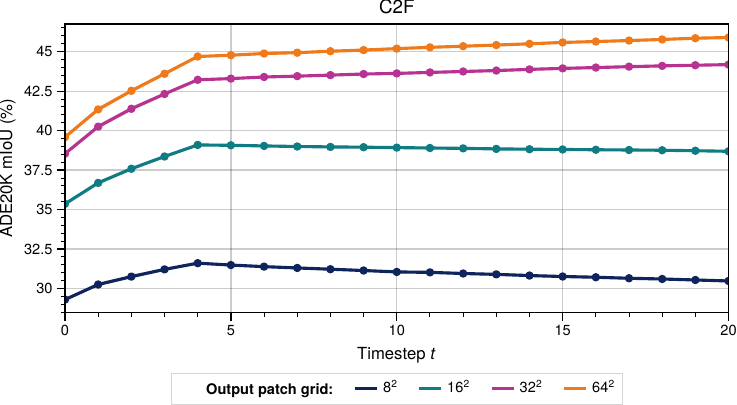}
  \caption{\textbf{Impact of canvas grid size on CanViT-B ADE20K mIoU (C2F policy, $T = \adeMaxT$).}
  Mean mIoU across $n = \canvasGridImpactN$ evaluation runs, at each timestep, at canvas grids $c \times c$ ($c \in \{\canvasGridImpactGrids\}$).}
  \label{supp:fig-canvas-grid-impact}
\end{figure}

\subsection{DINOv3 ImageNet-1k linear classification probes}
\label{supp:in1k-dinov3-probe-training}

DINOv3~\cite{simeoniDINOv32026} was released with a pretrained ImageNet-1k linear classification head only for the 7B-parameter flagship model, not for the smaller ViT checkpoints (ViT-S/16, ViT-S+/16, ViT-B/16, ViT-L/16, ViT-H+/16). Additionally, only ImageNet-ReaL top-1 accuracy, rather than standard ImageNet-1k top-1 validation accuracy, was reported for those non-flagship checkpoints.
We trained linear probes for all five smaller DINOv3 ViT models, which all match or exceed the best ImageNet-ReaL top-1 validation accuracy reported by the DINOv3 authors (Table~\ref{tab:dinov3-probes}).

\textbf{Feature extraction.} We extract CLS tokens from the last transformer layer of each DINOv3 ViT model at $512 \times 512$ input resolution (1024 patch tokens with $16 \times 16$\,px patches), using bfloat16 automatic mixed precision for the forward pass. CLS tokens are stored in float32. Data loading uses Inception-crop augmentation for the training split. Features are pre-computed and cached to disk before any probe training begins. This decouples the expensive backbone forward pass from the probe optimization, which iterates over cached feature vectors rather than running the backbone or loading images. This makes extensive hyperparameter search over the full training set practical. Since augmentation is applied at extraction time and baked into the cached features, each extraction run through the training split (with different random shuffling and augmentation) produces a separate file. To train with $N$ augmentation epochs, $N$ extraction passes must be run and stored. During probe training, each file is treated as an independent epoch.

\textbf{Hyperparameter (HP) search.} For each backbone, we sweep the following HPs with Optuna~\cite{akibaOptunaNextgenerationHyperparameter2019} to maximize ImageNet-ReaL top-1 validation accuracy~\cite{beyerAreWeDone2020}: loss (softmax or sigmoid cross-entropy); optimizer (AdamW or SGD); reference learning rate ($10^{-7}$ to $10^{-2}$, log-uniform); batch size (one of the top three powers of 2 fitting in GPU memory for that backbone); AdamW weight decay ($0$ to $0.1$, uniform), $\beta_1$ ($0.1$ to $0.99$, log-uniform), and $\beta_2 = \beta_1 + g (1 - \beta_1)$ with $g \in [0.01, 1.0]$ log-uniform (ensuring $\beta_2 > \beta_1$); SGD momentum ($0$ to $0.99$, uniform); and weight initialization (Gaussian $\mathcal{N}(0, 0.01)$ with zero bias, or PyTorch default). The peak learning rate is the reference learning rate times the batch size (linear scaling rule), with linear warmup over the first 10\% of steps followed by cosine annealing to zero. We report selected HPs in Table~\ref{tab:dinov3-probe-hparams}.

\begin{table}[h]
  \caption{\textbf{DINOv3 IN1k linear probe accuracy at $512 \times 512$ input resolution.}}
  \label{tab:dinov3-probes}
  \centering
  \begin{tabular}{lcc}
    \toprule
    \textbf{Model} & \textbf{IN-ReaL top-1 (official / ours)} & \textbf{IN1k top-1 (ours)} \\
    \midrule
        ViT-S/16 & 87.0\% / \textbf{87.08\%} & 81.40\% \\
    ViT-S+/16 & 88.0\% / \textbf{88.08\%} & 82.89\% \\
    ViT-B/16 & 89.3\% / \textbf{89.54\%} & 85.00\% \\
    ViT-L/16 & 90.2\% / \textbf{90.42\%} & 87.44\% \\
    ViT-H+/16 & 90.3\% / \textbf{90.31\%} & 87.65\% \\
 
    \bottomrule
  \end{tabular}
\end{table}

\begin{table}[h]
  \caption{\textbf{Selected hyperparameters for each DINOv3 IN1k linear probe.}}
  \label{tab:dinov3-probe-hparams}
  \centering
  \small
  \setlength{\tabcolsep}{2pt}%
  \begin{tabular}{@{}lccccccc@{}}%
    \toprule
    \textbf{Model} & \textbf{Loss} & \textbf{Batch size} & \textbf{Peak LR} & \textbf{Weight decay} & \boldmath$(\beta_1, \beta_2)$ & \textbf{DINOv3 init} & \textbf{Epochs (aug., total)} \\
    \midrule
        ViT-S/16 & softmax & 2048 & $6.6 \times 10^{-3}$ & $6.0 \times 10^{-2}$ & $(0.22, 0.71)$ & Yes & 10, 20 \\
    ViT-S+/16 & softmax & 1024 & $1.1 \times 10^{-3}$ & $3.8 \times 10^{-2}$ & $(0.10, 0.94)$ & No & 10, 20 \\
    ViT-B/16 & softmax & 1024 & $2.8 \times 10^{-4}$ & $9.2 \times 10^{-2}$ & $(0.56, 0.73)$ & Yes & 10, 20 \\
    ViT-L/16 & sigmoid & 2048 & $3.0 \times 10^{-3}$ & $3.1 \times 10^{-2}$ & $(0.40, 0.63)$ & Yes & 10, 20 \\
    ViT-H+/16 & sigmoid & 1024 & $8.2 \times 10^{-3}$ & $4.9 \times 10^{-2}$ & $(0.78, 0.93)$ & No & 8, 16 \\
 
    \bottomrule
  \end{tabular}
\end{table}

\FloatBarrier
\section{Pretraining ablations}
\label{supp:ablation-study}

\begin{figure}[h]
  \centering
  \includegraphics[width=\linewidth]{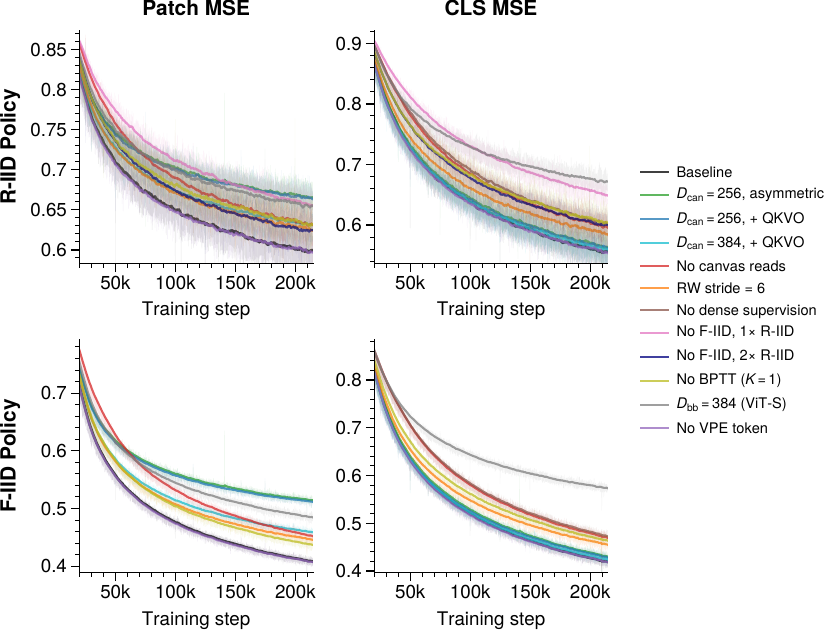}
  \caption{\textbf{Ablation study: loss curves during ImageNet-21k pretraining.}
  Faint lines: logged per-batch losses. Bold lines: exponential moving average of logged losses.}
  \label{fig:ablation-loss-curves}
\end{figure}

\begin{table}[h]
  \caption{\textbf{Ablation study: reconstruction quality on ADE20K validation set.}
  We compare CanViT's predicted DINOv3 ViT-B features against reference teacher features after \ablT{} \ablPolicyName{} glimpses.
  Params: the model alone, without readouts or probe. GFLOPs: per-timestep forward pass cost, including the ADE20K probe.
  $\Delta$: \% change in cosine similarity relative to baseline (mean over $n = \ablNRuns$ evaluation runs per variant).
  Sp.\ = spatial (dense, patch-level); CLS = CLS-token.}
  \label{tab:ablations}
  \centering
  \small
  \begin{tabular}{lcccc}
    \toprule
    \textbf{Variant} & \textbf{Params} & \textbf{GFLOPs} & \textbf{$\Delta$ Sp.\ Cos} & \textbf{$\Delta$ CLS Cos} \\
    \midrule
        CanViT-B (baseline) & 95.2M & 15.8 & ref & ref \\
    \midrule
    \textbf{a} $D_{\mathrm{can}}=256$, asymmetric ($n_h = 2$) & 88.1M & 13.2 & $-$12.0\% & $-$2.1\% \\
    \textbf{b} $D_{\mathrm{can}}=256$, + QKVO ($n_h = 2$) & 88.8M & 14.8 & $-$11.7\% & $-$1.6\% \\
    \textbf{c} $D_{\mathrm{can}}=384$, + QKVO ($n_h = 3$) & 91.0M & 17.3 & $-$5.8\% & $-$1.0\% \\
    \textbf{d} No canvas reads & 90.4M & 14.2 & $-$6.5\% & $-$7.9\% \\
    \textbf{e} RW stride = 6 (1R / 1W) & 88.8M & 13.7 & $-$3.9\% & $-$6.0\% \\
    \textbf{f} No dense supervision & 95.2M & 15.8 & $-$98.7\% & $-$8.9\% \\
    \textbf{g} No F-IID, 1$\times$ R-IID & 95.2M & 15.8 & $-$9.4\% & $-$16.2\% \\
    \textbf{h} No F-IID, 2$\times$ R-IID & 95.2M & 15.8 & $-$5.3\% & $-$9.4\% \\
    \textbf{i} No BPTT ($K = 1$) & 95.2M & 15.8 & $-$3.8\% & $-$7.5\% \\
    \textbf{j} $D_{\mathrm{bb}}=384$ (ViT-S) & 26.4M & 5.9 & $-$8.4\% & $-$21.1\% \\
    \textbf{k} No VPE token & 95.2M & 15.6 & $-$0.1\% & $-$0.3\% \\
 
    \bottomrule
  \end{tabular}
\end{table}

To assess the influence of our design choices at the level of architecture and pretraining, we conducted an ablation study using short pretraining runs.
For our ablation baseline and each ablated variant, we allocated approximately \ablStepsK{} optimizer steps (\ablGlimpsesM{}M glimpses for the baseline, slightly over 10\% of our flagship's pretraining compute), with a learning rate warmup of 20k steps.

We report training loss curves on Figure~\ref{fig:ablation-loss-curves}, disaggregated by reconstruction target (patch/CLS) and by policy (R-IID/F-IID).
As our chosen step count results in slightly more than 1 epoch on the 13.2 million ImageNet-21k images in our pretraining dataset, these training curves are representative of the model's generalization capabilities.

To further quantify generalization, we evaluated each ablation variant on the ADE20K validation set by measuring cosine similarity between the canvas reconstruction and DINOv3 ViT-B teacher features under the R-IID policy.
We report the relative change compared to baseline on Table~\ref{tab:ablations}, alongside computational footprint and parameter count, with \bootstrapCiPct\% CIs and $t = 0$ results in Table~\ref{tab:ablation-eval-supp}.
To relate reconstruction quality to downstream accuracy, we also evaluated the same frozen checkpoints on ADE20K and ImageNet-1k (Table~\ref{tab:ablation-downstream}). For ADE20K, we trained one probe per checkpoint as in Appendix~\ref{supp:task-ade20k}. For ImageNet-1k, we fused each checkpoint's CLS projection with our DINOv3 ViT-B probe as in Appendix~\ref{supp:task-in1k}, without further training.

\begin{table}[t]
  \caption{\textbf{Ablation study: reconstruction quality on ADE20K validation set (supplementary).}
  Cosine similarity (\%) with DINOv3 ViT-B teacher features; cells are \bootstrapCiPct\% bootstrap CIs across $n = \ablNRuns$ independent \ablPolicyName{} runs per variant.
  Sp.\ = spatial (patch-level), CLS = CLS-token.}
  \label{tab:ablation-eval-supp}
  \centering
  \small
  \begin{tabular}{lcccc}
    \toprule
    \multirow{2}{*}{\textbf{Variant}}
      & \multicolumn{2}{c}{\textbf{$t = 0$}}
      & \multicolumn{2}{c}{\textbf{$t = \ablTLast$}} \\
    \cmidrule(l){2-3} \cmidrule(l){4-5}
    & Sp.\ & CLS & Sp.\ & CLS \\
    \midrule
        CanViT-B (baseline) & 46.9--47.4 & 46.4--47.2 & 71.3--71.4 & 69.5--69.7 \\
    \midrule
    \textbf{a} $D_{\mathrm{can}}=256$, asymmetric & 42.5--42.8 & 45.6--46.2 & 62.7--62.8 & 68.1--68.2 \\
    \textbf{b} $D_{\mathrm{can}}=256$, + QKVO & 42.9--43.2 & 46.4--46.6 & 63.0--63.1 & 68.4--68.6 \\
    \textbf{c} $D_{\mathrm{can}}=384$, + QKVO & 44.5--45.0 & 45.5--46.3 & 67.1--67.3 & 68.8--69.0 \\
    \textbf{d} No canvas reads & 45.1--45.7 & 44.1--44.7 & 66.6--66.7 & 64.0--64.2 \\
    \textbf{e} RW stride = 6 & 45.2--45.3 & 44.3--44.5 & 68.5--68.5 & 65.3--65.5 \\
    \textbf{f} No dense supervision & 0.9--0.9 & 42.9--43.6 & 0.9--0.9 & 63.3--63.5 \\
    \textbf{g} No F-IID, 1$\times$ R-IID & 43.7--44.0 & 40.1--40.5 & 64.6--64.7 & 58.2--58.3 \\
    \textbf{h} No F-IID, 2$\times$ R-IID & 45.4--46.1 & 43.0--43.9 & 67.5--67.6 & 63.0--63.1 \\
    \textbf{i} No BPTT ($K = 1$) & 45.9--46.3 & 45.1--45.3 & 68.5--68.7 & 64.4--64.5 \\
    \textbf{j} $D_{\mathrm{bb}}=384$ (ViT-S) & 42.9--43.3 & 37.6--38.0 & 65.3--65.4 & 54.9--55.0 \\
    \textbf{k} No VPE token & 46.4--47.3 & 45.9--46.8 & 71.2--71.3 & 69.4--69.5 \\
 
    \bottomrule
  \end{tabular}
\end{table}

\begin{table}[t]
  \caption{\textbf{Ablation study: ADE20K mIoU (\%) and ImageNet-1k top-1 accuracy (\%) of the frozen ablation checkpoints ($\ablDownstreamGrid \times \ablDownstreamGrid$ canvas).}
  ADE20K: \ablDownstreamSegPolicy{} policy. ImageNet-1k: \ablDownstreamClfPolicy{} policy, with $\pm$ \bootstrapCiPct\% bootstrap CI half-width ($n = \ablDownstreamN$ runs).}
  \label{tab:ablation-downstream}
  \centering
  \small
  \begin{tabular}{lcccc}
    \toprule
    \multirow{2}{*}{\textbf{Variant}}
      & \multicolumn{2}{c}{\textbf{ADE20K mIoU}}
      & \multicolumn{2}{c}{\textbf{IN1k top-1}} \\
    \cmidrule(l){2-3} \cmidrule(l){4-5}
    & $t = 0$ & $t = 20$ & $t = 0$ & $t = 20$ \\
    \midrule
        CanViT-B (baseline) & 29.2 & 37.9 & 61.31 & 69.16 $\pm$ 0.04 \\
    \midrule
    \textbf{a} $D_{\mathrm{can}}=256$, asymmetric & 26.3 & 34.1 & 60.03 & 66.97 $\pm$ 0.03 \\
    \textbf{b} $D_{\mathrm{can}}=256$, + QKVO & 26.8 & 34.8 & 60.14 & 67.38 $\pm$ 0.03 \\
    \textbf{c} $D_{\mathrm{can}}=384$, + QKVO & 27.9 & 36.5 & 61.08 & 67.94 $\pm$ 0.04 \\
    \textbf{d} No canvas reads & 25.3 & 31.9 & 55.94 & 60.00 $\pm$ 0.03 \\
    \textbf{e} RW stride = 6 & 26.8 & 33.8 & 57.10 & 62.60 $\pm$ 0.04 \\
    \textbf{f} No dense supervision & 12.3 & 17.4 & 54.29 & 59.85 $\pm$ 0.05 \\
    \textbf{g} No F-IID, 1$\times$ R-IID & 18.8 & 29.6 & 39.18 & 50.05 $\pm$ 0.04 \\
    \textbf{h} No F-IID, 2$\times$ R-IID & 22.2 & 32.3 & 47.80 & 58.97 $\pm$ 0.03 \\
    \textbf{i} No BPTT ($K = 1$) & 27.6 & 34.2 & 57.04 & 61.59 $\pm$ 0.03 \\
    \textbf{j} $D_{\mathrm{bb}}=384$ (ViT-S) & 24.4 & 32.4 & 44.31 & 52.15 $\pm$ 0.02 \\
    \textbf{k} No VPE token & 29.3 & 37.0 & 61.39 & 68.71 $\pm$ 0.02 \\
 
    \bottomrule
  \end{tabular}
\end{table}

\textbf{Capacity--expressiveness trade-offs.}
The removal of canvas-side QKVO projections is key to the low overhead of Canvas Attention.
On any given training or inference budget, this allows for more frequent canvas--backbone interactions, a larger canvas embedding dimension (semantic resolution), and the use of more canvas patches (spatial resolution).
However, at fixed canvas dimensionality, the ablation of these projections reduces the expressiveness of each individual cross-attention operation.
To assess the well-foundedness of this trade-off, we reduced canvas dimensionality from $D_{\mathrm{can}} = 1024$ to $D_{\mathrm{can}} = 256$, leading to a dramatic drop in patch reconstruction quality (\abldelta{DcanTwoFiveSix}{SceneNorm}, Table~\ref{tab:ablations}\,\textbf{a}).
Re-introducing canvas-side QKVO projections in a FLOP-matched manner forces the use of a small canvas, resulting in a loss of per-position information capacity and a failure to rescue reconstruction quality via increased expressiveness (\abldelta{QkvoDcanTwoFiveSix}{SceneNorm} and \abldelta{QkvoDcanThreeEightFour}{SceneNorm} respectively, Table~\ref{tab:ablations}\,\textbf{b,c}).

\textbf{Frequency and directionality of canvas--backbone interaction.}
CanViT interleaves Canvas Attention Read/Write operations along depth.
In CanViT-B, this corresponds to 3 reads and 3 writes per glimpse (R/W stride of 2), evenly spread across its ViT-B backbone's 12 Transformer blocks.
Write operations are required in order to update the canvas and produce dense outputs.
In contrast, Read operations can be readily ablated; doing so results in a large drop in both patch-level (\abldelta{NoReads}{SceneNorm}) and CLS-level (\abldelta{NoReads}{ClsNorm}) reconstruction quality (Table~\ref{tab:ablations}\,\textbf{d}).
This highlights the benefit of canvas-to-backbone communication, which underpins top-down recurrent feedback across timesteps, indirect canvas-to-canvas interaction via the backbone, and generally allows backbone-side computation to benefit from the high-capacity workspace constituted by the canvas.
When increasing the R/W stride from 2 to 6 Transformer blocks, which results in just 1 read and 1 write per glimpse, we observe a similar yet slightly less pronounced effect (\abldelta{RwStrideSix}{SceneNorm}, Table~\ref{tab:ablations}\,\textbf{e}).
Together, these results show the benefits of frequent, bidirectional Canvas Attention operations, and point at the importance of within-glimpse canvas refinement and contextually-aware backbone computation.

\textbf{Dense latent supervision.}
Omitting dense supervision is contrary to the goal of using a frozen CanViT's canvas features for dense tasks.
However, this objective-level intervention has no effect on the model's architecture or raw expressiveness, and could theoretically enhance CLS reconstruction by allowing the model's representations to be specialized for this purpose, rather than requiring them to support both CLS-level and patch-level reconstruction.
In our ablation study, the reverse was true (\abldelta{NoDense}{ClsNorm}, Table~\ref{tab:ablations}\,\textbf{f}): removing patch-level supervision \emph{degrades} CLS reconstruction, indicating that the patch-level objective also benefits CLS-level learning rather than competing with it for capacity.

\textbf{F-IID rollouts.}
During pretraining, we average losses and gradients across two rollouts that start from a random (R-IID) or full-scene zoomed-out (F-IID) viewpoint for each scene.
The inclusion of an F-IID rollout ensures that at least one glimpse has full spatial coverage and that the full-scene viewpoint, $(x = 0, y = 0, s = 1)$, can be seen during training.
Simply removing the F-IID rollout (\abldelta{NoFiidOneRiid}{SceneNorm} spatial, Table~\ref{tab:ablations}\,\textbf{g}) dramatically slows down the decrease \emph{of the R-IID loss}; however, this also halves the total number of glimpses per optimizer step.
To control for this, we trained a second variant that replaces the F-IID rollout with a second R-IID rollout, preserving the total number of glimpses per step (Table~\ref{tab:ablations}\,\textbf{h}).
The controlled variant still incurs a substantial degradation (\abldelta{NoFiidTwoRiid}{SceneNorm} spatial, \abldelta{NoFiidTwoRiid}{ClsNorm} CLS), confirming that the full-scene viewpoint itself is beneficial, beyond its contribution as an additional rollout.

\textbf{Temporal credit assignment.}
CanViT uses the smallest possible truncated \ac{BPTT} chunk size, $K = 2$, in order for gradient updates to take into account a glimpse and its successor.
Setting $K = 1$ roughly halves the backward-pass memory footprint, but eliminates gradient flow across time altogether.
Given temporally-dense supervision and highly-informative canvas tokens, it may still be possible to obtain meaningful results with $K=1$, as the model learns to greedily produce a best-guess reconstruction at each timestep, which incidentally produces an informative canvas for the next timestep to reuse.
In this ablation, we also decrease the stop probability from 0.5 to 0.25 to keep the expected number of glimpses per optimizer step comparable.
We find that removing BPTT (Table~\ref{tab:ablations}\,\textbf{i}) degrades both spatial (\abldelta{NoBptt}{SceneNorm}) and CLS (\abldelta{NoBptt}{ClsNorm}) reconstruction, indicating that even minimal temporal gradient flow ($K = 2$) contributes meaningfully to learning.

\textbf{Backbone embedding dimension.}
Reducing the backbone embedding dimension from $D_{\mathrm{bb}} = 768$ to $D_{\mathrm{bb}} = 384$ is exactly equivalent to using a ViT-S backbone, rather than ViT-B.
This results in the largest drop in parameter count, per-glimpse computational footprint, and CLS reconstruction quality (\abldelta{VitS}{ClsNorm}) across all ablations.
However, the impact of a narrower backbone on patch (spatial) reconstruction (\abldelta{VitS}{SceneNorm}), while significant, remains lower than that of several other ablations.

\textbf{VPE token.}
Among all considered ablations, the removal of the VPE token had the lowest impact across both policies and both loss types (\abldelta{NoVpe}{SceneNorm} spatial, \abldelta{NoVpe}{ClsNorm} CLS; Figure~\ref{fig:ablation-loss-curves}, Table~\ref{tab:ablations}\,\textbf{k}).
However, the downstream results of the same checkpoints (Table~\ref{tab:ablation-downstream} and Appendix~\ref{supp:vpe-long-horizon}) show that VPE is beneficial at long horizons and high resolution.

\FloatBarrier
\section{Inference latency}
\label{supp:latency-experiments}

\begin{figure}[h]
  \centering
  \includegraphics[width=\linewidth]{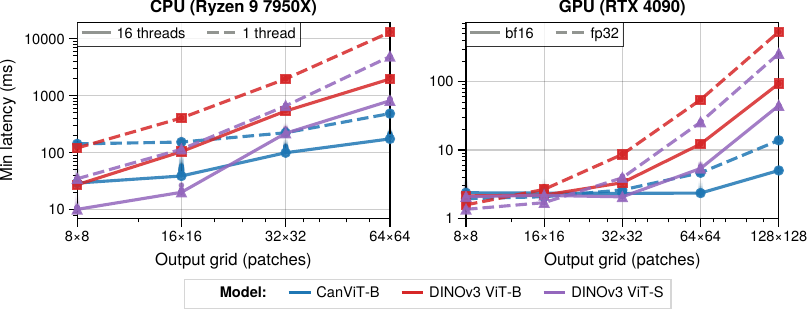}
  \caption{\textbf{Latency scaling with output resolution.} Faint scatter: individual iterations.}
  \label{fig:inference-efficiency}
\end{figure}

\begin{table}[h]
  \caption{\textbf{Latency and peak memory by hardware, precision, and output resolution.}
  \protect\newline\textbf{Bold} = lowest latency and peak memory per row.}
  \label{tab:hw-bench}
  \centering
  \footnotesize%
  \setlength{\tabcolsep}{4pt}%
  \setlength{\aboverulesep}{1pt}\setlength{\belowrulesep}{1pt}%
  \begin{tabular}{llcccccc}
    \toprule
    \multirow{2}{*}{\textbf{Device}} & \multirow{2}{*}{\textbf{Scene grid}}
      & \multicolumn{2}{c}{\textbf{CanViT-B}}
      & \multicolumn{2}{c}{\textbf{DINOv3 ViT-B/16}}
      & \multicolumn{2}{c}{\textbf{DINOv3 ViT-S/16}} \\
    \cmidrule(l){3-4} \cmidrule(l){5-6} \cmidrule(l){7-8}
    & & ms & MB & ms & MB & ms & MB \\
    \midrule
        CPU, 1T & $8 \times 8$ & 143 & --- & 121 & --- & \textbf{34.8} & --- \\
    CPU, 1T & $16 \times 16$ & 154 & --- & 407 & --- & \textbf{115} & --- \\
    CPU, 1T & $32 \times 32$ & \textbf{224} & --- & 1975 & --- & 663 & --- \\
    CPU, 1T & $64 \times 64$ & \textbf{486} & --- & 13049 & --- & 4882 & --- \\
    \midrule
    CPU, 16T & $8 \times 8$ & 29.3 & --- & 27.4 & --- & \textbf{10.1} & --- \\
    CPU, 16T & $16 \times 16$ & 39.0 & --- & 104 & --- & \textbf{20.2} & --- \\
    CPU, 16T & $32 \times 32$ & \textbf{99.8} & --- & 539 & --- & 219 & --- \\
    CPU, 16T & $64 \times 64$ & \textbf{175} & --- & 1969 & --- & 817 & --- \\
    \midrule
    CUDA, AMP bf16 & $8 \times 8$ & 2.36 & 397 & 2.12 & 359 & \textbf{2.07} & \textbf{100} \\
    CUDA, AMP bf16 & $16 \times 16$ & 2.33 & 402 & 2.18 & 362 & \textbf{2.16} & \textbf{101} \\
    CUDA, AMP bf16 & $32 \times 32$ & 2.31 & 419 & 3.30 & 388 & \textbf{2.05} & \textbf{115} \\
    CUDA, AMP bf16 & $64 \times 64$ & \textbf{2.33} & 499 & 12.2 & 416 & 5.34 & \textbf{135} \\
    CUDA, AMP bf16 & $128 \times 128$ & \textbf{4.99} & 794 & 93.3 & 607 & 44.6 & \textbf{250} \\
    \midrule
    CUDA, fp32 & $8 \times 8$ & 1.93 & 395 & 1.57 & 356 & \textbf{1.35} & \textbf{99} \\
    CUDA, fp32 & $16 \times 16$ & 2.06 & 408 & 2.64 & 359 & \textbf{1.69} & \textbf{101} \\
    CUDA, fp32 & $32 \times 32$ & \textbf{2.59} & 430 & 8.60 & 378 & 3.90 & \textbf{112} \\
    CUDA, fp32 & $64 \times 64$ & \textbf{4.62} & 544 & 54.3 & 454 & 25.5 & \textbf{154} \\
    CUDA, fp32 & $128 \times 128$ & \textbf{13.8} & 996 & 535 & 758 & 257 & \textbf{326} \\
 
    \bottomrule
  \end{tabular}
\end{table}

We report minimum latency and peak memory scaling in Figure~\ref{fig:inference-efficiency} and Table~\ref{tab:hw-bench}.
Benchmarks run on an NVIDIA GeForce RTX 4090 and on an AMD Ryzen 9 7950X 16-Core Processor.
CUDA benchmarks use both float32 and AMP bfloat16, and \texttt{torch.compile}.
CPU benchmarks use float32 in eager mode, with 1 thread or 16 threads (the physical-core count).

Each timed iteration processes a single input (batch size $B = 1$) with explicit device synchronization both before and after.
After \benchWarmupItersPhrase{}, measurement runs at least \benchMinItersPhrase{} and continues until either a \benchTimeBudgetS{}-second budget or a \benchMaxIters{}-iteration limit is reached, whichever comes first.
\ifnum\benchNPasses>1\relax We repeat each configuration \benchNPassesPhrase{} and pool the per-iteration latencies. \fi
Within the measurement window, we record peak GPU memory and minimum forward-pass latency.
To vary the scene grid size during benchmarking, we adjust the canvas resolution for CanViT and the input (and thus, output) resolution for DINOv3, as in Figure~\ref{fig:resolution-and-mask-size-analysis} and Appendix~\ref{supp:iou-obj}.

\FloatBarrier
\section{Interpretability}
\label{supp:interpretability}

\textbf{PCA visualizations.}
To visualize grids of high-dimensional glimpse or canvas patch tokens as RGB images (Figure~\ref{fig:example-canvit-rollout}, Figure~\ref{fig:canvit-high-level-arch}, Figure~\ref{fig:canvas-attention}, Figure~\ref{fig:canvas-evolution}), we adopt a similar approach to that of DINOv3~\cite{simeoniDINOv32026}, by performing \ac{PCA} across tokens then mapping groups of three consecutive PCs to RGB. In order to prevent global variance from washing out local detail when visualizing a subset of the canvas, we apply min-max scaling to the resulting RGB channels across the visible region. We apply this protocol to layer-normalized~\cite{baLayerNormalization2016} tokens, such that each individual token has unit variance and zero mean across its dimensions.

\begin{figure}[h]
  \centering
  \includegraphics[width=0.85\linewidth]{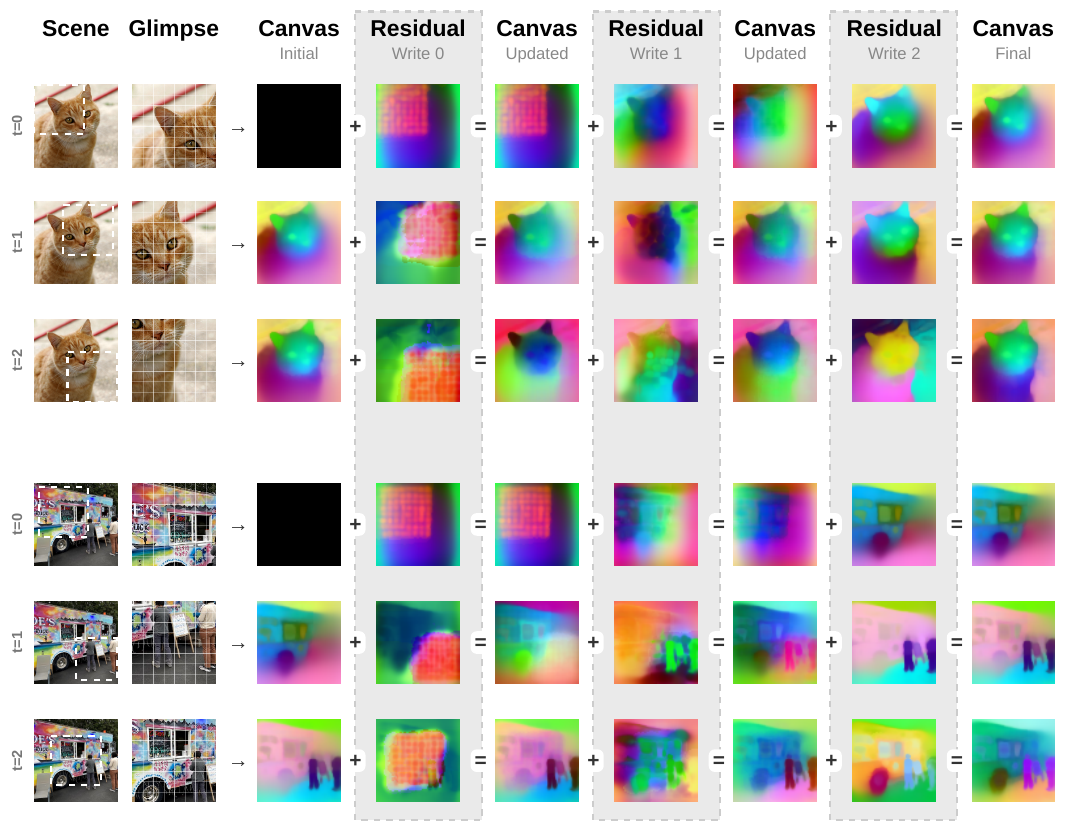}
  \caption{\textbf{Canvas updates within and across glimpses.} CanViT performs multiple Canvas Attention Write operations per glimpse, each producing a residual that is summed with the canvas (Figure~\ref{fig:canvit-high-level-arch}). To isolate the contribution of individual Write operations, we capture intermediate residuals and canvases after each Write operation within a $128^2$\,px glimpse, and visualize these snapshots with \acf{PCA} across tokens. We compute PCA bases independently for all snapshots rather than using a shared basis, in order to visualize the spatial structure of each individual snapshot. We observe an emergent progression in the spatial structure of residuals as within-glimpse processing unfolds from Write~0 to Write~2, from localized and feature-level content to scene-wide and semantically grounded content. Write-0 residuals reflect the structure of the glimpse's $8 \times 8$ patch grid. Write-1 residuals begin to reflect the structure of the objects within the glimpse. Write-2 residuals extrapolate beyond the confines of the glimpse, and delineate sharper object boundaries.}
  \label{fig:canvas-evolution}
\end{figure}

\FloatBarrier
\section{Additional experiments}
\label{supp:additional-experiments}

We ran additional experiments during the review and rebuttal period.
Most rebuttal-period experiments were run on Google Cloud using Spot and DWS Flex TPU v5p and TPU v6e instances, with code written in JAX~\cite{jax2018github} and Flax NNX~\cite{flax2020github}.

\FloatBarrier
\subsection{Impact of the pretraining dataset}
\label{supp:dataset-scale}

We pretrained our flagship on ImageNet-21k (IN21k) to maximize the quality of CanViT-B's representations through scene diversity, rather than by relying on a complex augmentation pipeline on top of the smaller IN1k.
To measure the effect of this choice, we pretrained a new CanViT-B checkpoint on IN1k instead of IN21k, using the same compute budget and hyperparameters as in our flagship results.
We found that, unlike IN21k pretraining, IN1k-only pretraining exhibited late instability, and thus retained the checkpoint that best reconstructed the teacher's features on ImageNet-1k validation images before the instability, at \dsInOneStepPct\% of training (step \dsInOneStep{} of \dsSchedule{}).
We then fit ADE20K probes to this new pretrained checkpoint.
We computed IN1k classification performance using frozen features, as well as via LP-FT fine-tuning.
For LP-FT, we fine-tuned both checkpoints anew in JAX, with the fine-tuning hyperparameters of Appendix~\ref{supp:task-in1k}.

\begin{table}[h]
  \caption{\textbf{ADE20K mIoU (\%) and ImageNet-1k top-1 accuracy (\%) by pretraining dataset, CanViT-B (C2F policy, $t = \dsT$).}
  ADE20K with a $\dsAdeGrid \times \dsAdeGrid$ canvas, IN1k with a $\dsInkGrid \times \dsInkGrid$ canvas.
  Cells: $\pm$ \bootstrapCiPct\% bootstrap CI half-width ($n \geq \dsMinN$ runs).}
  \label{tab:dataset-scale}
  \centering
  \small
  \begin{tabular}{lccc}
    \toprule
    & \textbf{ADE20K mIoU, frozen} & \textbf{IN1k top-1, frozen} & \textbf{IN1k top-1, LP-FT} \\
    \midrule
        CanViT-B, IN21k-pretrained & 45.89 $\pm$ 0.02 & 81.13 $\pm$ 0.02 & 84.44 $\pm$ 0.02 \\
    CanViT-B, IN1k-pretrained & 43.40 $\pm$ 0.03 & 81.49 $\pm$ 0.02 & 84.19 $\pm$ 0.02 \\
 
    \bottomrule
  \end{tabular}
\end{table}

Here, we observe that pretraining on IN1k instead of IN21k results in a minor drop in LP-FT top-1 IN1k classification accuracy, a slight \emph{increase} in frozen-probe IN1k accuracy, and a drop in ADE20K segmentation accuracy.
These opposite effects may be due to a potential in-domain advantage when pretraining on IN1k (train) then evaluating on IN1k (val), which vanishes given cross-dataset transfer (when evaluating on ADE20K) or given the increased risk of overfitting presented by fine-tuning (when using LP-FT on IN1k).
Regardless, with IN1k pretraining alone and C2F evaluation, \mbox{CanViT-B} still reaches \dsInOneLpft\% top-1 IN1k accuracy after LP-FT and \dsInOneAde{}\% ADE20K mIoU, ahead of every prior model in Table~\ref{tab:active-vision-comparison}.
Tables~\ref{tab:dataset-scale-ade20k}, \ref{tab:dataset-scale-in1k-frozen} and~\ref{tab:dataset-scale-in1k-lpft} report each policy's results for both checkpoints.

\begin{table}[!htbp]
  \caption{\textbf{ADE20K mIoU (\%) by policy and timestep, frozen CanViT-B pretrained on IN21k or IN1k ($\dsAdeGrid \times \dsAdeGrid$ canvas).}
  Cells: $\pm$ \bootstrapCiPct\% bootstrap CI half-width ($n = \dsAdeFrozenN$ runs) for stochastic policies.
  $^\dagger$: deterministic policy.
  \textbf{Bold} = best policy at each timestep, for each checkpoint.}
  \label{tab:dataset-scale-ade20k}
  \centering
  \small
  \setcellgapes{1pt}%
  \setlength{\tabcolsep}{4pt}%
  \begin{tabular}{lcccccccc}
    \toprule
    \textbf{Policy} & \dsAdeFrozenHeader \\
    \midrule
        \multicolumn{9}{l}{\textit{CanViT-B, IN21k-pretrained}} \\
    \policyswatch{2CA02C}EG-C2F$^\dagger$ & \textbf{39.6} & \textbf{42.2} & \textbf{43.3} & \textbf{44.1} & \textbf{44.7} & \textbf{45.6} & \textbf{45.8} & 45.7 \\
    \policyswatch{1F77B4}C2F & \textbf{39.6} & \makecell{41.3 \\ $\pm0.10$} & \makecell{42.5 \\ $\pm0.08$} & \makecell{43.6 \\ $\pm0.06$} & \textbf{\makecell{44.7 \\ $\pm0.03$}} & \makecell{45.1 \\ $\pm0.05$} & \makecell{45.6 \\ $\pm0.03$} & \textbf{\makecell{45.9 \\ $\pm0.02$}} \\
    \policyswatch{1A1A1A}F-IID & \textbf{39.6} & \makecell{41.2 \\ $\pm0.10$} & \makecell{42.0 \\ $\pm0.09$} & \makecell{42.5 \\ $\pm0.10$} & \makecell{43.0 \\ $\pm0.09$} & \makecell{44.1 \\ $\pm0.06$} & \makecell{44.7 \\ $\pm0.08$} & \makecell{44.9 \\ $\pm0.07$} \\
    \policyswatch{FF7F0E}R-IID & \makecell{18.3 \\ $\pm0.30$} & \makecell{27.1 \\ $\pm0.37$} & \makecell{31.6 \\ $\pm0.29$} & \makecell{34.3 \\ $\pm0.19$} & \makecell{36.3 \\ $\pm0.16$} & \makecell{40.5 \\ $\pm0.17$} & \makecell{42.6 \\ $\pm0.12$} & \makecell{43.1 \\ $\pm0.13$} \\
    \policyswatch{8C564B}RFS$^\dagger$ & \textbf{39.6} & 41.1 & 41.1 & 41.0 & 40.9 & 40.3 & 39.6 & 39.3 \\
    \policyswatch{D62728}F2C & \makecell{11.0 \\ $\pm0.38$} & \makecell{18.1 \\ $\pm0.32$} & \makecell{22.6 \\ $\pm0.29$} & \makecell{25.9 \\ $\pm0.33$} & \makecell{28.3 \\ $\pm0.29$} & \makecell{35.4 \\ $\pm0.36$} & \makecell{40.1 \\ $\pm0.19$} & \makecell{42.9 \\ $\pm0.15$} \\
    \midrule
    \multicolumn{9}{l}{\textit{CanViT-B, IN1k-pretrained}} \\
    \policyswatch{2CA02C}EG-C2F$^\dagger$ & \textbf{36.4} & \textbf{38.7} & \textbf{40.2} & \textbf{41.0} & \textbf{41.7} & \textbf{43.0} & \textbf{43.3} & 43.3 \\
    \policyswatch{1F77B4}C2F & \textbf{36.4} & \makecell{38.1 \\ $\pm0.03$} & \makecell{39.4 \\ $\pm0.07$} & \makecell{40.7 \\ $\pm0.04$} & \textbf{\makecell{41.7 \\ $\pm0.05$}} & \makecell{42.4 \\ $\pm0.05$} & \makecell{43.1 \\ $\pm0.04$} & \textbf{\makecell{43.4 \\ $\pm0.03$}} \\
    \policyswatch{1A1A1A}F-IID & \textbf{36.4} & \makecell{37.9 \\ $\pm0.03$} & \makecell{38.7 \\ $\pm0.06$} & \makecell{39.3 \\ $\pm0.08$} & \makecell{39.8 \\ $\pm0.07$} & \makecell{41.2 \\ $\pm0.07$} & \makecell{42.0 \\ $\pm0.07$} & \makecell{42.2 \\ $\pm0.07$} \\
    \policyswatch{FF7F0E}R-IID & \makecell{16.9 \\ $\pm0.30$} & \makecell{25.1 \\ $\pm0.27$} & \makecell{29.6 \\ $\pm0.24$} & \makecell{32.5 \\ $\pm0.30$} & \makecell{34.3 \\ $\pm0.31$} & \makecell{38.5 \\ $\pm0.23$} & \makecell{40.5 \\ $\pm0.26$} & \makecell{41.1 \\ $\pm0.25$} \\
    \policyswatch{8C564B}RFS$^\dagger$ & \textbf{36.4} & 37.5 & 37.6 & 37.5 & 37.4 & 36.8 & 36.2 & 36.0 \\
    \policyswatch{D62728}F2C & \makecell{9.4 \\ $\pm0.16$} & \makecell{16.1 \\ $\pm0.17$} & \makecell{20.7 \\ $\pm0.34$} & \makecell{24.0 \\ $\pm0.26$} & \makecell{26.5 \\ $\pm0.22$} & \makecell{33.5 \\ $\pm0.14$} & \makecell{38.2 \\ $\pm0.13$} & \makecell{40.9 \\ $\pm0.14$} \\
 
    \bottomrule
  \end{tabular}
\end{table}

\begin{table}[!htbp]
  \caption{\textbf{ImageNet-1k top-1 accuracy (\%) by policy and timestep, frozen CanViT-B pretrained on IN21k or IN1k ($\dsInkGrid \times \dsInkGrid$ canvas).}
  Cells: $\pm$ \bootstrapCiPct\% bootstrap CI half-width ($n = \dsInkFrozenN$ runs) for stochastic policies.
  $^\dagger$: deterministic policy.
  \textbf{Bold} = best policy at each timestep, for each checkpoint.}
  \label{tab:dataset-scale-in1k-frozen}
  \centering
  \small
  \setcellgapes{1pt}%
  \setlength{\tabcolsep}{4pt}%
  \begin{tabular}{lccccccccc}
    \toprule
    \textbf{Policy} & \dsInkFrozenHeader \\
    \midrule
        \multicolumn{10}{l}{\textit{CanViT-B, IN21k-pretrained}} \\
    \policyswatch{1F77B4}C2F & \textbf{76.79} & \textbf{\makecell{78.62 \\ $\pm0.05$}} & \textbf{\makecell{79.58 \\ $\pm0.02$}} & \textbf{\makecell{80.25 \\ $\pm0.03$}} & \textbf{\makecell{80.77 \\ $\pm0.02$}} & \textbf{\makecell{80.85 \\ $\pm0.02$}} & \textbf{\makecell{81.02 \\ $\pm0.03$}} & \textbf{\makecell{81.12 \\ $\pm0.03$}} & \makecell{81.13 \\ $\pm0.02$} \\
    \policyswatch{1A1A1A}F-IID & \textbf{76.79} & \makecell{78.50 \\ $\pm0.04$} & \makecell{79.32 \\ $\pm0.04$} & \makecell{79.81 \\ $\pm0.05$} & \makecell{80.13 \\ $\pm0.04$} & \makecell{80.37 \\ $\pm0.04$} & \makecell{80.86 \\ $\pm0.03$} & \makecell{81.09 \\ $\pm0.03$} & \textbf{\makecell{81.14 \\ $\pm0.02$}} \\
    \policyswatch{FF7F0E}R-IID & \makecell{47.73 \\ $\pm0.10$} & \makecell{65.23 \\ $\pm0.08$} & \makecell{72.23 \\ $\pm0.09$} & \makecell{75.50 \\ $\pm0.06$} & \makecell{77.24 \\ $\pm0.06$} & \makecell{78.22 \\ $\pm0.06$} & \makecell{79.71 \\ $\pm0.06$} & \makecell{80.37 \\ $\pm0.04$} & \makecell{80.56 \\ $\pm0.05$} \\
    \policyswatch{8C564B}RFS$^\dagger$ & \textbf{76.79} & 77.85 & 77.99 & 78.01 & 78.01 & 77.99 & 77.73 & 77.40 & 77.16 \\
    \policyswatch{D62728}F2C & \makecell{32.22 \\ $\pm0.07$} & \makecell{49.04 \\ $\pm0.08$} & \makecell{58.57 \\ $\pm0.07$} & \makecell{64.40 \\ $\pm0.06$} & \makecell{68.23 \\ $\pm0.10$} & \makecell{70.81 \\ $\pm0.10$} & \makecell{75.93 \\ $\pm0.08$} & \makecell{78.52 \\ $\pm0.06$} & \makecell{79.78 \\ $\pm0.05$} \\
    \midrule
    \multicolumn{10}{l}{\textit{CanViT-B, IN1k-pretrained}} \\
    \policyswatch{1F77B4}C2F & \textbf{77.75} & \textbf{\makecell{79.01 \\ $\pm0.02$}} & \textbf{\makecell{79.74 \\ $\pm0.04$}} & \textbf{\makecell{80.28 \\ $\pm0.03$}} & \textbf{\makecell{80.68 \\ $\pm0.02$}} & \textbf{\makecell{80.82 \\ $\pm0.02$}} & \textbf{\makecell{81.15 \\ $\pm0.01$}} & \textbf{\makecell{81.39 \\ $\pm0.02$}} & \textbf{\makecell{81.49 \\ $\pm0.02$}} \\
    \policyswatch{1A1A1A}F-IID & \textbf{77.75} & \makecell{78.99 \\ $\pm0.03$} & \makecell{79.62 \\ $\pm0.03$} & \makecell{80.04 \\ $\pm0.02$} & \makecell{80.31 \\ $\pm0.04$} & \makecell{80.53 \\ $\pm0.03$} & \makecell{80.98 \\ $\pm0.05$} & \makecell{81.27 \\ $\pm0.04$} & \makecell{81.36 \\ $\pm0.03$} \\
    \policyswatch{FF7F0E}R-IID & \makecell{51.66 \\ $\pm0.16$} & \makecell{68.03 \\ $\pm0.12$} & \makecell{74.19 \\ $\pm0.10$} & \makecell{76.94 \\ $\pm0.08$} & \makecell{78.32 \\ $\pm0.07$} & \makecell{79.10 \\ $\pm0.06$} & \makecell{80.36 \\ $\pm0.04$} & \makecell{80.90 \\ $\pm0.03$} & \makecell{81.06 \\ $\pm0.04$} \\
    \policyswatch{8C564B}RFS$^\dagger$ & \textbf{77.75} & 78.46 & 78.71 & 78.76 & 78.79 & 78.76 & 78.56 & 78.23 & 78.08 \\
    \policyswatch{D62728}F2C & \makecell{37.84 \\ $\pm0.10$} & \makecell{54.25 \\ $\pm0.06$} & \makecell{62.93 \\ $\pm0.09$} & \makecell{67.99 \\ $\pm0.07$} & \makecell{71.26 \\ $\pm0.07$} & \makecell{73.43 \\ $\pm0.06$} & \makecell{77.50 \\ $\pm0.05$} & \makecell{79.44 \\ $\pm0.04$} & \makecell{80.38 \\ $\pm0.03$} \\
 
    \bottomrule
  \end{tabular}
\end{table}

\begin{table}[!htbp]
  \caption{\textbf{ImageNet-1k top-1 accuracy (\%) by policy and timestep, CanViT-B fine-tuned (LP-FT) from the IN21k or IN1k checkpoint ($\dsInkGrid \times \dsInkGrid$ canvas).}
  Cells: $\pm$ \bootstrapCiPct\% bootstrap CI half-width ($n = \dsInkLpftN$ runs) for stochastic policies.
  $^\dagger$: deterministic policy.
  \textbf{Bold} = best policy at each timestep, for each checkpoint.}
  \label{tab:dataset-scale-in1k-lpft}
  \centering
  \small
  \setcellgapes{1pt}%
  \setlength{\tabcolsep}{4pt}%
  \begin{tabular}{lccccccccc}
    \toprule
    \textbf{Policy} & \dsInkLpftHeader \\
    \midrule
        \multicolumn{10}{l}{\textit{CanViT-B, IN21k-pretrained}} \\
    \policyswatch{1F77B4}C2F & \textbf{80.54} & \textbf{\makecell{82.27 \\ $\pm0.02$}} & \textbf{\makecell{83.19 \\ $\pm0.03$}} & \textbf{\makecell{83.78 \\ $\pm0.03$}} & \textbf{\makecell{84.24 \\ $\pm0.03$}} & \textbf{\makecell{84.31 \\ $\pm0.03$}} & \textbf{\makecell{84.39 \\ $\pm0.02$}} & \textbf{\makecell{84.44 \\ $\pm0.02$}} & \textbf{\makecell{84.44 \\ $\pm0.02$}} \\
    \policyswatch{1A1A1A}F-IID & \textbf{80.54} & \makecell{82.18 \\ $\pm0.02$} & \makecell{82.87 \\ $\pm0.03$} & \makecell{83.29 \\ $\pm0.03$} & \makecell{83.57 \\ $\pm0.02$} & \makecell{83.77 \\ $\pm0.02$} & \makecell{84.19 \\ $\pm0.03$} & \makecell{84.34 \\ $\pm0.03$} & \makecell{84.37 \\ $\pm0.02$} \\
    \policyswatch{FF7F0E}R-IID & \makecell{50.99 \\ $\pm0.08$} & \makecell{68.85 \\ $\pm0.07$} & \makecell{75.83 \\ $\pm0.08$} & \makecell{78.97 \\ $\pm0.06$} & \makecell{80.62 \\ $\pm0.05$} & \makecell{81.53 \\ $\pm0.07$} & \makecell{82.91 \\ $\pm0.05$} & \makecell{83.41 \\ $\pm0.05$} & \makecell{83.56 \\ $\pm0.05$} \\
    \policyswatch{8C564B}RFS$^\dagger$ & \textbf{80.54} & 81.72 & 81.83 & 81.74 & 81.61 & 81.45 & 81.05 & 80.67 & 80.39 \\
    \policyswatch{D62728}F2C & \makecell{35.70 \\ $\pm0.10$} & \makecell{53.27 \\ $\pm0.10$} & \makecell{62.89 \\ $\pm0.11$} & \makecell{68.62 \\ $\pm0.09$} & \makecell{72.22 \\ $\pm0.07$} & \makecell{74.70 \\ $\pm0.08$} & \makecell{79.26 \\ $\pm0.07$} & \makecell{81.29 \\ $\pm0.06$} & \makecell{82.28 \\ $\pm0.04$} \\
    \midrule
    \multicolumn{10}{l}{\textit{CanViT-B, IN1k-pretrained}} \\
    \policyswatch{1F77B4}C2F & \textbf{79.93} & \textbf{\makecell{81.75 \\ $\pm0.03$}} & \textbf{\makecell{82.66 \\ $\pm0.02$}} & \textbf{\makecell{83.29 \\ $\pm0.02$}} & \textbf{\makecell{83.80 \\ $\pm0.02$}} & \textbf{\makecell{83.90 \\ $\pm0.02$}} & \textbf{\makecell{84.07 \\ $\pm0.02$}} & \textbf{\makecell{84.19 \\ $\pm0.02$}} & \textbf{\makecell{84.19 \\ $\pm0.02$}} \\
    \policyswatch{1A1A1A}F-IID & \textbf{79.93} & \makecell{81.67 \\ $\pm0.03$} & \makecell{82.42 \\ $\pm0.05$} & \makecell{82.88 \\ $\pm0.04$} & \makecell{83.18 \\ $\pm0.04$} & \makecell{83.37 \\ $\pm0.04$} & \makecell{83.83 \\ $\pm0.04$} & \makecell{84.06 \\ $\pm0.04$} & \makecell{84.11 \\ $\pm0.03$} \\
    \policyswatch{FF7F0E}R-IID & \makecell{52.92 \\ $\pm0.12$} & \makecell{69.74 \\ $\pm0.11$} & \makecell{76.13 \\ $\pm0.09$} & \makecell{79.01 \\ $\pm0.07$} & \makecell{80.49 \\ $\pm0.05$} & \makecell{81.26 \\ $\pm0.05$} & \makecell{82.53 \\ $\pm0.03$} & \makecell{83.06 \\ $\pm0.03$} & \makecell{83.23 \\ $\pm0.03$} \\
    \policyswatch{8C564B}RFS$^\dagger$ & \textbf{79.93} & 81.06 & 81.22 & 81.19 & 81.07 & 80.92 & 80.59 & 80.19 & 80.05 \\
    \policyswatch{D62728}F2C & \makecell{38.61 \\ $\pm0.09$} & \makecell{55.47 \\ $\pm0.11$} & \makecell{64.37 \\ $\pm0.08$} & \makecell{69.61 \\ $\pm0.06$} & \makecell{72.97 \\ $\pm0.05$} & \makecell{75.23 \\ $\pm0.05$} & \makecell{79.18 \\ $\pm0.03$} & \makecell{80.93 \\ $\pm0.05$} & \makecell{81.93 \\ $\pm0.05$} \\
 
    \bottomrule
  \end{tabular}
\end{table}

\FloatBarrier
\subsection{Viewpoint Encoding at long horizons and high resolution}
\label{supp:vpe-long-horizon}

We hypothesized that the VPE token's additional location information could help in long-horizon and high-resolution settings, where integration of fine detail matters most.
To test this, we evaluated the VPE and no-VPE ablation checkpoints (Table~\ref{tab:ablations}, baseline and row~\textbf{k}) with the EG-C2F policy, using ADE20K probes for $32 \times 32$ (C32) and $64 \times 64$ (C64) canvases, on \vpeScenes{} scenes respectively.
We report the results in Table~\ref{tab:vpe-long-horizon}, which show that VPE is beneficial at long horizons and high resolution.

\begin{table}[h]
  \caption{\textbf{Ablation study: ADE20K mIoU (\%) with and without the VPE token, EG-C2F policy.}
  \vpeGeometries{}}
  \label{tab:vpe-long-horizon}
  \centering
  \small
  \begin{tabular}{lccccc}
    \toprule
    & \vpeHeader \\
    \midrule
        C32 (VPE) & 29.24 & 35.63 & 37.62 & 37.88 & 37.89 \\
    C32 (no VPE) & 29.34 & 35.47 & 37.17 & 37.20 & 37.03 \\
    \textbf{C32 (VPE) $-$ C32 (no VPE)} & \textbf{$-$0.10} & \textbf{+0.17} & \textbf{+0.45} & \textbf{+0.69} & \textbf{+0.86} \\
    \midrule
    C64 (VPE) & 29.89 & 36.71 & 38.76 & 39.04 & 39.00 \\
    C64 (no VPE) & 29.59 & 35.80 & 37.75 & 37.95 & 37.72 \\
    \textbf{C64 (VPE) $-$ C64 (no VPE)} & \textbf{+0.30} & \textbf{+0.91} & \textbf{+1.01} & \textbf{+1.09} & \textbf{+1.28} \\
 
    \bottomrule
  \end{tabular}
\end{table}

\FloatBarrier
\subsection{ImageNet-1k accuracy against inference cost}
\label{supp:in1k-efficiency}

While our core efficiency claims are made within the scope of dense scene-wide prediction, our IN1k results can be further contextualized in two respects: on accuracy against inference cost, and against GFNet~\cite{wangGlanceFocusDynamic2020} in addition to AdaptiveNN~\cite{wangEmulatingHumanlikeAdaptive2025} (Table~\ref{tab:in1k-efficiency}).
We use the ``anytime prediction'' inference setting, which fixes the number of glimpses per image.
We recomputed every AdaptiveNN and GFNet cell from their released code and checkpoints, ensuring that all three models are measured the same way.
All FLOP counts follow the convention of two FLOPs per MAC, and include policy networks and classification heads.
For AdaptiveNN and GFNet, the multiple-glimpse rows use the largest number of glimpses their released models take.
Parameter counts include every module that runs up to the evaluated glimpse.
AdaptiveNN and GFNet run a policy network and a second encoder after the first glimpse, and a classifier at every glimpse, so their parameter counts grow with the number of glimpses.
For the $\inkEffLpftSmallGrid^2$-canvas LP-FT rows, we fine-tuned the pretrained flagship with an $\inkEffLpftSmallGrid^2$ canvas.

\begin{table}[h]
  \caption{\textbf{ImageNet-1k top-1 accuracy (\%), parameters and cumulative inference GFLOPs of CanViT-B (C2F policy), AdaptiveNN and GFNet at a fixed number of glimpses.}
  Cells: $\pm$ \bootstrapCiPct\% bootstrap CI half-width over evaluation runs (\inkEffRuns).}
  \label{tab:in1k-efficiency}
  \centering
  \small
  \begin{tabular}{lccccc}
    \toprule
    \textbf{Model} & \textbf{Input} & \textbf{Glimpses} & \textbf{Params (M)} & \textbf{GFLOPs} & \textbf{Top-1} \\
    \midrule
        \multicolumn{6}{l}{\textit{Single glimpse}} \\
    CanViT-B, LP-FT (canvas $32^2$) & $128^2$\,px & 1 & 95.9 & 15.48 & 80.59 \\
    CanViT-B, LP-FT (canvas $8^2$) & $128^2$\,px & 1 & 95.9 & 13.80 & 80.54 \\
    CanViT-B, frozen (canvas $32^2$) & $128^2$\,px & 1 & 95.9 & 15.48 & 76.79 \\
    CanViT-B, frozen (canvas $8^2$) & $128^2$\,px & 1 & 95.9 & 13.80 & 76.73 \\
    AdaptiveNN (DeiT-S) & $112^2$\,px & 1 & 22.0 & 2.15 & 73.96 \\
    GFNet (EfficientNet-B3) & $144^2$\,px & 1 & 12.2 & 0.86 & 78.04 \\
    GFNet (MobileNet-V3-L 1.00) & $128^2$\,px & 1 & 5.5 & 0.15 & 70.72 \\
    \midrule
    \multicolumn{6}{l}{\textit{Multiple glimpses}} \\
    CanViT-B, LP-FT (canvas $32^2$) & $128^2$\,px & 3 & 95.9 & 46.44 & 83.21 $\pm$ 0.03 \\
    CanViT-B, LP-FT (canvas $8^2$) & $128^2$\,px & 3 & 95.9 & 41.41 & 83.09 $\pm$ 0.04 \\
    CanViT-B, frozen (canvas $32^2$) & $128^2$\,px & 3 & 95.9 & 46.44 & 79.58 $\pm$ 0.02 \\
    CanViT-B, frozen (canvas $8^2$) & $128^2$\,px & 3 & 95.9 & 41.41 & 79.29 $\pm$ 0.04 \\
    AdaptiveNN (DeiT-S) & $112^2$\,px & 5 & 82.4 & 14.48 & 82.29 \\
    GFNet (EfficientNet-B3) & $144^2$\,px & 4 & 37.4 & 3.49 & 79.92 \\
    GFNet (MobileNet-V3-L 1.00) & $128^2$\,px & 3 & 16.6 & 0.45 & 73.67 \\
 
    \bottomrule
  \end{tabular}
\end{table}

\FloatBarrier
\subsection{Architecture control}
\label{supp:architecture-control}

To compare our architecture with AdaGlimpse's under the same DINOv3 teacher and data, we ran a small-scale controlled experiment, pitting an AdaGlimpse baseline (\ctrlParamsBaseline{}M parameters, including its Transformer decoder to DINOv3 space) against CanViT-B (\ctrlParamsCanvit{}M parameters, including its linear projections to DINOv3 space), with the same glimpse budget as our ablation baseline (approximately 110M glimpses; Appendix~\ref{supp:ablation-study}), from a random initialization.
We trained both models with the targets and loss of Section~\ref{subsec:distillation}.
As we controlled for IN1k-vs-IN21k effects above, we used IN1k images here.
We used 512 scenes per step, one \mbox{F-IID} and one \mbox{R-IID} rollout per scene, with dense DINOv3 supervision across $T = \ctrlT$ timesteps.
We implemented both models in JAX, porting AdaGlimpse's architecture from its PyTorch release, and trained them for \ctrlPretrainSteps{} optimizer steps.
Following AdaGlimpse's released code, the baseline divides each glimpse into $\ctrlTokensPerSide \times \ctrlTokensPerSide$ crops and downsamples each crop to a single patch, representing a glimpse with \ctrlBaselineTokens{} tokens.
CanViT-B instead splits the glimpse into $\ctrlCanvitPatchesPerSide \times \ctrlCanvitPatchesPerSide$ patches of $\ctrlPatchPx^2$\,px, representing it with \ctrlCanvitTokens{} tokens.
We ran this experiment on TPU v6e chips, with a pretraining cost of \ctrlChipHoursCanvit{} v6e chip-hours for CanViT-B and \ctrlChipHoursBaseline{} v6e chip-hours for the AdaGlimpse baseline.
We fit one ADE20K linear probe to each model's frozen features.
We trained these probes over $T = \ctrlT$ timesteps, each model's training horizon.
We report the results in Table~\ref{tab:architecture-control} with an EG-C2F policy and a $\ctrlGrid \times \ctrlGrid$ output feature map / canvas resolution.

\begin{table}[h]
  \caption{\textbf{Architecture control: ADE20K mIoU (\%)} \textbf{against cumulative forward inference FLOPs.}
  EG-C2F policy, $\ctrlGrid \times \ctrlGrid$ canvas / output feature map. GFLOPs include each model's ADE20K probe.}
  \label{tab:architecture-control}
  \centering
  \small
  \begin{tabular}{ccccc}
    \toprule
    $t$ & \textbf{mIoU, CanViT-B} & \textbf{mIoU, AdaGlimpse} & \textbf{GFLOPs, CanViT-B} & \textbf{GFLOPs, AdaGlimpse} \\
    \midrule
        0 & 26.86 & 12.00 & 15.8 & 70.2 \\
    1 & 30.52 & 13.73 & 31.6 & 141.3 \\
    2 & 32.77 & 14.88 & 47.4 & 213.5 \\
    3 & 33.69 & 15.61 & 63.2 & 286.8 \\
 
    \bottomrule
  \end{tabular}
\end{table}

Here, CanViT-B is ahead at every glimpse count, by \ctrlLead{} mIoU at $t = \ctrlLastT$.
It reaches its fourth glimpse for less compute than the baseline spends on its first, while requiring considerably less training compute under the same harness.
This highlights our architecture's efficiency, both at training and at inference; CanViT's ability to inexpensively output scene-wide dense features allows it to reallocate a given budget to \emph{more} glimpses and \emph{higher-resolution} glimpses, resulting in performance improvements across the board.

\FloatBarrier
\subsection{COCO object detection}
\label{supp:coco-detection}

While object detection is a challenging and interesting setting, we did not pursue it in the main experiments, as none of our baselines report results on standard detection benchmarks.
Building on the architecture control above (AdaGlimpse vs.\ CanViT under the same teacher distillation training scheme), we ran a preliminary object detection experiment on COCO~\cite{linMicrosoftCOCOCommon2014}.
We fine-tuned the architecture-control CanViT-B and AdaGlimpse checkpoints with an identical CenterNet~\cite{zhouObjectsPoints2019} detection head, for approximately 35 epochs, with $T = \cocoMatchedT$ glimpses per scene.
We also fine-tuned the flagship checkpoint for \cocoFlagshipEpochs{} epochs, with a longer training, BPTT and evaluation horizon ($T = \cocoFlagshipT$ glimpses in training, evaluation up to $t = \cocoFlagshipLastT$). Table~\ref{tab:coco-settings} lists the training settings.

\begin{table}[!h]
  \caption{\textbf{COCO detection training settings} of the two architecture-control models and of the flagship detector.}
  \label{tab:coco-settings}
  \centering
  \small
  \begin{tabular}{lcc}
    \toprule
    \textbf{Hyperparameter} & \textbf{Architecture control} & \textbf{Flagship} \\
    \midrule
    Peak learning rate, detection head & $1 \times 10^{-3}$ & $4 \times 10^{-4}$ \\
    Peak learning rate, pretrained model & $1 \times 10^{-4}$ & $4 \times 10^{-4}$ \\
    Weight decay & $1 \times 10^{-3}$ & $1 \times 10^{-1}$ \\
    Batch size & 256 & 1{,}024 \\
    Glimpses per scene & $T = \cocoMatchedT$ & $T = \cocoFlagshipT$ \\
    Optimizer steps at evaluation & 16{,}000 of 92{,}411 & 16{,}172 \\
    \midrule
    Optimizer & \multicolumn{2}{c}{AdamW, gradient norm clipped at 1.0} \\
    Schedule & \multicolumn{2}{c}{Warmup (20\% of steps) $\to$ cosine decay to zero} \\
    Training rollouts & \multicolumn{2}{c}{1 F-IID + 1 R-IID, min.\ scale 0.25, BPTT through all glimpses} \\
    Head & \multicolumn{2}{c}{$3 \times 3$ conv (256 channels); $1 \times 1$ convs for heatmap, size, offset} \\
    Loss & \multicolumn{2}{c}{focal ($\alpha = 2$, $\beta = 4$) $+\ 0.4 \times$ size L1 $+$ offset L1} \\
    Augmentation & \multicolumn{2}{c}{horizontal flip, photometric distortion} \\
    \bottomrule
  \end{tabular}
\end{table}

\begin{table}[h]
  \caption{\textbf{COCO val2017 AP by timestep, C2F policy ($\cocoGrid \times \cocoGrid$ canvas / output feature map, $\cocoScene^2$\,px scenes).}
  Cells: $\pm$ \bootstrapCiPct\% bootstrap CI half-width ($n = \cocoRepeats$ runs). ---: not evaluated.}
  \label{tab:coco-detection}
  \centering
  \small
  \setlength{\tabcolsep}{4pt}%
  \begin{tabular}{lccccccc}
    \toprule
    \textbf{Model (initialized from)} & \cocoHeader \\
    \midrule
        \makecell[l]{AdaGlimpse \\ (architecture control)} & 5.55 & \makecell{6.37 \\ $\pm0.040$} & \makecell{7.02 \\ $\pm0.034$} & \makecell{7.59 \\ $\pm0.023$} & \makecell{8.12 \\ $\pm0.006$} & --- & --- \\
    \makecell[l]{CanViT-B \\ (architecture control)} & 12.01 & \makecell{14.41 \\ $\pm0.091$} & \makecell{16.16 \\ $\pm0.059$} & \makecell{17.67 \\ $\pm0.059$} & \makecell{19.02 \\ $\pm0.025$} & --- & --- \\
    \midrule
    \makecell[l]{CanViT-B \\ (flagship)} & 19.31 & \makecell{22.08 \\ $\pm0.039$} & \makecell{24.37 \\ $\pm0.079$} & \makecell{26.43 \\ $\pm0.057$} & \makecell{28.39 \\ $\pm0.026$} & \makecell{29.78 \\ $\pm0.043$} & \makecell{32.42 \\ $\pm0.014$} \\
 
    \bottomrule
  \end{tabular}
\end{table}

We observe that the architecture-control CanViT-B more than doubles the baseline's AP at every timestep (Table~\ref{tab:coco-detection}).
Fine-tuned from the flagship checkpoint, CanViT-B improves with every glimpse up to $t = \cocoFlagshipLastT$.
We see this limited, simple experiment as encouraging evidence that CanViT can be straightforwardly trained to perform detection tasks.

\FloatBarrier
\section{Declarations}
\label{supp:declarations}

\subsection{Availability of code and data}
\label{supp:declarations:code-and-data}
We release our code (\url{https://github.com/m2b3/CanViT}) and trained checkpoints (\url{https://huggingface.co/canvit}) under the MIT license, as well as the data needed to regenerate our figures and tables (\url{https://doi.org/10.6084/m9.figshare.32148076}).

\subsection{Broader impact}
\label{supp:declarations:broader-impact}
CanViT's constant-memory recurrent architecture, fast sequential inference and scalability to large scenes make it an attractive candidate for future extension to real-time, high-FPS video processing, embodied active perception, and high-resolution static image processing. In turn, such extensions may support applications in a variety of domains, including edge/physical AI, medical imaging, and camera- or satellite-based environmental monitoring.

As with nearly all general-purpose computer vision models, derivatives of this work might be incorporated into surveillance and security technology. As such, we invite researchers building upon active-vision foundation models to engage with the societal implications of any specific application.

\subsection{Compute reporting}
\label{supp:declarations:compute-reporting}
Before submission, we used approximately 2{,}500 H100-equivalent hours on our SLURM-based compute platform. Our pretraining runs required under 24 GB of GPU memory (18.2 GB for the 2M-step CanViT-B pretraining run described in Appendix~\ref{supp:pretraining-details}). For interactive development, probe training, and downstream evaluations, we additionally used a single NVIDIA RTX 4090 workstation; we did not have fine-grained usage tracking in place there, but estimate a total of approximately 2{,}000 RTX 4090 GPU-hours.

The ImageNet-1k fine-tuning experiments of Appendix~\ref{supp:task-in1k} were run on Google Cloud TPU v6e-4 spot instances using \texttt{torch\_xla}. The total TPU cost of these experiments was under 800 USD (\computeAprilChipHours{} v6e chip-hours). The total runtime of our reported IN1k fine-tuning run was under 15 wall-clock hours on TPU v6e-4 using SPMD data parallelism. We did not perform extensive optimization of our \texttt{torch\_xla} code.

During the review and rebuttal period, we used a further \computeSlurmHours{} H100-equivalent hours on our SLURM-based compute platform and
at least \computeRtxHours{} GPU-hours on the RTX 4090 workstation.
On Google Cloud, we used \computeTpuVsixeChipHours{} TPU v6e and \computeTpuVfivepChipHours{} TPU v5p chip-hours, at a list price of \computeTpuUsd{} USD.

The numbers above include failed runs, preliminary experiments, and hyperparameter sweeps.

\Needspace{5\baselineskip}
\subsection{Statistical analysis}
\label{supp:declarations:statistical-analysis}

All confidence intervals reported in this paper are \bootstrapCiPct\% bootstrap CIs.
For the four stochastic viewing policies (F-IID, R-IID, C2F, F2C), each \emph{evaluation run} is one full pass over the evaluation set under a fresh policy seed. We bootstrap the mean across runs (percentile method, 10{,}000 resamples). We plot these CIs as shaded bands (often too tight to be visible) in Figure~\ref{fig:main-ade20k-in1k-results} and Figure~\ref{supp:fig-canvas-grid-impact}. The other two non-learned policies (EG-C2F, RFS) are deterministic and yield single-run point estimates.
For the learned viewing policies (Appendix~\ref{supp:learned-policies}), each training run yields two deterministic policies, its best and last checkpoints. For each objective and checkpoint, we bootstrap the mean across its \rlTrainingRuns{} training runs in the same way (Figure~\ref{fig:main-ade20k-in1k-results}\textbf{D}, Table~\ref{tab:learned-policies}).
For Figure~\ref{fig:resolution-and-mask-size-analysis}\textbf{A,B}, we bootstrap a LOWESS regression across ground-truth masks (1{,}000 resamples; full methodology in Appendix~\ref{supp:iou-obj}).

\subsection{Licenses and attribution}
\label{supp:declarations:licenses-and-attribution}

\textbf{Pretraining and evaluation.}
CanViT-B was pretrained on ImageNet-21k's \texttt{winter21\_whole} split~\cite{ridnikImageNet21KPretrainingMasses2021}, and evaluated on the ImageNet-1k~\cite{russakovskyImageNetLargeScale2015} and ADE20K~\cite{zhouSceneParsingADE20K2017} datasets.
In Appendix~\ref{supp:additional-experiments}, we also pretrained CanViT-B on ImageNet-1k, and fine-tuned and evaluated detectors on COCO~\cite{linMicrosoftCOCOCommon2014}. CanViT's DINOv3 ViT-B teacher is distributed under the DINOv3 License.

\textbf{Example images.}
The cat image (Cat03.jpg) used as an example in figures throughout the paper is sourced from Wikimedia Commons and is attributed to Fir0002/Flagstaffotos under the CC BY-NC 3.0 license. Other example images were sourced from the Places365 dataset~\cite{zhouPlaces10Million2017} and used solely for illustration purposes.

\makeatletter
\if@preprint\else
  \section*{NeurIPS Paper Checklist}

\begin{enumerate}

\item {\bf Claims}
    \item[] Question: Do the main claims made in the abstract and introduction accurately reflect the paper's contributions and scope?
    \item[] Answer: \answerYes{}
    \item[] Justification: Our key contributions are stated in the abstract and Section~\ref{sec:introduction}, and substantiated in the architecture description (Section~\ref{sec:arch}), pretraining recipe (Section~\ref{sec:pretraining}), and main experimental results (Section~\ref{sec:experiments}).

\item {\bf Limitations}
    \item[] Question: Does the paper discuss the limitations of the work performed by the authors?
    \item[] Answer: \answerYes{}
    \item[] Justification: We discuss the limitations in Section~\ref{sec:conclusion}.

\item {\bf Theory assumptions and proofs}
    \item[] Question: For each theoretical result, does the paper provide the full set of assumptions and a complete (and correct) proof?
    \item[] Answer: \answerYes{}
    \item[] Justification: While most of the paper is empirical, theoretical results about Viewpoint Encoding are described and supported by proofs in Appendix~\ref{supp:vpe}.

    \item {\bf Experimental result reproducibility}
    \item[] Question: Does the paper fully disclose all the information needed to reproduce the main experimental results of the paper to the extent that it affects the main claims and/or conclusions of the paper (regardless of whether the code and data are provided or not)?
    \item[] Answer: \answerYes{}
    \item[] Justification: To reproduce our experimental results, we describe the CanViT architecture and Canvas Attention mechanism in detail (Section~\ref{sec:arch} and Appendix~\ref{supp:canvas-attention-pseudocode}), the pretraining protocol (Section~\ref{sec:pretraining} and Appendix~\ref{supp:pretraining-details}), including all necessary hyperparameters, the evaluation details (Section~\ref{sec:experiments} and Appendix~\ref{supp:evaluation-details}), the ablation study (Appendix~\ref{supp:ablation-study}), and the additional experiments (Appendix~\ref{supp:additional-experiments}). This complements our code release.

\item {\bf Open access to data and code}
    \item[] Question: Does the paper provide open access to the data and code, with sufficient instructions to faithfully reproduce the main experimental results, as described in supplemental material?
    \item[] Answer: \answerYes{}
    \item[] Justification: We release code and checkpoints under the MIT license, as well as the data necessary to regenerate figures and tables (Appendix~\ref{supp:declarations:code-and-data}). We document the code in its repository, each checkpoint in its model card, and the data in the README of its release.

\item {\bf Experimental setting/details}
    \item[] Question: Does the paper specify all the training and test details (e.g., data splits, hyperparameters, how they were chosen, type of optimizer) necessary to understand the results?
    \item[] Answer: \answerYes{}
    \item[] Justification: We describe training and evaluation details in Section~\ref{sec:experiments}, Appendix~\ref{supp:pretraining-details}, Appendix~\ref{supp:evaluation-details}, Appendix~\ref{supp:ablation-study} and Appendix~\ref{supp:additional-experiments}.

\item {\bf Experiment statistical significance}
    \item[] Question: Does the paper report error bars suitably and correctly defined or other appropriate information about the statistical significance of the experiments?
    \item[] Answer: \answerYes{}
    \item[] Justification: For experiments involving stochastic viewing policies, we report \bootstrapCiPct\% bootstrap CIs over independent evaluation runs, shown as shaded bands in Figure~\ref{fig:main-ade20k-in1k-results} and Figure~\ref{supp:fig-canvas-grid-impact} and as $\pm$ half-widths or lo--hi ranges in Table~\ref{supp:tab-policy-eval}, Table~\ref{supp:tab-ade20k-canvas-sweep}, Table~\ref{supp:tab-in1k-frozen-eval}, Table~\ref{supp:tab-in1k-canvas-sweep}, Table~\ref{supp:tab-in1k-finetuned-eval}, Table~\ref{tab:ablation-eval-supp}, Table~\ref{tab:ablation-downstream}, Table~\ref{tab:dataset-scale}, Table~\ref{tab:dataset-scale-ade20k}, Table~\ref{tab:dataset-scale-in1k-frozen}, Table~\ref{tab:dataset-scale-in1k-lpft}, Table~\ref{tab:in1k-efficiency}, and Table~\ref{tab:coco-detection}. For the learned viewing policies, the CIs are over independent training runs (Figure~\ref{fig:main-ade20k-in1k-results}\textbf{D}, Table~\ref{tab:learned-policies}). See Appendix~\ref{supp:declarations:statistical-analysis} for more details.

\item {\bf Experiments compute resources}
    \item[] Question: For each experiment, does the paper provide sufficient information on the computer resources (type of compute workers, memory, time of execution) needed to reproduce the experiments?
    \item[] Answer: \answerYes{}
    \item[] Justification: We discuss compute workers, GPU memory, and time of execution in Appendix~\ref{supp:declarations:compute-reporting}, and storage (precomputed teacher features) in Appendix~\ref{supp:pretraining-details}.

\item {\bf Code of ethics}
    \item[] Question: Does the research conducted in the paper conform, in every respect, with the NeurIPS Code of Ethics \url{https://neurips.cc/public/EthicsGuidelines}?
    \item[] Answer: \answerYes{}
    \item[] Justification: The research we present here fully conforms with the NeurIPS Code of Ethics.

\item {\bf Broader impacts}
    \item[] Question: Does the paper discuss both potential positive societal impacts and negative societal impacts of the work performed?
    \item[] Answer: \answerYes{}
    \item[] Justification: We discuss potential positive and negative impacts in Appendix~\ref{supp:declarations:broader-impact}.

\item {\bf Safeguards}
    \item[] Question: Does the paper describe safeguards that have been put in place for responsible release of data or models that have a high risk for misuse (e.g., pre-trained language models, image generators, or scraped datasets)?
    \item[] Answer: \answerNA{}
    \item[] Justification: The work presented in this paper does not pose any such direct risk.

\item {\bf Licenses for existing assets}
    \item[] Question: Are the creators or original owners of assets (e.g., code, data, models), used in the paper, properly credited and are the license and terms of use explicitly mentioned and properly respected?
    \item[] Answer: \answerYes{}
    \item[] Justification: We cite and credit existing assets throughout the paper and in Appendix~\ref{supp:declarations:licenses-and-attribution}.

\item {\bf New assets}
    \item[] Question: Are new assets introduced in the paper well documented and is the documentation provided alongside the assets?
    \item[] Answer: \answerYes{}
    \item[] Justification: We release our code, checkpoints and data with their documentation (Appendix~\ref{supp:declarations:code-and-data}).

\item {\bf Crowdsourcing and research with human subjects}
    \item[] Question: For crowdsourcing experiments and research with human subjects, does the paper include the full text of instructions given to participants and screenshots, if applicable, as well as details about compensation (if any)?
    \item[] Answer: \answerNA{}
    \item[] Justification: Our paper does not involve crowdsourcing nor research with human subjects.

\item {\bf Institutional review board (IRB) approvals or equivalent for research with human subjects}
    \item[] Question: Does the paper describe potential risks incurred by study participants, whether such risks were disclosed to the subjects, and whether Institutional Review Board (IRB) approvals (or an equivalent approval/review based on the requirements of your country or institution) were obtained?
    \item[] Answer: \answerNA{}
    \item[] Justification: Our paper does not involve crowdsourcing nor research with human subjects.

\item {\bf Declaration of LLM usage}
    \item[] Question: Does the paper describe the usage of LLMs if it is an important, original, or non-standard component of the core methods in this research? Note that if the LLM is used only for writing, editing, or formatting purposes and does \emph{not} impact the core methodology, scientific rigor, or originality of the research, declaration is not required.
    \item[] Answer: \answerNA{}
    \item[] Justification: Our methodology does not involve LLMs as any important, original, or non-standard components. However, we have made extensive use of agentic AI tools, such as Claude Code with Claude Opus 4.7 to 5.5, when developing, debugging and refactoring our code.

\end{enumerate}

\fi
\makeatother

\end{document}